\documentclass[twoside,11pt]{article}

\usepackage[utf8]{inputenc}
\usepackage{subfiles}
\usepackage[titletoc]{appendix}
\usepackage{url}

\usepackage{multirow}
\usepackage{tabularx}
\usepackage{arydshln}
\usepackage{subcaption}
\usepackage{wrapfig}
\usepackage{censor}

\usepackage{blindtext}

\usepackage[ruled,vlined]{algorithm2e}

\usepackage{bm}
\usepackage{bbm}
\usepackage[mathscr]{eucal}
\usepackage{amssymb}
\usepackage{amsfonts}
\usepackage{mathtools}
\usepackage[american]{babel}
\usepackage{multicol}
\usepackage{multirow}
\usepackage{pifont}

\usepackage{thmtools, thm-restate}
\usepackage{booktabs}
\usepackage{nicefrac}
\usepackage{soul}

\usepackage[textsize=scriptsize]{todonotes}
\usepackage{jmlr2e}

\def\@onedot{\ifx\@let@token.\else.\null\fi\xspace}

\makeatother

\DeclareMathOperator*{\argmax}{arg\,max}
\DeclareMathOperator*{\argmin}{arg\,min}

\def \inv {^{-1}}

\usepackage{lastpage}
\jmlrheading{24}{2023}{1-\pageref{LastPage}}{10/21; Revised
8/22}{2/23}{21-1213}{Marlos C. Machado, Andr\'e Barreto, Doina Precup, and Michael Bowling}
\ShortHeadings{title}{Machado, Barreto, Precup, and Bowling}

\ShortHeadings{Temporal Abstraction in RL with the Successor Representation}{Machado, Barreto, Precup, and Bowling}
\firstpageno{1}

\setuptodonotes{color=gray!30}

\begin{document}

\SetKwComment{Comment}{$\triangleright$ \ }{}
\SetKwInput{KwInput}{Input}
\SetKwInput{KwOutput}{Output}
\SetCommentSty{normalfont}

\title{Temporal Abstraction in Reinforcement Learning\\ with the Successor Representation}

\author{\name Marlos C. Machado \email machado@ualberta.com\\
       \addr DeepMind\\
       \addr Alberta Machine Intelligence Institute (Amii)\\
       \addr Department of Computing Science, University of Alberta\\
       \addr Edmonton, AB, Canada
       \AND
       \name Andr\'e Barreto \email andrebarreto@deepmind.com\\
       \addr DeepMind\\
       \addr London, United Kingdom
       \AND
       \name Doina Precup \email doinap@deepmind.com\\
       \addr DeepMind\\
       \addr Quebec AI Institute (Mila)\\
       \addr School of Computer Science, McGill University\\
       \addr Montreal, QC, Canada
       \AND
       \name Michael Bowling \email mbowling@ualberta.ca\\
       \addr DeepMind\\
       \addr Alberta Machine Intelligence Institute (Amii)\\
       \addr Department of Computing Science, University of Alberta\\
       \addr Edmonton, AB, Canada
       \vspace{-0.3cm}
}

\editor{Jan Peters}
\maketitle

\vspace{-0.6cm}

\begin{abstract}%   <- trailing '%' for backward compatibility of .sty file

Reasoning at multiple levels of temporal abstraction is one of the key attributes of intelligence. In reinforcement learning, this is often modeled through temporally extended courses of actions called \emph{options}. Options allow agents to make predictions and to operate at different levels of abstraction within an environment.  Nevertheless, approaches based on the options framework often start with the assumption that a reasonable set of options is known beforehand. When this is not the case, there are no definitive answers for which options one should consider. In this paper, we argue that the successor representation, which encodes states based on the pattern of state visitation that follows them, can be seen as a natural substrate for the discovery and use of temporal abstractions. To support our claim, we take a big picture view of recent results, showing how the successor representation can be used to discover options that facilitate either temporally-extended exploration or planning. We cast these results as instantiations of a general framework for option discovery in which the agent’s representation is used to identify useful options, which are then used to further improve its representation. This results in a virtuous, never-ending, cycle in which both the representation and the options are constantly refined based on each other. Beyond option discovery itself, we also discuss how the successor representation allows us to augment a set of options into a combinatorially large counterpart without additional learning. This is achieved through the combination of previously learned options. Our empirical evaluation focuses on options discovered for temporally-extended exploration and on the use of the successor representation to combine them. Our results shed light on important design decisions involved in the definition of options and demonstrate the synergy of different methods based on the successor representation, such as eigenoptions and the option keyboard.\looseness=-1


\end{abstract}

\section{Introduction}
In the reinforcement learning problem, an agent interacts with its environment such that the agent receives an observation from the environment and takes an action based on the received observations. This interaction takes place at every time step, which is often the fundamental unit of time in this problem formulation. Nevertheless, several decision making problems,  such as robot locomotion~\citep{Stone05}, strategy games like StarCraft~\citep{Vinyals19}, and balloon navigation~\citep{Bellemare20}, involve operating over different time scales. The options framework \citep{Precup00,Sutton99} is maybe the most common formalism that allows us to do so, giving agents the ability to reason in terms of actions extended in time. This framework models courses of actions as \emph{options}, which have the ability to accelerate learning in different ways, allowing, for example, faster credit assignment~\citep[e.g.,][]{Mann14,Solway14}, better exploration~\citep[e.g.,][]{Baranes13,Fruit17}, and transfer~\citep[e.g.,][]{Konidaris07,Topin15}.

Despite the attention received by the options framework, it is still not clear where options should come from---a problem referred to as \emph{option discovery}. In this paper 
we argue that the \emph{successor representation} (SR) is a natural substrate for temporal abstractions in reinforcement learning. The SR~\citep{Dayan93} is a representation that generalizes between states using the similarity between their successors, that is, the similarity between the states that follow the current state given the environment’s dynamics and the agent’s policy. It allows us to discover options that are effective not only for planning \citep[e.g.,][]{Hoang21,Ramesh19,Stachenfeld17}, but also for temporally-extended exploration \citep[e.g.,][]{Jinnai19,Machado17,Machado18b}. The SR also allows us to combine existing options without additional learning \citep{Barreto19}. Furthermore, recent studies suggest that the SR models remarkably well behaviors observed in the brain~\citep[e.g.,][]{Momennejad17,Stachenfeld14,Stachenfeld17}.

In this paper, we present a general framework for option discovery in which the agent learns a representation that is used to identify meaningful options, which are then used to improve the agent's representation in a virtuous, never-ending, cycle (Sections \ref{sec:framework}~and~\ref{sec:illustration_rod_cycle}). To support our claim about the role of the SR for temporal abstraction, we show how we can instantiate this cycle with the SR, and how the SR is conducive to option discovery. We summarize existing methods that use the SR for option discovery, providing intuitions about the motivation behind them, and connecting papers from different contexts (Sections~\ref{sec:temporally_extended_exploration} and~\ref{sec:related_work}). Moreover, regardless of how effective a discovery method is, while more options means a more expressive set of behaviors, more options often makes learning and using these options more difficult. Thus, we also discuss an approach based on the SR for combining options, the \emph{option keyboard} \citep{Barreto19}, which addresses this issue by allowing the agent to extend, without extra learning, a finite set of options to a combinatorially large counterpart~(Section \ref{sec:options_keyboard}).

We perform numerical simulations to assess how effective options discovered by different methods are in capturing environment properties. Such an ability is particularly important for problems in which a fixed reward function is not easily defined, such as continual~\citep{Brunskill14,Mankowitz18}, multitask~\citep{Teh17}, and transfer learning~\citep{Taylor09}. We evaluate the impact of different design decisions every option discovery method needs to make~(Section \ref{sec:experiments_temporally_extended_exploration}). We present evidence on the potential of instantiating a cycle in which both the representation and the options are constantly refined based on each other (Section~\ref{sec:illustration_rod_cycle}), and on the synergy of different approaches based on the SR~(Section \ref{sec:experiments_ok}). We focus our discussion mostly on toy domains to provide intuition without confounding factors. We use navigation tasks throughout the paper because they are intuitive and it is easier to generate visualizations with them. From an agent's perspective, these tasks are not different from other tasks, the agent is always traversing an unknown state space. We review the extensions to more complex solutions when discussing relevant related work.

Besides the sections already discussed, we present the required background in Sections~\ref{sec:background} and~\ref{sec:successor_representation}, the related work in Section~\ref{sec:related_work}, and the conclusion in Section~\ref{sec:conclusion}. While the main contributions in Sections~\ref{sec:temporally_extended_exploration} and~\ref{sec:options_keyboard} are on presenting existing results under a single formulation, the results in Sections~\ref{sec:framework},~\ref{sec:experiments_temporally_extended_exploration},~\ref{sec:illustration_rod_cycle}, and~\ref{sec:experiments_ok} are novel and have not been presented anywhere~else.

\section{Background}~\label{sec:background}

In this section we introduce the formalism behind reinforcement learning and the options framework~\citep{Precup00,Sutton99}. 
Throughout this paper, as a convention, we indicate random variables by capital letters (e.g., $S_t$, $R_t$), vectors by bold lowercase letters (e.g., $\boldsymbol{\theta}, \boldsymbol{\phi}$), matrices by bold capital letters (e.g., $\mathbf{P_\pi}, \mathbf{\Psi_\pi}$), functions by non-bold lowercase letters (e.g., $v$, $q$), and sets with a calligraphic font (e.g., $\mathscr{S}, \mathscr{A}$).

\subsection{Reinforcement Learning}

Reinforcement learning (RL) is a problem formulation that allows us to tackle sequential decision making problems. In RL we consider an agent interacting with an unknown environment in a sequential manner, aiming to maximize cumulative reward. We often assume that the environment can be modeled as a finite Markov decision process (MDP). An MDP  is formally defined as a 4-tuple $\langle\mathscr{S}, \mathscr{A}, p, r\rangle$. Starting from state $S_0 \in \mathscr{S}$, at each time step $t$ the agent takes an action $A_t \in \mathscr{A}$, to which~the environment responds with a state $S_{t + 1} \in \mathscr{S}$, according to a transition probability kernel $p(s' | s, a) \doteq \Pr(S_{t+1} = s' | S_t = s, A_t = a)$, and with a bounded reward signal $R_{t+1} \in \mathbb{R}$, with $r(s, a)$ indicating the expected reward for a transition from state $s$ under action $a$, that is, $r(s, a) \doteq \mathbb{E}[R_{t+1} \ | \ S_{t} = s, A_t = a]$. 

The agent's goal is to learn a policy $\pi : \mathscr{S} \times \mathscr{A} \rightarrow [0,1]$ that maps each state to a probability distribution over actions. Specifically, the agent seeks a policy that maximizes, in expectation, the (discounted) cumulative sum of rewards, also known as \emph{return}, defined~as \begin{eqnarray}{G_t = \sum_{k=0}^{\infty} \gamma^k R_{t+k+1}},\end{eqnarray} with $\gamma \in [0, 1)$, the discount factor, defining the relative value of future rewards.

In this paper we focus on value-based methods. To obtain the policy $\pi$, we estimate the state-value function, $v_\pi : \mathscr{S} \rightarrow \mathbb{R}$, or the state-action value function, $q_\pi : \mathscr{S}\times \mathscr{A} \rightarrow \mathbb{R}$. The value of a state $s$ when following a policy $\pi$, $v_\pi(s)$, is defined to be the return from that state: $v_\pi(s) \doteq \mathbb{E}_\pi\big[G_t | S_t = s \big]$. The state-action value function is defined similarly, but it takes into consideration the action taken, that is, $q_\pi(s, a) \doteq \mathbb{E}_\pi\big[G_t | S_t = s, A_t = a \big]$, where the expectation in both definitions is with respect to the policy $\pi$ and the probability kernel~$p$. Importantly, these functions can be defined recursively~\citep{Bellman57}, for example:
\begin{eqnarray}
v_\pi(s) &=& \sum_{a} \pi(a|s) \sum_{s'}\sum_{r} p(s', r | s, a) \Big [ r + \gamma v_\pi(s') \Big ]. \label{eq:bellman_v}
\end{eqnarray}

These equations can also be written in matrix form. The state-value function, for example, can be defined with ${\bf v_\pi}$, ${\bf r} \in \mathbb{R}^{|\mathscr{S}|}$, and $\mathbf{P_\pi} \in \mathbb{R}^{|\mathscr{S}| \times |\mathscr{S}|}$: 
\begin{eqnarray}
{\bf v_\pi} = {\bf r} + \gamma \mathbf{P_\pi} {\bf v_\pi} = (I - \gamma \mathbf{P_\pi})\inv {\bf r}, \label{eq:value_function}
\end{eqnarray}
where $\mathbf{P_\pi}$ is the transition probability induced by $\pi$, i.e., $\mathbf{P_\pi}(s, s') = \sum\nolimits_a \pi(a|s) p(s'|s, a)$.

In the RL problem we assume the agent does not know $\mathbf{P_\pi}$ nor $\mathbf{r}$ beforehand. Instead, RL methods directly estimate $v_\pi$ or $q_\pi$ from samples $(s, a, r, s')$. Most approaches alternate between a \emph{policy evaluation} step, that is, estimating the value of the agent's current policy, and a \emph{policy improvement} step, which defines a new policy from these estimates: 
\begin{eqnarray}
\pi(s) \doteq \argmax_{a \in \mathscr{A}} Q(s, a). \label{eq:policy_improvement}
\end{eqnarray}

Q-Learning~\citep{Watkins92} is the most well-known algorithm for estimating the value of the optimal policy, $\pi_*$. It has the following update rule for the $q_{\pi_*}$ estimate, $Q$: \begin{eqnarray}
Q(S_t, A_t) \leftarrow Q(S_t, A_t) + \alpha \Big(R_{t+1} + \gamma \max_{a \in \mathscr{A}} Q(S_{t+1}, a) - Q(S_t, A_t) \Big),
\end{eqnarray}
where $\alpha$ is the algorithm's step-size parameter.

When the value of each state (or state-action pair) is individually stored, this is a \emph{tabular} method. Nevertheless, generalization is required, and desirable, in problems with large state spaces, where it is infeasible to learn an individual value for each state. This is done by parameterizing the function $V$ or $Q$ with a set of parameters ${\bm \theta}$. We write, given the parameters $\bm \theta$, $V(s; {\bm \theta}) \approx v_\pi(s)$ and $Q(s, a; {\bm \theta}) \approx q_\pi(s, a)$. In the past, a common approach was to use linear function approximation where $Q(s, a; {\bm \theta}) ={\bm \theta}^\top \bm{\phi}(s,a)$, in which ${\bm \theta}$ is a vector of weights and $\bm{\phi}(s,a)$ denotes a static feature representation of the state $s$ when taking action $a$. It is now common to use a neural network to compute a non-linear function approximation of the value function, an approach popularized by \citet{Mnih13,Mnih15} with Deep Q-Network (DQN). The study of algorithms that use neural networks as function approximators has since been dubbed \emph{deep reinforcement learning}.

\subsection{Temporal Abstraction in RL: The Options Framework}

Sequential decision making usually involves planning, acting, and learning about temporally extended courses of actions over different time scales. In reinforcement learning, \emph{options} are a well-known formalization of the notion of actions extended in time that allow us to represent courses of actions~\citep{Precup00,Sutton99}.

An option $\omega \in \Omega$ is a 3-tuple
\begin{eqnarray}
\omega = \langle \mathcal{I}_\omega, \pi_\omega, \beta_\omega \rangle,
\end{eqnarray}
where $\mathcal{I}_\omega \subseteq \mathscr{S}$ denotes the option's initiation set, $\pi_\omega : \mathscr{S} \times \mathscr{A} \rightarrow [0,1]$ denotes the option's policy, such that $\sum_a \pi_\omega(\cdot, a) = 1$, and $\beta_\omega : \mathscr{S} \rightarrow [0, 1]$ denotes the option's termination condition, that is, the probability that option $\omega$ will terminate at a given state. In this paper, we consider the \emph{call-and-return} option execution model in which a high-level policy, $\mu : \mathscr{S} \times \Omega \rightarrow [0,1]$, dictates the agent's behavior. Notice that the actions originally defined in the MDP are a special case of options, that is, $\mathscr{A} \subseteq \Omega$. Finally, we often write that an agent \emph{follows or takes} an option $\omega$, meaning that the agent, in a state in $\mathcal{I}_\omega$, commits to act according to the option's policy, $\pi_\omega$, until its termination condition is satisfied. To distinguish between options and actions, we often refer to the actions originally defined in the problem formulation as \emph{primitive actions}.

Options have different use cases, including planning, exploration, and credit assignment. In this paper we mostly focus on exploration. Specifically, we discuss the \emph{option discovery} problem, which consists in discovering useful options from the agent's stream of experience. In other words, we discuss different algorithms that, given a set of samples $(s, a, s', r)$, autonomously define extended courses of actions, represented by an initiation set, a policy, and a termination condition, such that they allow for temporally-extended exploration. The algorithms we discuss can all be cast as part of a general framework for option discovery, which we discuss in the next section. In subsequent sections we show how different algorithms instantiate this framework.

\section{A Framework for Option Discovery from Representation Learning}\label{sec:framework}

We first introduce a general approach for option discovery that is driven by the representation learning process. It follows a constructivist approach~\citep{Piaget63} depicted as a cycle in which options discovered from previous iterations act as a scaffold for more complex behaviors discovered in subsequent iterations. The framework depicted in Figure~\ref{fig:option_discovery_cycle} distills the main steps of this cycle. Note that, while we present these steps sequentially, they can be executed concurrently at different time scales. Below we further discuss each step.\\

\begin{figure}[h]
    \centering
    \includegraphics[width=0.8\columnwidth]{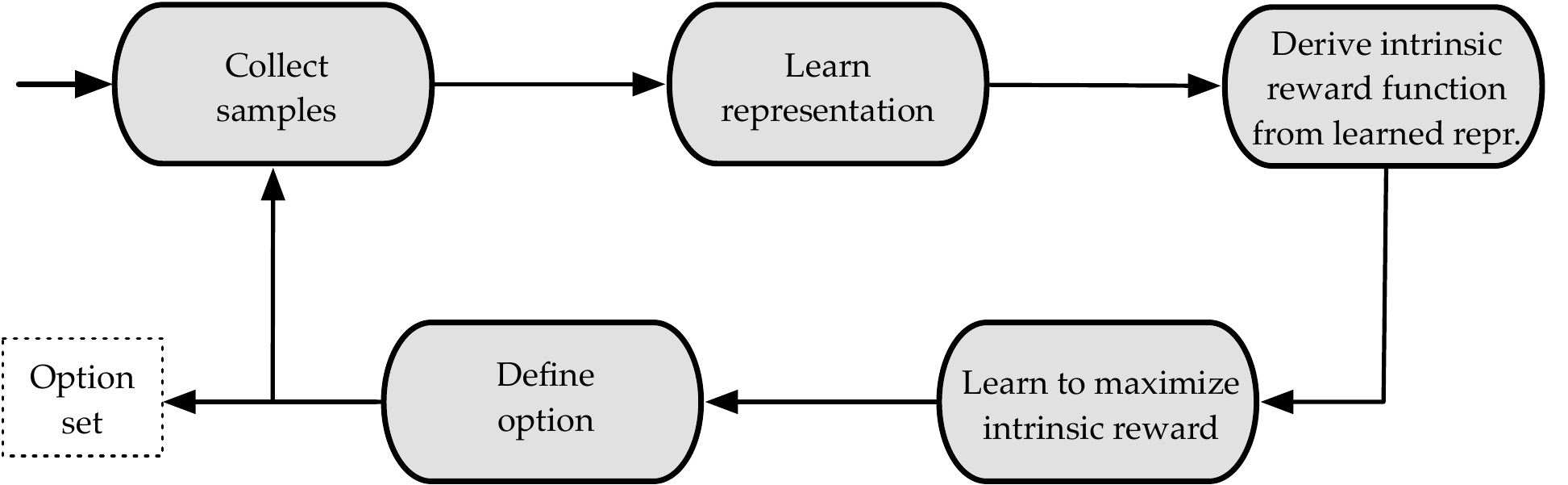}
    \caption{
    Representation-driven Option Discovery (ROD) cycle~\citep{Machado19}. The option discovery algorithms discussed in this paper can be seen as instantiating this cycle. The incoming arrow to \emph{Collect samples} depicts the start of the process. The arrow from \emph{Define option} to \emph{Option set} highlights the output generated by the ROD cycle. Note that other generated artifacts can also be used by the agent outside the ROD cycle, such as the learned representation.}~\label{fig:option_discovery_cycle}
\end{figure}

\emph{Collect samples:} The first step in each iteration of the representation-driven option discovery (ROD) cycle is to have the agent collect data in the form of trajectories. Selecting actions uniformly at random is an obvious first choice for the agent's policy. Once options have been identified, more possibilities for this step become available.\\

\emph{Learn representation:} In most problems of interest, the agent should learn a representation of its environment while acting in the world. Methods that are reward agnostic can be easier to implement, especially in early learning.
%are often more general, as they do not need to rely on elaborate exploration strategies.
In this paper, we focus on the \emph{successor representation} as the output of this step, because, as we discuss in the next section, it naturally captures the dynamics of the environment.\\

\emph{Derive an intrinsic reward function from the learned representation:} After a representation is learned, the agent can use it to define an intrinsic reward function which an option could maximize. The algorithms we discuss here either use spectral analysis or some clustering of the successor representation to define this \emph{intrinsic} reward function. The first is often associated with more efficient exploration while the latter usually leads to more efficient credit assignment.\\

\emph{Learn to maximize intrinsic reward:} Once a representation has been learned and the intrinsic reward function has been defined, the agent needs to learn to maximize the (discounted) sum of these rewards, which is a standard reinforcement learning problem.  The learned policy is the policy of this new option being discovered. This can be done in parallel for multiple options, with off-policy learning, which allows one to learn about policies that are different from the policy that generated the observations.\\

\emph{Define option:} Finally, there are different ways to define the option's initiation set and termination condition, which give rise to different algorithms~\citep[e.g.,][]{Jinnai19,Machado17}. We discuss several possibilities in Sections~\ref{sec:temporally_extended_exploration} and~\ref{sec:experiments_temporally_extended_exploration} when introducing and evaluating instantiations of this framework. The output of this step can be immediately incorporated into the agent's option set, but it can also be used in the next iteration of the ROD cycle, ideally improving the data collection step, which then allows the discovery of more complex options.\\

As previously mentioned, the ROD cycle can be agnostic to the ultimate reward function that the agent may want to optimize. This is particularly important when aiming at discovering options for temporally-extended exploration. If the option discovery process consists in learning options that, for example, replicate observations (e.g., feature activation, state visitation), new options might allow the agent to better navigate in the environment by making events that were rare, or virtually impossible, more likely. Figure~\ref{fig:results_cycle_exploration}, adapted from \citeauthor{Jinnai20}'s~\citeyear{Jinnai20} work, presents a concrete example of this behavior in a task with a continuous state space. In Section~\ref{sec:illustration_rod_cycle} we present another illustration of multiple iterations of the ROD cycle that was generated by an algorithm we introduce in this paper.

\begin{figure}[t]
    \centering
    \includegraphics[width=\columnwidth]{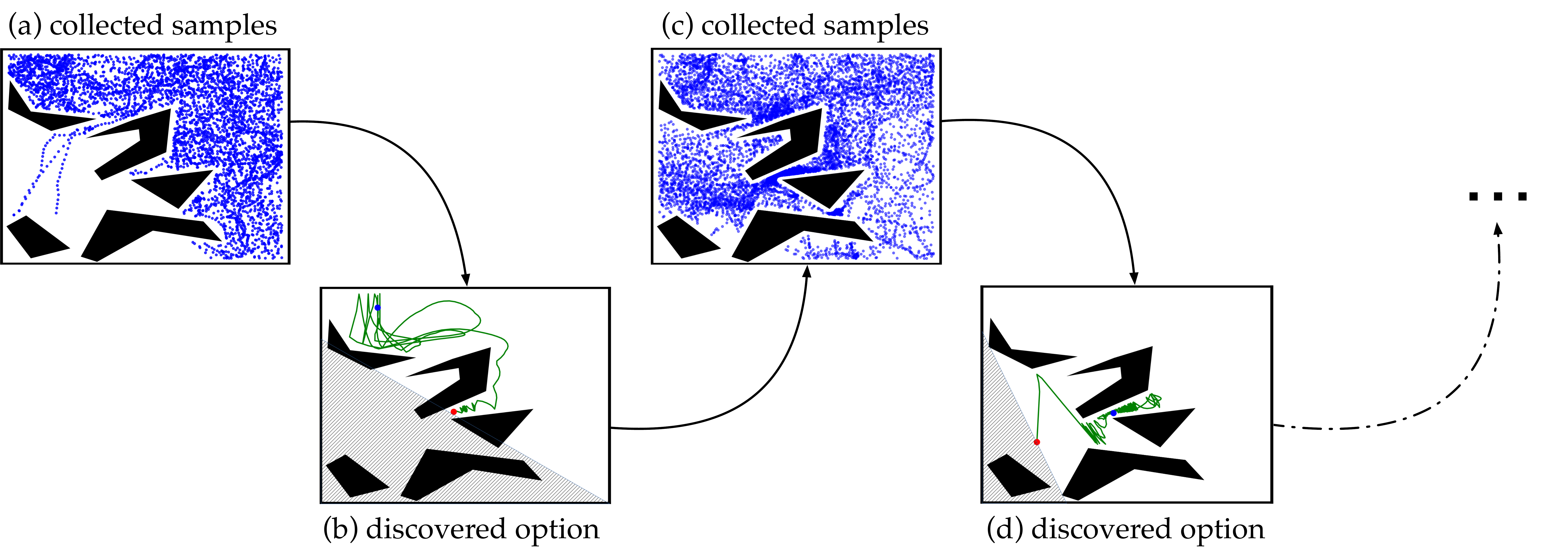}
    \caption{Two ROD cycles in the Pinball domain~\citep{Konidaris09}, originally presented by  \citet{Jinnai20}.  Blue dots denote state visitation, green lines the options' trajectories, and shaded regions the states in which the options terminate. The agent follows a random walk in every iteration. The agent first struggles to go through some narrow passages, as we can see in the samples collected at the first iteration of the ROD cycle (Fig.~\ref{fig:results_cycle_exploration}a). At the end of the cycle, the  discovered option takes the agent to the boundary of the region it visited (Fig.~\ref{fig:results_cycle_exploration}b). This new option enables the agent to visit regions that were originally hard to reach with a random walk, as we can see when looking at the samples collected in the second iteration (Fig.~\ref{fig:results_cycle_exploration}c). These samples were generated by the agent following a policy that  uniformly chooses between the primitive actions and the option discovered in the previous iteration. This process can, of course, be further refined. Fig.~\ref{fig:results_cycle_exploration}d depicts the option discovered at the end of the second iteration. }~\label{fig:results_cycle_exploration}
\end{figure}

\section{The Successor Representation}~\label{sec:successor_representation}

The successor representation \citep[SR;][]{Dayan93} is a classic method for automatically extracting a representation from the agent's observation, giving an answer to what representations one should use when performing function approximation. In this paper, we claim that the SR could be the natural substrate for temporal abstraction in reinforcement learning, as it is a representation learning method conducive to the discovery and use of options. The algorithms we present, one way or another, use the SR as their representation when instantiating the ROD cycle. This is due to the fact that it has a particular structure that captures the dynamics of the environment, as we discuss below.

\subsection{Tabular Setting}

The SR is a representation that captures the underlying environment dynamics. It does so by assigning similar values to states that are close in time; in other words, the notion of similarity between states is based on how similar their successor states are under a policy~$\pi$. Formally, the SR is defined to be the current and expected future occupancy of state $s'$ given the agent's policy $\pi$ and its starting state $s$.

The SR, with respect to a policy $\pi$, $\mathbf{\Psi_\pi}$, is defined as
\begin{equation}
\mathbf{\Psi_\pi}(s, s') = \mathbb{E}_{\pi, p} \Bigg[\sum_{t=0}^\infty \gamma^t \mathbbm{1}_{\{S_t = s'\}} \Big| S_0 = s \Bigg],
\end{equation}
where $\mathbbm{1}$ denotes the indicator function and $\gamma \in [0, 1)$. Thus, each state $s$ is represented as an $|\mathscr{S}|$-dimensional vector whose $i$-th component is the expected discounted visitation to each state in the environment. Figure~\ref{fig:sr} illustrates this concept. 

\begin{figure}[t]
   \begin{center}
    \includegraphics[width=0.95\columnwidth]{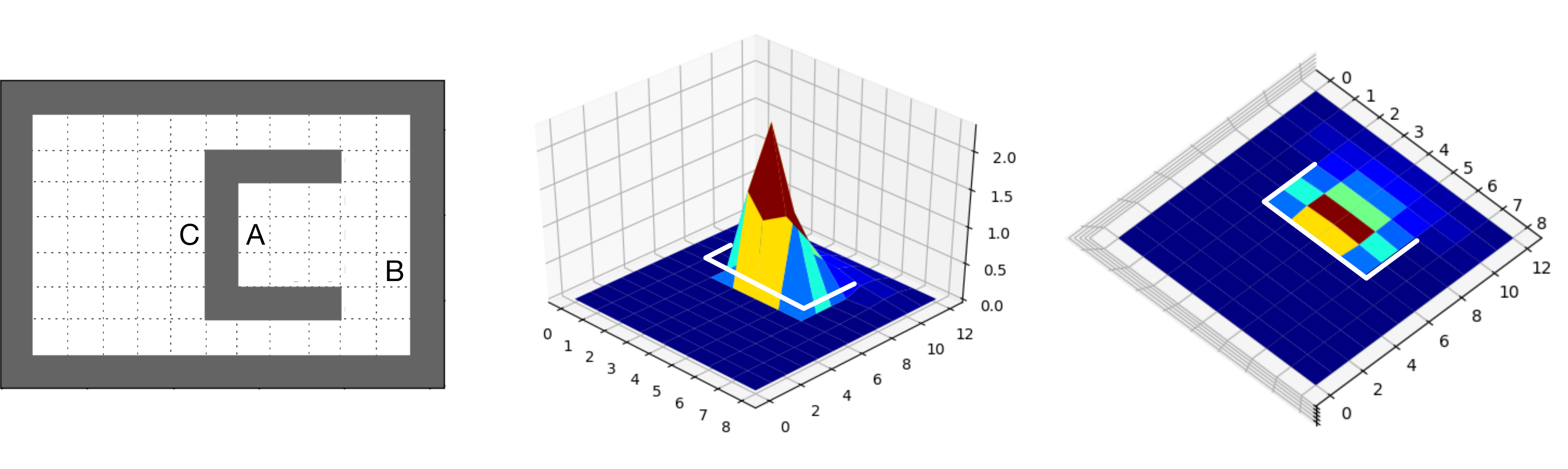}
     \centering
    \caption[Successor representation for the uniform random policy.]{
    Example similar to \citeauthor{Dayan93}'s~(\citeyear{Dayan93}) of the SR, w.r.t. the uniform random policy, of state A (left). Consider a navigation task where the agent has access to its $(x,y)$ coordinates. It is tempting to use some distance metric such as the Euclidean distance to define distance between states. However, if one considers the gray tiles to be walls, an agent in point A can reach point B much \emph{quicker} than point C. The SR captures this distinction, ensuring that, in this representation, point A is closer to point B than it is to point C. The plots of the SR were generated using a discretization of the grid, where each tile is a state. Red represents larger values while blue represents smaller values (states that are temporally further away). Recall the SR of a state, in the tabular case, is an $|\mathscr{S}|$-dimensional representation, thus allowing us to depict it as a heatmap over the state space.}\label{fig:sr}
    \end{center}
\end{figure}

Importantly, the SR can be estimated from samples with temporal-difference learning methods \citep{Sutton88}, where the reward function is replaced by the state occupancy:
\begin{eqnarray}
\hat{\Psi}(S_t, j) &\leftarrow& \hat{\Psi}(S_t,j) + \eta \Bigg[\mathbbm{1}_{\{S_t = j\}} + \gamma \hat{\Psi}(S_{t + 1}, j) - \hat{\Psi}(S_t, j) \Bigg],\label{eq:sr_td}
\end{eqnarray}
for all $j \in \mathscr{S}$, where $\eta$ is the step-size. Algorithm~\ref{alg:sr_closed_form} depicts an implementation of the SR. For clarity, because we refer back to this pseudo-code in later sections, we wrote it in such a way that transitions are stored in a data set $\mathcal{D}$. We do so to be able to write other components of the ROD cycle more compactly; but it is not necessary for any individual component.

\begin{algorithm}[t]
\caption{Successor Representation}\label{alg:sr_closed_form}
\KwInput{$\eta$ \Comment*[r]{Step-size}} 
\hspace{1.2cm} $\gamma \in [0, 1)$ \Comment*[r]{SR's discount factor}
\hspace{1.2cm} $\mathcal{D}$ \Comment*[r]{Data set with $(s,a, r, s')$ transitions}

\KwOutput{$\Psi \in \mathbb{R}^{|\mathscr{S}|\times|\mathscr{S}|}$ \newline}

$\Psi(i,j) \gets 0 \ \ \  \forall \ i, j < |\mathscr{S}|$ \Comment*[r]{Initialize the SR with zeros}

\For{$(s, a, s')$ \textrm{\normalfont \textbf{in}} $\mathcal{D}$}{
    \For{$i \gets 0$ \KwTo $|\mathscr{S}|$}{
        $\delta \gets \mathbbm{1}_{\{s = i\}} + \gamma \Psi(s', i) - \Psi(s, i)$\\
        $\Psi(s, i) \gets \Psi(s,i) + \eta \delta$
    }
}
\end{algorithm}

The SR can also be seen as a collection of general value functions~\citep{Sutton11} with a fixed discount factor and individual state visitation as cumulants. In this case, instead of seeing the SR as a representation, one would see it as a collection of predictions. The SR also corresponds~to~the~Neumann~series~of~$\gamma \mathbf{P}_\pi$:
\begin{eqnarray}
\label{eq:sr_matrix}
\mathbf{\Psi_\pi} = \sum_{t=0}^\infty (\gamma \mathbf{P_\pi})^t = (\mathbf{I} - \gamma \mathbf{P_\pi})\inv.
\end{eqnarray}

Thus, one can see the SR as an estimate of how often the agent expects to visit each state in the future, weighted by the discount factor. In fact, the SR is part of the solution when computing a value function (see Eq.~\ref{eq:value_function}):
\begin{eqnarray}
\label{eq:sr_reward}
{\bf v_\pi} = (\mathbf{I} - \gamma \mathbf{P_\pi})^{-1}  {\bf r} = \mathbf{\Psi_\pi} {\bf r}.
\end{eqnarray}
In words, one can compute the return by multiplying the SR and the estimates of the expected \emph{immediate} rewards in each state. This sum of weighted rewards relies on the SR to provide the weights, which encodes expected future state visitation.

The SR can be presented in multiple ways based on the several connections it has to other results in the field. The matrix in Eq.~\ref{eq:sr_matrix} is also known, for example, as the LSTD matrix~\citep{Lagoudakis03}. We further discuss some of these connections after presenting the generalization of the SR to the function approximation case.

\subsection{Successor Features: From States to Features}

The definitions given so far for the SR are limited to the tabular case. \emph{Successor features}~\citep[SFs;][]{Barreto17} are a generalization of the SR that can be extended to the function approximation setting.

Let $\bm{\phi}: \mathscr{S} \times \mathscr{A} \mapsto \mathbb{R}^d$ be a function that computes \emph{features}. The SFs of policy $\pi$ are
\begin{equation}
\bm{\psi}_\pi(s,a) \doteq \mathbb{E}_{\pi,p} \left[ \sum_{i=0}^{\infty} \gamma^{i} \bm{\phi}(S_{t+i}, A_{t+i}) \,|\, S_{t} = s, A_{t} = a \right].
\end{equation}
In words, $\bm{\psi}_{\pi, i}(s,a)$ encodes the discounted expected value of the $i$-th feature in the vector $\bm{\phi}(\cdot, \cdot)$ when the agent starts in state $s$, executes action $a$, and follows policy $\pi$ thereafter. The features $\bm{\phi}(\cdot, \cdot)$ can be either given to the agent or learned, as we discuss in Section~\ref{sec:related_work}. The update rule presented in Eq.~\ref{eq:sr_td} can be naturally extended to this definition.

SFs are a strict generalization of the SR. To see why this is so, suppose that $\mathscr{S}$ is finite and let $\bm{\phi}(s,a) = \bm{\phi}(s)$ for all $(s,a) \in \mathscr{S} \times \mathscr{A}$ (that is, $\bm{\phi}$ is a function of states only). Then, we can rewrite the definition of SFs in matrix form as 
\begin{equation}
    \label{eq:sfs_matrix}
\mathbf{\Psi_\pi} = \sum_{t=0}^\infty (\gamma \mathbf{P_\pi})^t \mathbf{\Phi} = (\mathbf{I} - \gamma \mathbf{P_\pi})\inv\mathbf{\Phi},
\end{equation}
where $\mathbf{\Phi} \in \mathbb{R}^{|\mathscr{S}| \times d}$ is a matrix encoding the feature representation of each state. When $d=|\mathscr{S}|$, if we define $\bm{\phi}_i(s_j) =\mathbbm{1}_{\{i=j\}}$, Eq.~\ref{eq:sfs_matrix} reduces to Eq.~\ref{eq:sr_matrix}. Again, this highlights the fact that the SR can be seen as the discounted state visitation distribution induced by policy $\pi$.

The connection between the SR and SFs also allows us to generalize Eq.~\ref{eq:sr_reward}. Assume there exists a $\bm{w} \in \mathbb{R}^d$ such that 
\begin{equation}
    \label{eq:reward_features}
r(s,a) = \bm{\phi}(s,a)^{\top}\bm{w} \text{ for all } (s,a) \in \mathscr{S} \times \mathscr{A}.
\end{equation}
Based on the definition of $q_\pi$ one can to show that~\citep{Barreto17} 
\begin{equation}
    \label{eq:sfs_q}
q_\pi(s,a) = \bm{\psi}_\pi(s,a)^{\top}\bm{w} \text{ for all } (s,a) \in \mathscr{S} \times \mathscr{A}.
\end{equation}
Again, when $d=|\mathscr{S}|$ and $\bm{\phi}_i(s_j,a)=\mathbbm{1}_{\{i=j\}}$ for all $a \in \mathscr{A}$, Eq.~\ref{eq:sfs_q} reduces to Eq.~\ref{eq:sr_reward}. Note that, once we have the SFs $\bm{\psi}_\pi$ of a policy $\pi$, Eq.~\ref{eq:sfs_q} allows us to instantaneously evaluate $\pi$ under any reward that can be represented as a linear combination of the features $\bm{\phi}$ (Eq.~\ref{eq:reward_features}).
This has been exploited in the past for transfer between tasks~\citep{Barreto17,Barreto18}.

\subsection{Properties of the Eigenvectors of the Successor Representation}\label{subsec:other_sr}

As aforementioned, the SR is present in several RL algorithms, either explicitly or implicitly. An important result for this paper is that the eigenvectors of the SR are equivalent to proto-value functions \citep[PVFs;][]{Mahadevan05,Machado18b}. We rely on this result to be able to also discuss algorithms originally presented under the PVFs formalism. The properties of the eigenvectors of the SR (i.e., PVFs) are particularly relevant to option discovery methods. A discussion of PVFs and the formal equivalence result between them and the eigenvectors of the SR is available in Appendix~\ref{app:pvfs}.

\begin{figure}[t]
    \centering
    \hspace{0.7cm}
    \begin{subfigure}[b]{0.17\columnwidth}
    \center
    \raisebox{1mm}{
        \includegraphics[width=\columnwidth]{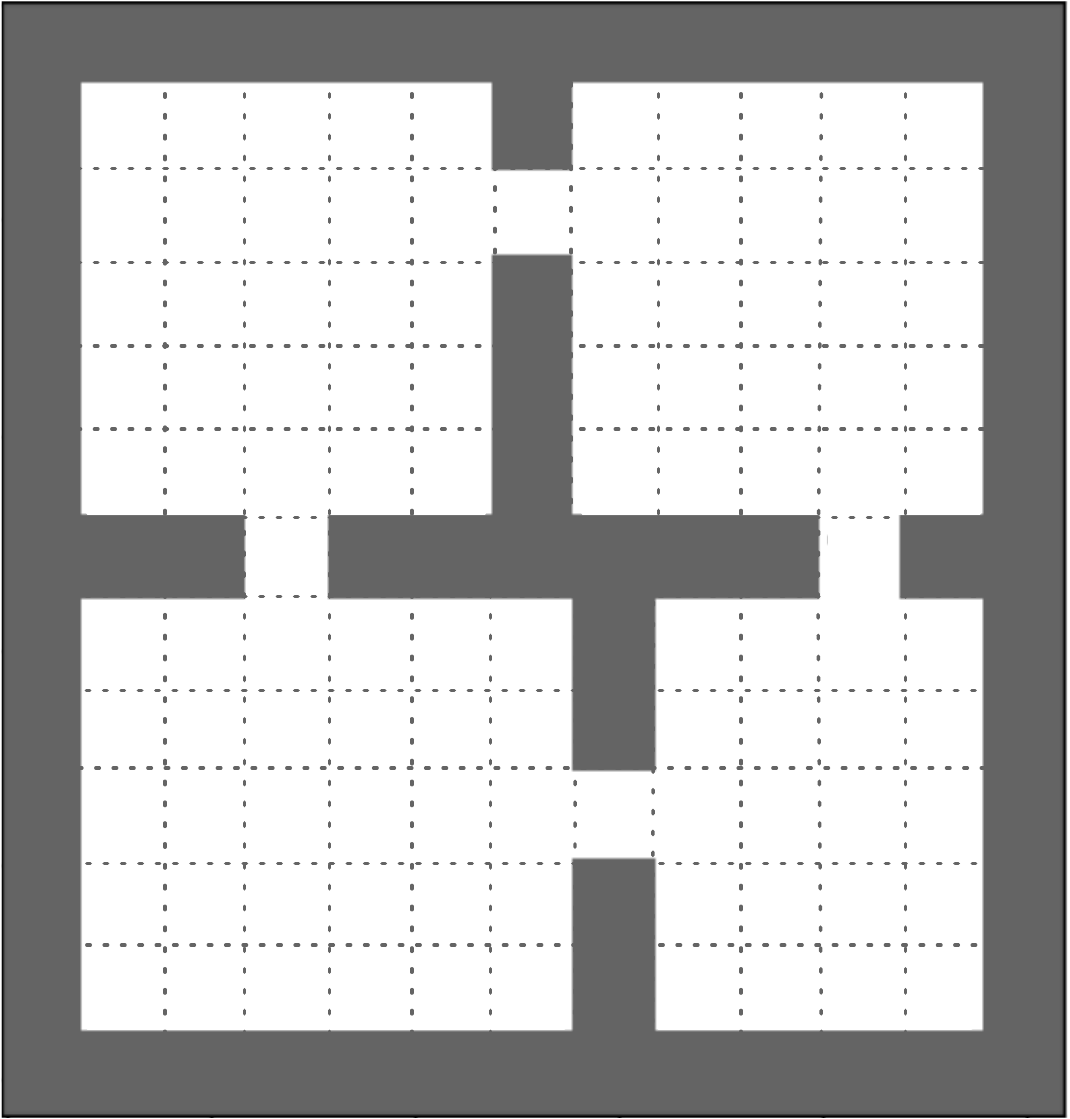}
        }
        \caption{Four-room}
    \end{subfigure}
    ~
    \begin{subfigure}[b]{0.23\columnwidth}
    \center
        \includegraphics[width=\columnwidth]{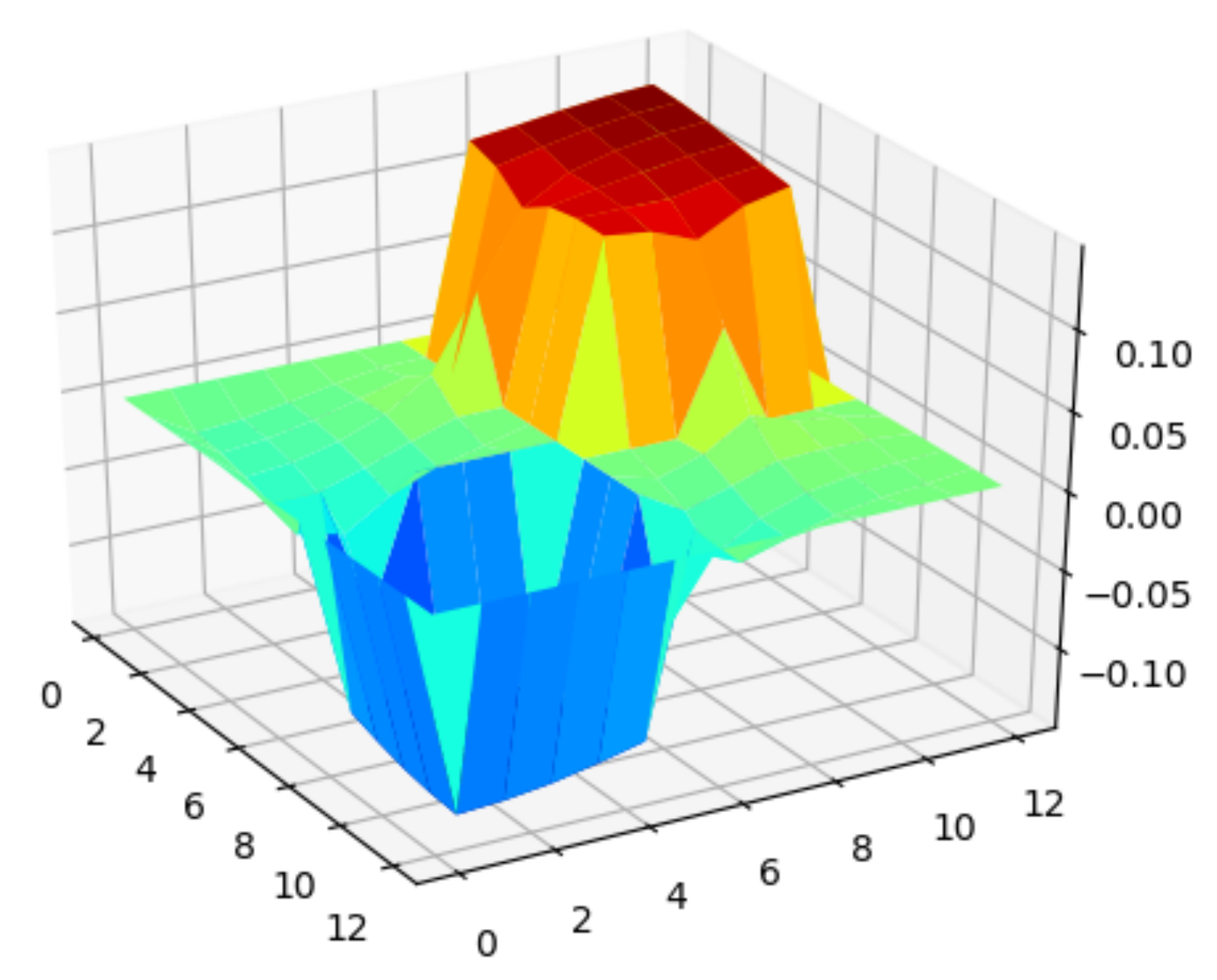}
        \caption{First PVF}
    \end{subfigure}
    ~
    \begin{subfigure}[b]{0.23\columnwidth}
    \center
        \includegraphics[width=\columnwidth]{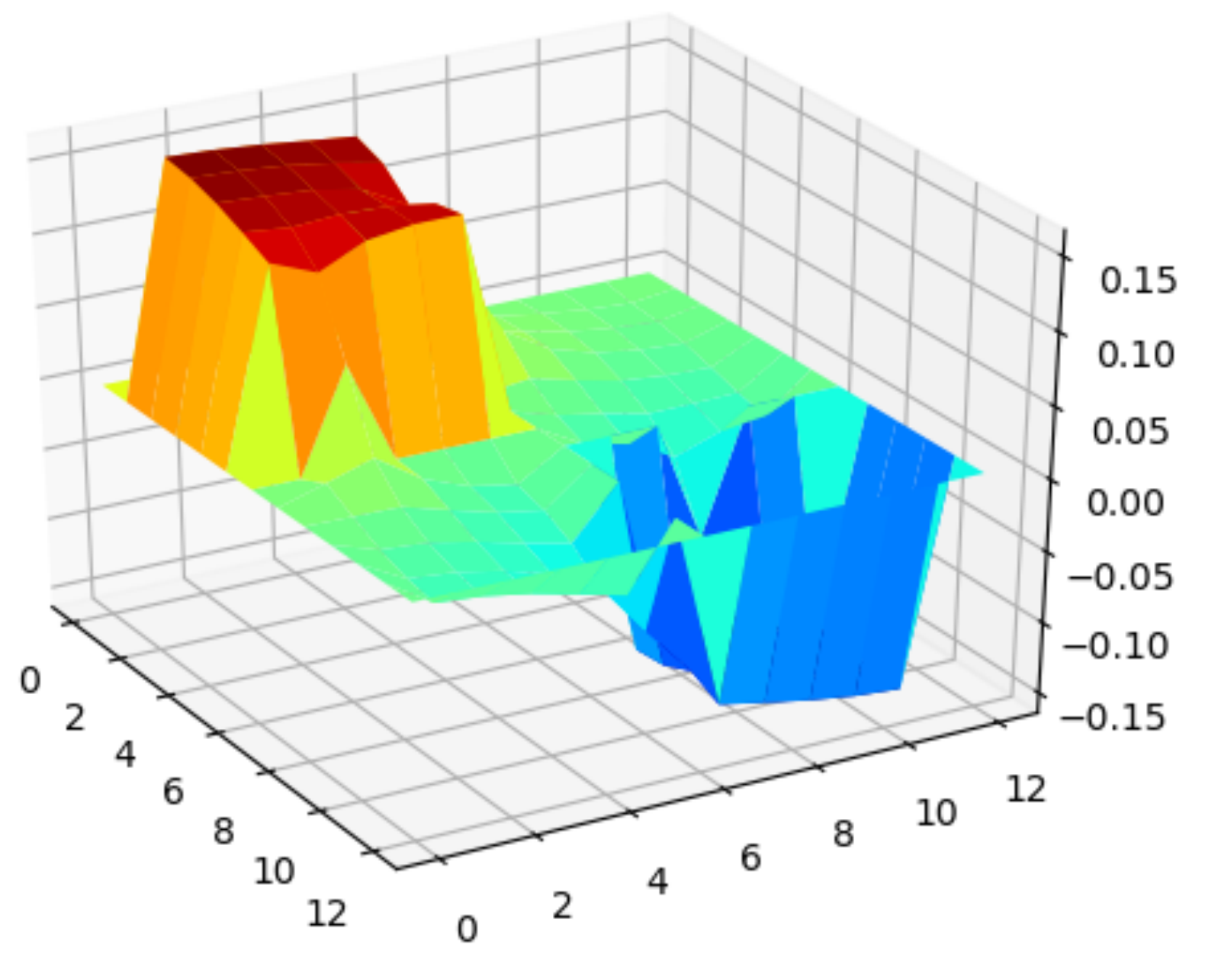}
        \caption{Second PVF}
    \end{subfigure}
    ~
    \begin{subfigure}[b]{0.23\columnwidth}
    \center
        \includegraphics[width=\columnwidth]{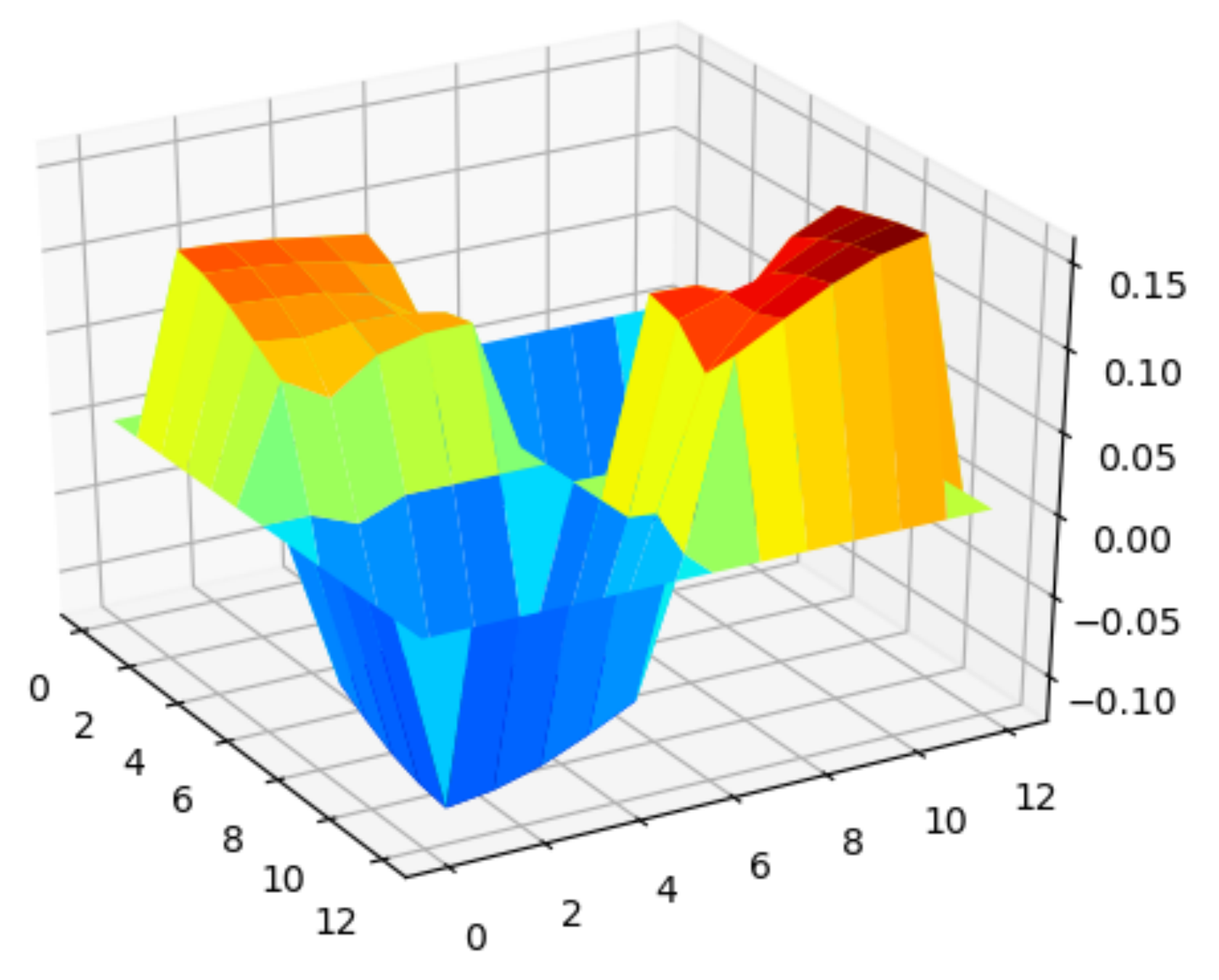}
        \caption{Third PVF}
    \end{subfigure}
\caption[First three PVFs in the four-room domain.]{First three PVFs in the \textbf{(a)} four-room domain. Gray squares represent walls and white squares represent accessible states. Four actions are available: \emph{up}, \emph{down}, \emph{right}, and \emph{left}. The transitions are deterministic and the agent is not allowed to move into a wall. \textbf{(b-d)} These plots depict the first, second, and third eigenvectors associated with each state. The axes are rotated for clarity. The bottom left corner of the four-room domain is the state closer to the reader. }\label{fig:pvfs}
\end{figure}

PVFs, and consequently the eigenvectors of the SR, capture temporal properties of an environment, with different eigenvectors capturing different time-scales of diffusion, a hallmark of Fourier analysis. This can be seen in Figure~\ref{fig:pvfs}, which depicts the first three PVFs in the four-room domain. The first eigenvectors capture longer time-scales, such as in Figures~\ref{fig:pvfs}b and~\ref{fig:pvfs}c, in which the biggest difference is seen between the two states that are furthest apart: the different diagonals in the environment. On the other hand, eigenvectors with corresponding larger eigenvalues\footnote{When using proto-value functions, the order of the eigenvectors is flipped, explaining why eigenoptions use the eigenvectors with corresponding lowest eigenvalues (see Theorem~\ref{th:equivalence} in Appendix~\ref{app:proofs}).} exhibit shorter time-scales, such as in Figure~\ref{fig:pvfs}d, in which the period of the curves depicted is already shorter---the distance between the states with largest and smallest values is smaller. Similar to value functions, PVFs are smooth, with the value of each state being a function of its neighbors. The methods we discuss below heavily benefit from such properties.\\

\section{Temporally-Extended Exploration}~\label{sec:temporally_extended_exploration}

To exemplify how the SR can be used for option discovery, we now discuss different methods that instantiate the ROD cycle using the SR as representation. In this section, we focus on discovering options useful for temporally-extended exploration. Specifically, we consider options that can be used by the agent, alongside primitive actions, when following a random walk. When an option is selected, instead of a primitive action, the agent acts according to the option's policy until its termination condition is satisfied.

Temporally-extended exploration with options is based on the intuition that agents explore the environment more effectively if they operate at a higher-level of abstraction. When acting according to options' policies, agents exhibit more directed behavior in contrast to the aimless dithering commonly observed when selecting primitive actions uniformly at random \citep{Dabney20, Jinnai19, Machado16, Machado17}. Intuitively, if one needs to explore, say, a building, it makes more sense to do so in terms of rooms than in terms of motor twitches. In fact, such an approach was used as a solution for the exploration problem in one of the recent high-profile success stories in artificial intelligence: the deployment of a reinforcement learning algorithm to navigate superpressure balloons in the stratosphere~\citep{Bellemare20}.

In this section, we discuss \emph{eigenoptions}~\citep{Machado17,Machado18b} and \emph{covering options}~\citep{Jinnai19, Jinnai20}. They instantiate the ROD cycle in different ways while using the SR as representation. These methods, and others \citep[e.g.,][]{Bar20}, use the eigenspectrum of the learned representation to guide the option discovery process. They are representative of a class of methods that is  motivated by the fact that the eigenvectors of the SR naturally encode the diffusion properties of the environment due to their close relationship to proto-vaue functions, as discussed in the previous section.

In this and in the next section, we use the differences between eigenoptions and covering options to highlight (and evaluate) some of the overall choices one can make when designing option discovery methods. Specifically, we discuss different ways of designing the intrinsic reward function that guides learning of the options' policy, the definitions of the options' initiation set and termination condition, and how these choices impact the design of online versions of these algorithms.

\subsection{Eigenoptions}

Eigenoptions are options defined by the eigenvectors of the SR.\footnote{\cite{Machado17} originally defined eigenoptions in terms of PVFs. Later, \cite{Machado18b} used the equivalence between PVFs and the eigenvectors of the SR to generalize eigenoptions to the setting in which the representation, that is, the SR, is learned online.} Each eigenvector assigns an intrinsic reward to every state in the environment. An eigenoption is an option, defined with respect to a specific eigenvector, that takes the agent to the state with largest (or smallest) value. Intuitively, what an eigenoption does is to ensure that there is an option that directly takes the agent to the state that was originally difficult to get to.

This description becomes clearer with an example. In the four-room domain, the second largest eigenvector of the SR, defined w.r.t. a uniform random policy, is depicted in Figure~\ref{fig:eigenoption} (the top eigenvector of the SR is constant). In this environment, the two states that are furthest apart are the states diagonally opposed in the corners, which is what the depicted eigenvector captures. The corresponding eigenoptions take the agent to one of those states.\\

\begin{figure}[t]
    \centering
    \includegraphics[width=0.9\columnwidth]{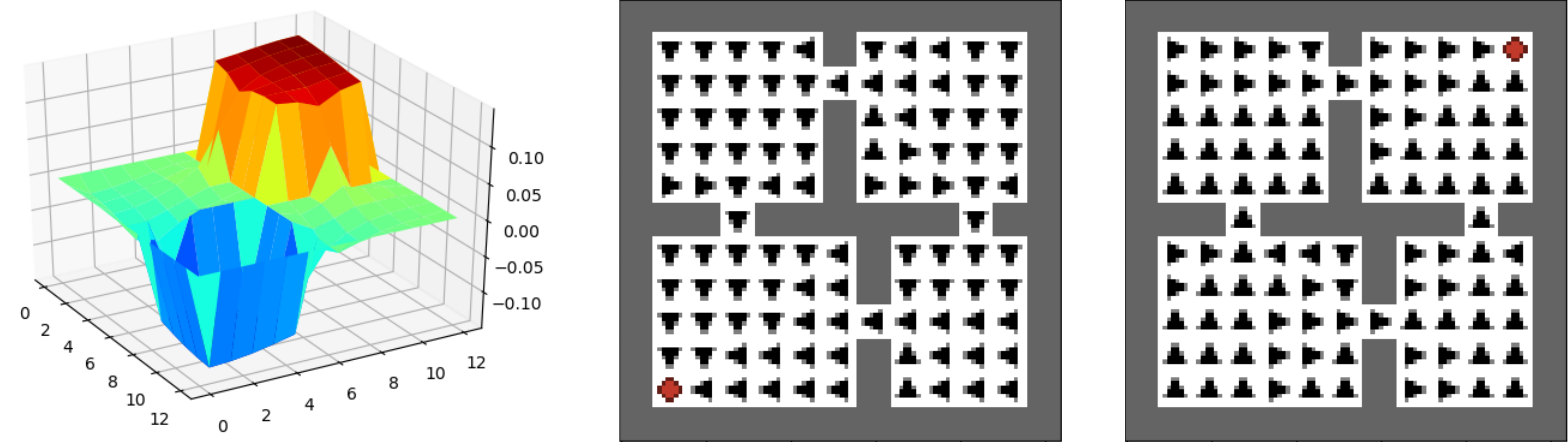}
    \caption{Second eigenvector of the SR and the two corresponding eigenoptions in the four-room domain. The actions of the deterministic policy are depicted as arrows and the state in which the option terminates is depicted in red (in the corners).}\label{fig:eigenoption}
\end{figure}

\emph{Learning the eigenoptions' policies.} To specify an option we need to define its policy, initiation set, and termination condition. An eigenoption's policy is defined by the intrinsic reward function $r^{\bf{e}}(s, s')$, obtained from the eigenvector ${\bf{e}} \in \mathbb{R}^{d}$ of the SR. It is defined as 
\begin{eqnarray}
r^{\bf{e}}(s, s') &=& {\bf e}^\top \big(\bm{\phi}(s') - \bm{\phi}(s)\big), \label{eq:eigenpurpose}
\end{eqnarray}
where $\bm{\phi}(s)$ denotes the feature representation of state $s$. Notice that, in the tabular case, if we define $\bm{\phi}(s)$ to be an $|\mathscr{S}|$-dimensional one-hot encoding of state $s$, $r^{\bf{e}}(s, s')$ becomes $${r^{\bf e}(s, s') = {\bf e}(s') - {\bf e}(s)}.$$

Importantly, the sign of an eigenvector is arbitrary. Thus, as aforementioned, this reward function can be interpreted as incentivizing the agent to either go to the highest or to the lowest point of the graph shown in Figure~\ref{fig:eigenoption}. In this example, these correspond to the top right and bottom left states in the environment.

We learn the option's policy in a newly defined MDP $\mathcal{M}^{\bf e} = \langle \mathscr{S}, \mathscr{A} \cup \{\bot\}, r^{\bf e}, p, \gamma \rangle$, where the state space and the transition probability kernel remain unchanged. The reward function is defined as in Eq. \ref{eq:eigenpurpose} and the action set is augmented by the action \emph{terminate} ($\bot$), which allows the agent to leave $\mathcal{M}^{\bf e}$ at no cost, returning to the original MDP in the same state it was when it left $\mathcal{M}^{\bf e}$. The discount factor can be chosen arbitrarily, although it impacts the timescale the option encodes by defining how myopic the agent will be w.r.t. $r^{\bf{e}}$. In this formulation, eigenoptions ignore the reward signal provided by the original MDP, but it is not difficult to imagine extensions in which this is not the case~\citep[c.f.,][]{Liu17,Sutton22}.

With $\mathcal{M}^{\bf e}$, we define the state-value function $v_{\pi}^{\bf e}(s)$, for policy $\pi$, as the expected value of the cumulative discounted intrinsic reward if the agent starts in state $s$ and follows policy $\pi$ until termination.  We define the action-value function $q_{\pi}^{\bf e}(s, a)$ similarly. The optimal value function for any intrinsic reward function obtained through $\bf{e}$ is then described as
$$ v_*^{\bf e}(s) = \max_{\pi} v{_{\pi}^{\bf e}}(s) \ \ \ \ \ \mbox{and} \ \ \ \ \ \ q_*^{\bf e}(s, a) = \max_{\pi} q{_{\pi}^{\bf e}}(s,a). $$

The option’s policy, $\pi_*^{\mathbf{e}}$, is the optimal policy w.r.t. the intrinsic reward function $r_i^\mathbf{e}$, i.e.,

$$\pi_*^{\bf e}(s) = \argmax_{a \in \mathscr{A}} q_{*}^{\bf e}(s, a).$$
Thus, finding the option's policy $\pi_*^{\bf e}$ becomes a traditional RL problem, with a different reward function. Importantly, the reward received for transitioning from one state to another is rarely zero, avoiding challenging exploration issues caused by sparse non-zero rewards.\\

\emph{Defining the eigenoptions' initiation sets and termination conditions.} When defining the MDP to learn the option's policy, we augment the agent's action set with the \emph{terminate} action so the agent can terminate the option. The termination condition is deterministic, with eigenoptions terminating when the agent is unable to accumulate further positive intrinsic rewards. This happens when the agent reaches the state with largest value assigned by the corresponding eigenvector (or a local maximum when $\gamma < 1$). Any subsequent sum of rewards will be at most zero. We formalize this condition by defining 
$$q_{\pi}^{{\bf e}+}(s, a) = 
      \left\{\begin{array}{rl} 
            q_{\pi}^{\bf e}(s, a), &\mbox{if $a \in \mathscr{A}$},\\
            0, &\mbox{if $a = \bot$,}
      \end{array}\right.
$$ 
where $q_{\pi}^{{\bf e}+}$ denotes the state-action value function, defined by $\pi$ and $r^{\bf e}$, augmented by the terminate action, which has value zero. When the terminate action is selected, control is returned to the higher level policy. In summary, because we break ties in favour of the terminate action, an option following a policy $\pi^{\bf e}$ terminates in state $s$ when $q_{\pi}^{\bf e} (s, a) \leq 0$ for all $a \in \mathscr{A}$.

The initiation set is defined to be the complement of the set of states in which an option terminates, i.e., all states in which there exists an action $a \in \mathscr{A}$ s.t. $q_{\pi}^{\bf e} (s, a) > 0$.~In~summary, an eigenoption consists of a policy $\pi^{{\bf e}+}$, which is augmented by the \emph{terminate} action,
\begin{eqnarray}
\pi^{{\bf e}+}(s) = \argmax_{a \in \mathscr{A} \cup \{\bot\}} q_{\pi}^{{\bf e}+}(s,a). \label{eq:extended_eigenoption_policy}
\end{eqnarray}
The termination and initiation sets are implicitly defined. That is, for an eigenoption $\omega^\mathbf{e}$, $\beta_{\omega^\mathbf{e}}(s) = 1$ if $q_{\pi}^{\bf e} (s, \cdot) \leq 0$, and $\beta_{\omega^\mathbf{e}}(s) = 0$ if there is an action $a \in \mathscr{A}$ such that $q_{\pi}^{\bf e} (s, a) > 0$. In this second case, $s \in \mathcal{I}_{\omega^\mathbf{e}}$. In practice, this means the option terminates in states that are assigned the (locally) largest values in the eigenvector; the option can be initialized in all other states. Algorithm~\ref{alg:eigenoptions_closed_form} and~\ref{alg:eigenoptions_online}, in Appendix~\ref{sec:pseudocode}, summarize the presentation of eigenoptions when computed in both closed-form and online. Importantly, for any eigenoption, there is always at least one state in which it terminates. The theorem below formalizes this result.

\begin{restatable}[\citeauthor{Machado17} \citeyear{Machado17}]{theorem}{thtermination}
Let $\omega = \langle\mathcal{I}_\omega, \pi_\omega, \beta_\omega\rangle$ denote an eigenoption. In a finite and ergodic MDP with $\gamma \in [0, 1)$, there is at least one state $s \in \mathscr{S}$ such that $\beta_\omega(s) = 1$.
\end{restatable}
\begin{proof}
See Appendix~\ref{app:proofs}.
\end{proof}

This result is due to the fact that the reward function resembles a potential function~\citep{Ng99}. It is not clear if it would be possible to obtain such a natural termination criteria if the agent maximized, for example, the feature itself, as in $r^{\bf{e}}(s) = {\bf e}^\top \bm{\phi}(s)$.

\begin{wrapfigure}{c}{0.45\columnwidth}
  \centering
  \includegraphics[width=0.45\columnwidth]{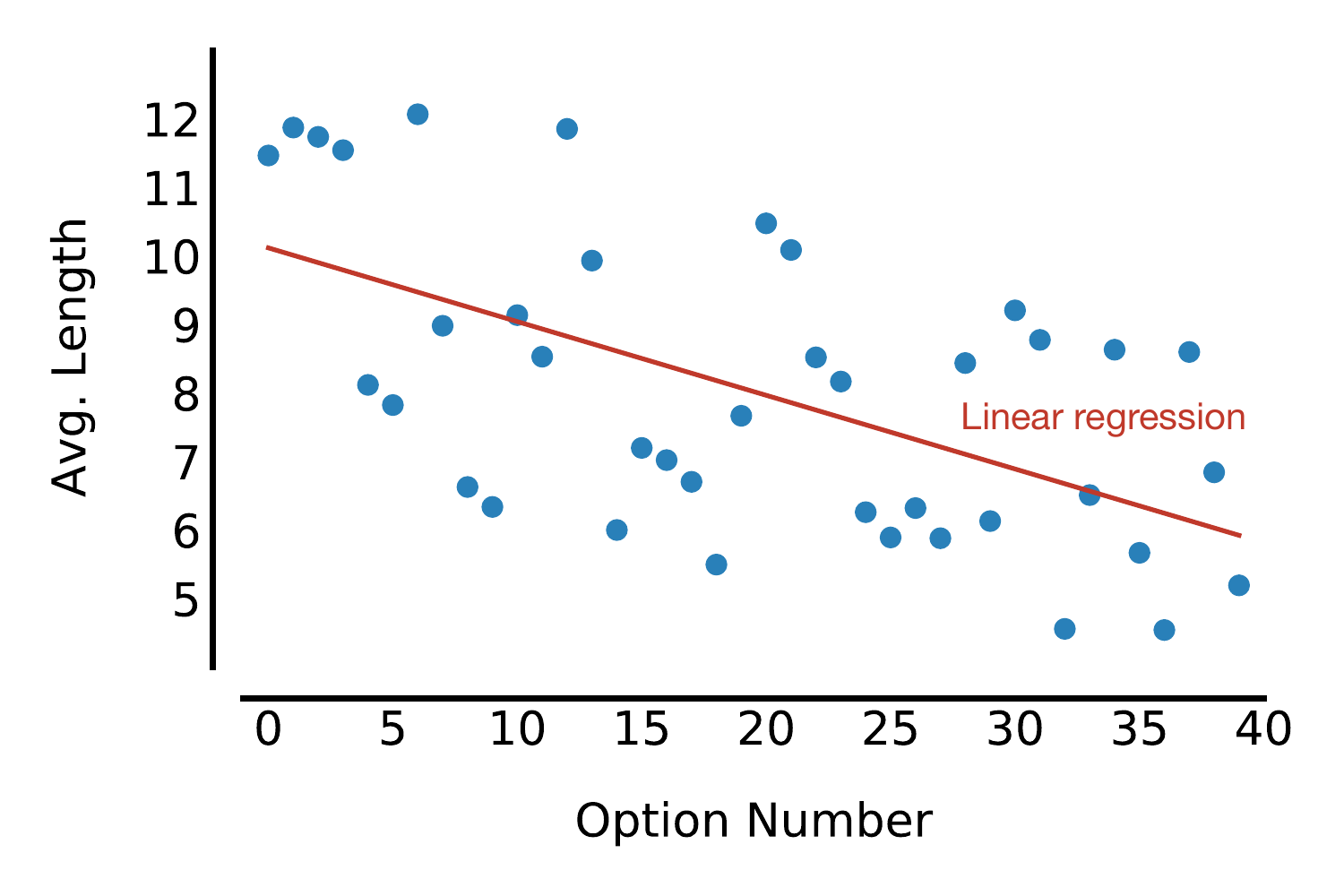}
    \caption{Avg. length of different eigen-options by the order they are discovered. The linear regression emphasizes the downward trend on the option length.}\label{fig:eigenoptions_length}
\end{wrapfigure}
Eigenoptions have several interesting properties that allow them to improve exploration. One of them is that different eigenoptions have different durations, effectively letting the agent operate at different time scales. We exemplify this in Figure~\ref{fig:eigenoptions_length}, where we show how the eigenoptions discovered first tend to be longer  (i.e., those discovered from eigenvectors with corresponding larger eigenvalues). We overlay a linear regression of the data to emphasize this trend. Besides their duration, eigenoptions can be easily sequenced, and they are task-independent because, as the SR, they do not depend on information related to the reward function. Nevertheless, there could be twice as many eigenoptions as states in the environment and it is not clear how to choose the number of desired options. Covering options \citep{Jinnai19}, discussed next, are an attempt to address this issue.

\subsection{Covering Options}
Covering options are options defined by the bottom eigenvector of the graph Laplacian (i.e., the first PVF), that is, the eigenvector with the smallest corresponding eigenvalue.\footnote{The eigenvector of the graph Laplacian with corresponding smallest eigenvalue is constant, so we refer to the second smallest one here.} They have the explicit goal of minimizing the environment's expected cover time---the number of steps required for a random walk to visit every state \citep{Broder89}. The intuition behind them is similar to the one behind eigenoptions, as they exploit the fact that the bottom eigenvector of the graph Laplacian captures the states that are furthest apart in the environment. Nevertheless, a covering option is defined as a \emph{point option}~\citep{Jinnai19b} connecting only two states. It explicit adds an edge to the graph representing the underlying MDP to connect the two furthest vertices in that graph, in an attempt to shrink its diameter. They are obtained with an iterative procedure in which, after the discovery of each option, the environment's underlying graph is updated and the option discovery procedure is executed again. Covering options are naturally defined by multiple iterations of the ROD cycle. Eigenoptions can be seen as adding multiple edges to the graph, connecting every state in the initiation set to one of the terminal states. Covering options are discovered one at a time and connect only two states.

The description of covering options becomes clearer with an example. In the four-room domain, the bottom eigenvector of the graph Laplacian, defined w.r.t. a uniform random policy, is depicted in Figure~\ref{fig:covering_option_example}, as well as its corresponding covering options---note the eigenvector is equivalent to the one depicted in Figure~\ref{fig:eigenoption}. In  this  environment,  the two states that are furthest apart are the states diagonally opposed in the corners. Covering options connect those two states, allowing the agent to easily navigate between them.
\vspace{0.3cm}

\begin{figure}[t]
    \centering
    \includegraphics[width=0.9\columnwidth]{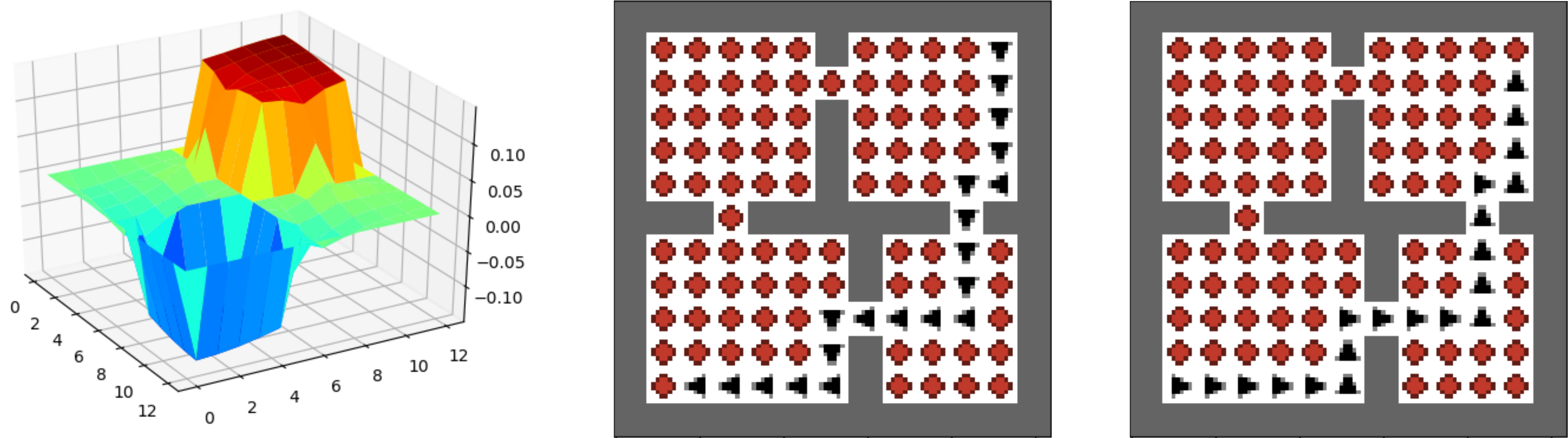}
    \caption{Second eigenvector of the graph Laplacian and corresponding covering options in the four-room domain. The actions of the deterministic policy are depicted as arrows and the state in which the option terminates is depicted in red.}\label{fig:covering_option_example}
\end{figure}

\emph{Learning the covering options' policies.} Formally, a covering option's policy is defined by the intrinsic reward function $r^{\bf{e}}(s, s')$ obtained from the non-constant eigenvector ${\bf e} \in \mathbb{R}^d$ of the graph Laplacian that has the smallest corresponding eigenvalue. It is defined as 
$$r^{\mathbf{e}}(s, s')=\begin{cases}
1,&\mbox{if}\quad \arg\max_{i} \ \mathbf{e}_i = s',\\
0, &\mbox{otherwise}.
\end{cases}
$$

In words, arriving at the state in which the eigenvector has the corresponding largest value leads to a reward of 1. The same reward function is applied to the negation of the eigenvector so the agent can also learn an option that leads to the state with smallest value.

This reward function is arguably simpler than the reward function used by eigenoptions, but it does not provide the agent with a gradient to follow when learning the option's policy, which might lead to exploration issues. We further discuss the pros and cons of each choice in the next sections.\\

\emph{Defining the covering options' initiation sets and termination conditions.} Because these are point options, their initiation sets and termination conditions are trivially defined. The initiation set has a single state, the one with corresponding smallest value in the considered eigenvector. Covering options terminate with probability $1$ when they reach the state with corresponding largest value. Formally, for a covering option $\omega^\mathbf{e}$,

\begin{align*}\beta_{\omega^\mathbf{e}}(s)=\begin{cases}
1,&\mbox{if}\quad \arg\max_i \mathbf{e}_i = s,\\
0, &\mbox{otherwise}.
\end{cases} && \mathcal{I}_{\omega^\mathbf{e}}(s)=\begin{cases}
1,&\mbox{if}\quad \arg\min_i \mathbf{e}_i = s,\\
0, &\mbox{otherwise}.
\end{cases}
\end{align*}
In which we overloaded the notation to have $\mathcal{I}_{\omega^\mathbf{e}}(s)$ denoting whether state $s$ belongs to the initiation set of option $\omega^\mathbf{e}$ or not. As before, we evaluate $\pm \mathbf{e}$. This can obviously lead to problems when reaching an individual state is unlikely (e.g., large/continuous state-spaces). Recently, \cite{Jinnai20} extended this notion to a region of the state-space.\\

\emph{Iterative discovery of covering options.} While the formulation of eigenoptions does not prescribe the number of options to be used, there always exist exactly two covering options associated with an SR. In order to get more covering options, one must compute an updated SR induced by the newly-added options and then use its dominant eigenvectors to compute two new options. This process can be repeated multiple times, resulting in an iterative procedure in which two covering options are added at each iteration. Thus, while eigenoptions work well with a single iteration of the ROD cycle, learning multiple options in parallel, covering options require a much lighter iteration, but several iterations of the ROD cycle. Algorithm~\ref{alg:co_closed_form} and~\ref{alg:co_online}, in Appendix~\ref{sec:pseudocode}, summarize the presentation of covering options when computed in both closed-form and online.

Importantly, every iteration is guaranteed to improve the upper bound of the expected number
of steps an agent would need, when following a uniform random policy, in order to visit every state in the environment. We refer the reader to the work by \citet{Jinnai19} for the formal statement behind this result.

Covering options provide a simpler approach for option discovery. This approach discovers a fixed number of options at each iteration, the intrinsic reward function the agent needs to maximize is trivial, as well as the definitions of the initiation set and termination condition. Nevertheless, it is not clear that this simplicity implies better empirical performance. Is it empirically better to use the full spectrum of the graph Laplacian or only the bottom eigenvector? Do covering options face an exploration issue when learning the option's policy? Are point options more effective than options defined over most of the state space? In the next section, we empirically evaluate these different choices option discovery methods have.

\section{Evaluation of Temporally-Extended Exploration with Options}~\label{sec:experiments_temporally_extended_exploration}

\citet{Machado17,Machado18b} and \citet{Jinnai19,Jinnai20} have already presented empirical evidence about the efficacy of the approaches discussed in the previous section. Thus, we focus instead on comparing components of these approaches. These comparisons elicit fundamental questions surrounding option discovery, such as the impact of different initiation sets, how to best use a limited number of samples observed by the agent, and on the trade-offs posed by different numbers of options available to the agent at a given moment. We evaluate these choices in the context of temporally-extended exploration, when using options to provide persistent behaviors inside uniformly random policies in both closed-form and online settings. This analysis also allows us to highlight different choices that can be made when instantiating the ROD cycle.

Because we are interested in the agent's exploration capabilities in a given environment, irrespective of a specific reward function, we use the \emph{diffusion time} as the main evaluation metric. The diffusion time reflects the expected number of decisions\footnote{We use the term \emph{decisions} instead of steps to estimate the likelihood that a sequence of random choices (i.e., options or primitive actions) will lead to the desired state.} required to navigate between any two states while following a random walk~\citep{Dayan92,Machado17}. A small expected number of decisions implies  that the agent is more likely to reach any state with a random policy. The diffusion time captures the agent's ability to learn about the structure of the environment, and it is particularly relevant in settings without a single, fixed task, such as continual~\citep{Brunskill14,Mankowitz18}, multitask~\citep{Teh17}, and transfer learning~\citep{Taylor09}. Moreover, the diffusion time allows us to  summarize more easily a large number of results. While we will use this metric throughout most of the paper, in Section~\ref{sec:reward_max} we also evaluate how eigenoptions and covering options impact the speed at which agents accumulate rewards.

In tabular domains, we can easily compute the diffusion time with dynamic programming. We do so by defining an MDP in which the value function of a state $s$, under a uniform random policy, encodes the expected number of decisions required to navigate between state $s$ and a chosen goal state. We compute the expected number of decisions between any two states by setting one as the goal and checking the value of the other. We then compute the expected number of decisions across the entire state space by averaging over all possible pairs of the initial state and the goal state. The MDP in which the value function of state $s$ encodes the expected number of decisions from $s$ to a goal state has $\gamma = 1$ and a reward function of $+1$ at every decision that does not lead to the goal state. Policy evaluation computes the expected number of decisions the agent will take before arriving to the goal state. When computing the diffusion time, we iterate over all possible states, defining them as terminal states. We report both average and median as summary statistics of diffusion time.

To provide a direct comparison between eigenoptions and covering options, in Section~\ref{sec:diffusion_time} we evaluate the agent's diffusion time induced by these approaches. In Section~\ref{sec:reward_max}, we check how the insights obtained from this comparison impact reward maximization. In Section~\ref{sec:ablation_exploration}, we evaluate the impact of different approaches for defining the options' initiation sets and termination conditions, and of using the whole eigenspectrum of the successor representation. We report our last set of experiments in Section~\ref{sec:online_experiments}, when we discuss how these algorithms behave in the online setting. We summarize our findings in Section~\ref{sec:summary_experiments_exploration}.

\subsection{Diffusion Time of Eigenoptions and Covering Options}~\label{sec:diffusion_time}

We report the diffusion time obtained by eigenoptions and covering options in Figure~\ref{fig:comparison_diffusion_time}. We use the four-room domain~\citep{Sutton99}, which we implemented with Gym-Minigrid~\citep{gym_minigrid}. For simplicity, we consider deterministic transitions and we compute both sets of options in closed form.

\begin{figure}[t]
     \centering
         \includegraphics[width=\textwidth]{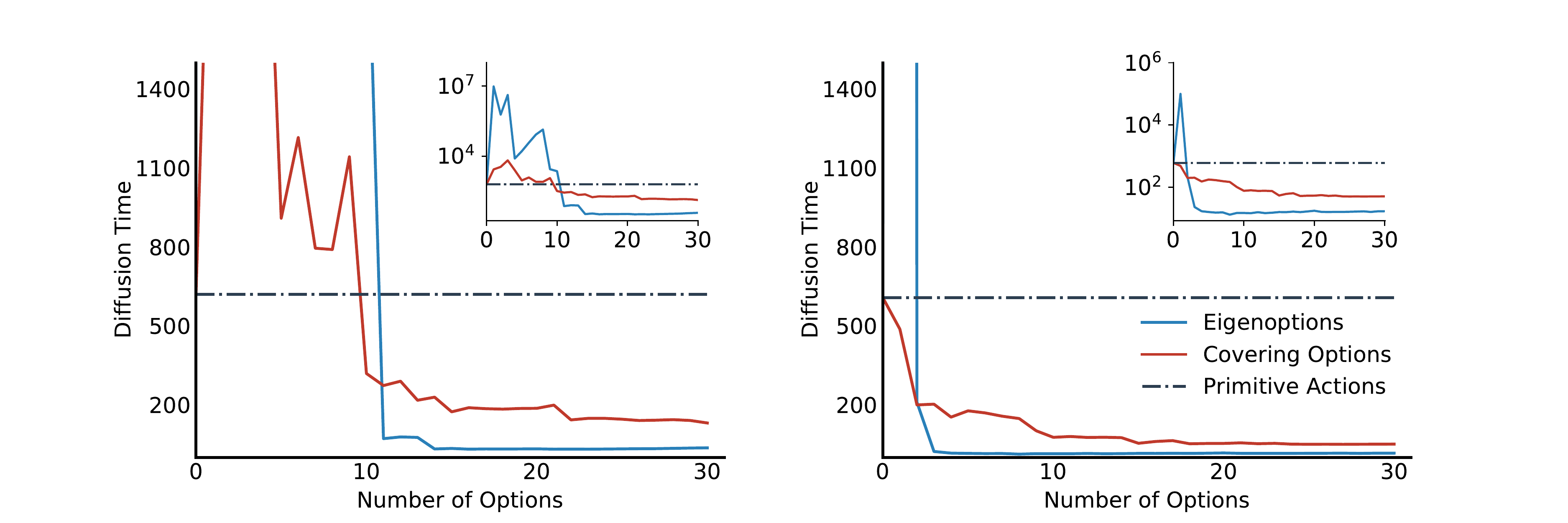}
     \caption{Comparison between the agent's average (left) and median (right) diffusion time in the four-room domain when using covering options and eigenoptions. The inset plots depict, in log-scale, the range in which the diffusion time lies.}
     \label{fig:comparison_diffusion_time}
\end{figure}

The average diffusion time we report for eigenoptions is similar to what \citet{Machado17} reports. Covering options had not been evaluated under this metric yet. We observe that after a sufficient number of options becomes available, eigenoptions lead to a smaller diffusion time, although they lead to much worse performance with fewer options. Figure~\ref{fig:comparison_diffusion_time} shows that once options start to outperform primitive actions, given the same number of options, eigenoptions lead to a lower diffusion time. This suggests that  it is better to use options derived from, say, the top ten eigenvectors of the SR, as done by eigenoptions, than to use the top eigenvector of the SR in ten successive iterations, as done by covering options.\footnote{In these experiments, for covering options, we used the eigenvectors of the Laplacian because we are computing them in closed form and also because, differently than the SR, it does not lead to eigenvectors with a complex component. We further discuss this issue in Section~\ref{sec:online_experiments} and Appendix~\ref{app:mismatch_co_sr_pvfs}.} This is supported by other results that also show the benefits of looking beyond the first eigenvector of time-based representations~\citep{Bar20}.

Aside from the number of eigenvectors they use, another difference between eigenoptions and covering options is with respect to how the options are defined. While an eigenoption is defined in the whole state space, the initiation set of a covering option has a single state. The main consequence of this is that at the beginning, before the options sufficiently connect the underlying graph, eigenoptions tend to create sink states. The agent needs to take enough random actions to reach a state while also \emph{not} choosing, by chance, options that move it away from the desired state. This explains the much worse performance of eigenoptions when fewer options are added. Point options have a less ubiquitous effect.

There is a stark difference between the reported average and median diffusion time. The average diffusion time is heavily impacted by outliers while the median diffusion time is more representative of performance for a random pair of states. As options are added, most of the states become easier to reach. However, without enough options, it is very difficult to reach states that are far from the options' terminal states. The average diffusion time captures this dichotomy. It has very high values at first because the options (mainly eigenoptions) make some states almost impossible to reach. The median is not impacted by these worst-case scenarios. It is particularly impressive to see that covering options, even with a single option, are already capable of reducing an agent's median diffusion time.

\subsection{Maximizing Rewards after Temporally-Extended Exploration}~\label{sec:reward_max}

A natural question to ask is how eigenoptions and covering options impact an agent's ability to accumulate reward in a single, fixed task. We answer this question in the setting in which agents learn the values of primitive actions while being allowed to also act according to the options' policies. In this setting, options do not incur an additional cost in terms of sample complexity~\citep{Brunskill14} because the agent does not learn state-option values. Nevertheless,  the agent is still able to exhibit temporally extended-exploration when acting with respect to an option's policy. Specifically, we use Q-learning~\citep{Watkins92} to learn the agent's policy with $\epsilon$-greedy exploration. When an exploratory step is chosen, the agent randomly chooses amongst all primitive actions and options with equal probability. When an option is selected, the agent acts according to the option's policy until it terminates, while updating, off-policy, the value of each of the primitive actions taken. Additionally, this evaluation allows us to evaluate the setting in which the number of steps taken by the agent matter. While the diffusion time ignores the set of states visited by the agent when following an option, the off-policy updates we use do not.

We evaluated our agents in the four-room domain with the agent having access to a varied number of options. We used ten tasks defined by different, randomly sampled, start and goal states. Episodes were at most 1,000 steps long and we evaluated the agents for 50 episodes. The agent observes a reward signal of $0$ until reaching the goal, when it observes a reward signal of $+1$. The Q-learning parameters we use are $\alpha=0.1$, $\gamma=0.9$, and $\epsilon=0.05$. We use the options pre-computed from the previous section, adding them to  the agent's option set according to their corresponding eigenvalue (or iteration), as previously discussed.

Figure~\ref{fig:comparison_return} depicts the performance of an agent augmented with eigenoptions or covering options. We chose a representative setting amongst the ten tasks. The results for the other nine tasks can be found in Appendix~\ref{sec:appendix_accum_return}. We observe that, for eigenoptions, as few as four options are enough to accelerate learning. We did not observe four eigenoptions hurting agent's performance in any of the ten tasks we randomly sampled. On the other hand, surprisingly, covering options did not improve the agent's performance. This was consistent across the ten tasks. We conjecture this is due to the sparse initiation set that reduces the effectiveness of covering options in the online setting because they are rarely sampled. Although this may be alleviated by the use of an agent that also learns option values, this strategy can have unforeseen consequences, such as a higher sample complexity for learning the optimal policy.

\begin{figure}[t]
     \centering
     \begin{subfigure}[b]{0.22\textwidth}
         \centering
         \raisebox{.3cm}{\includegraphics[width=\textwidth]{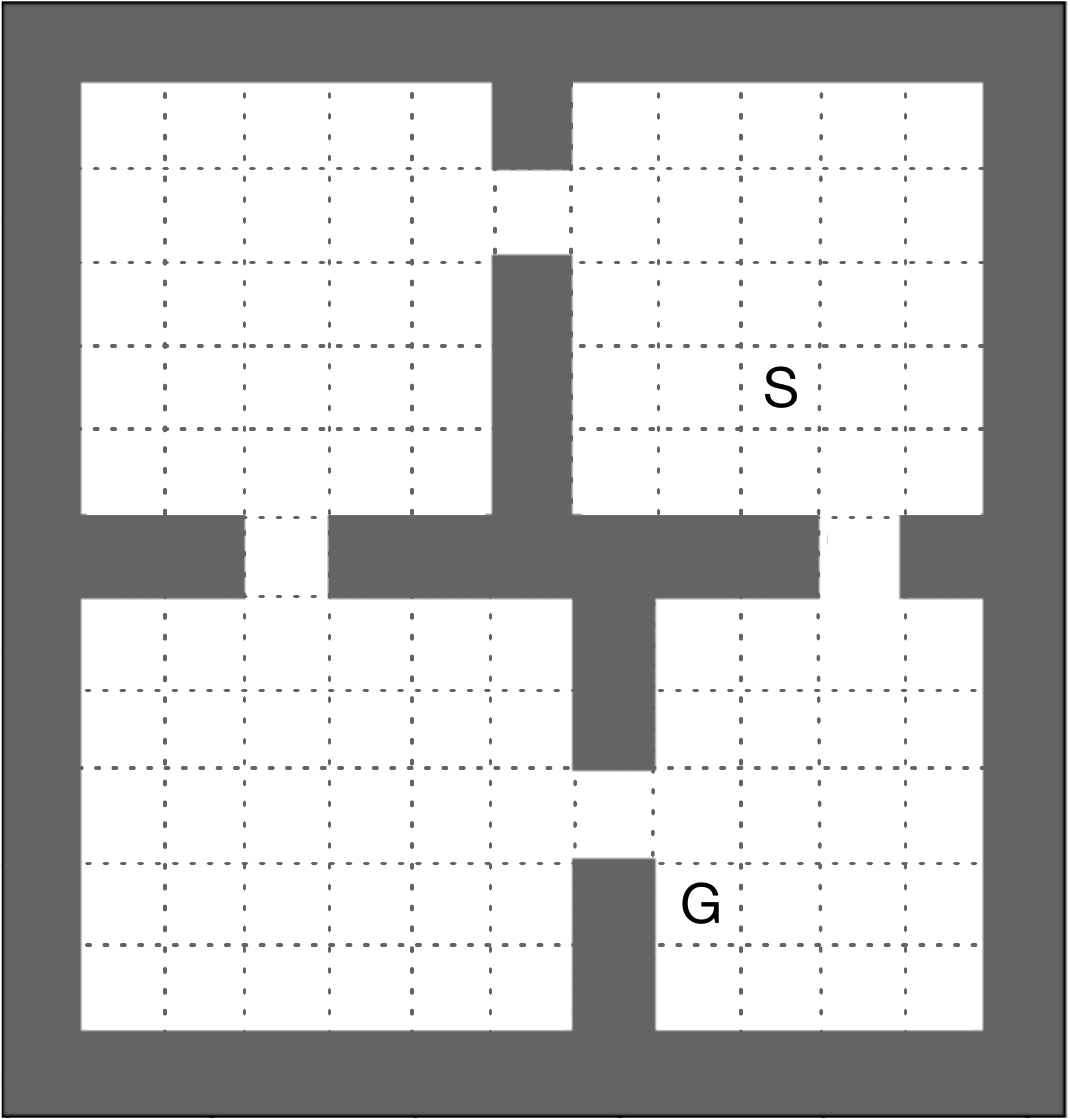}}
         \caption{Task}
     \end{subfigure}
     \hfill
     \begin{subfigure}[b]{0.37\textwidth}
         \centering
         \includegraphics[width=\textwidth]{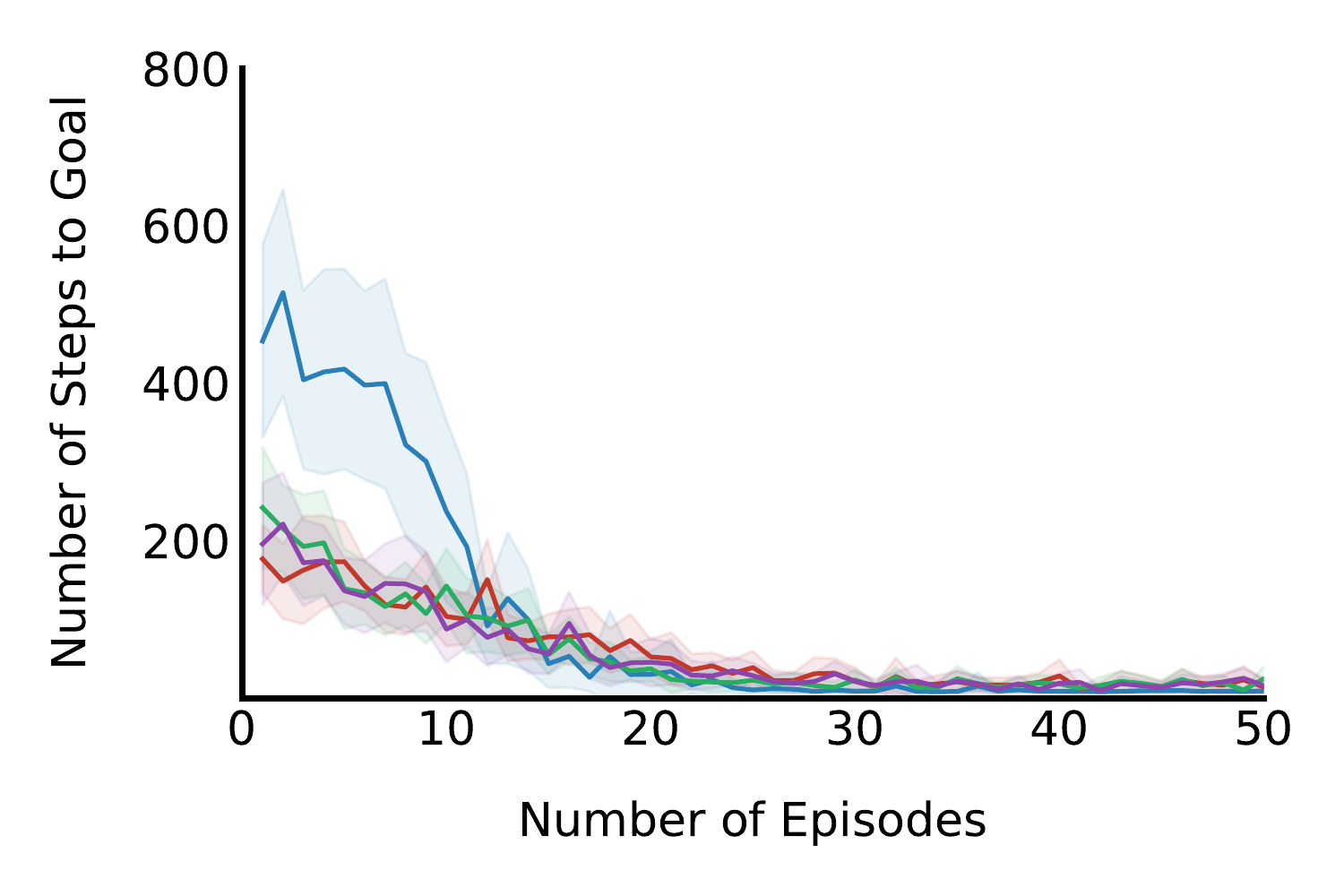}
         \caption{Eigenoptions}
     \end{subfigure}
     \hfill
     \begin{subfigure}[b]{0.37\textwidth}
         \centering
         \includegraphics[width=\textwidth]{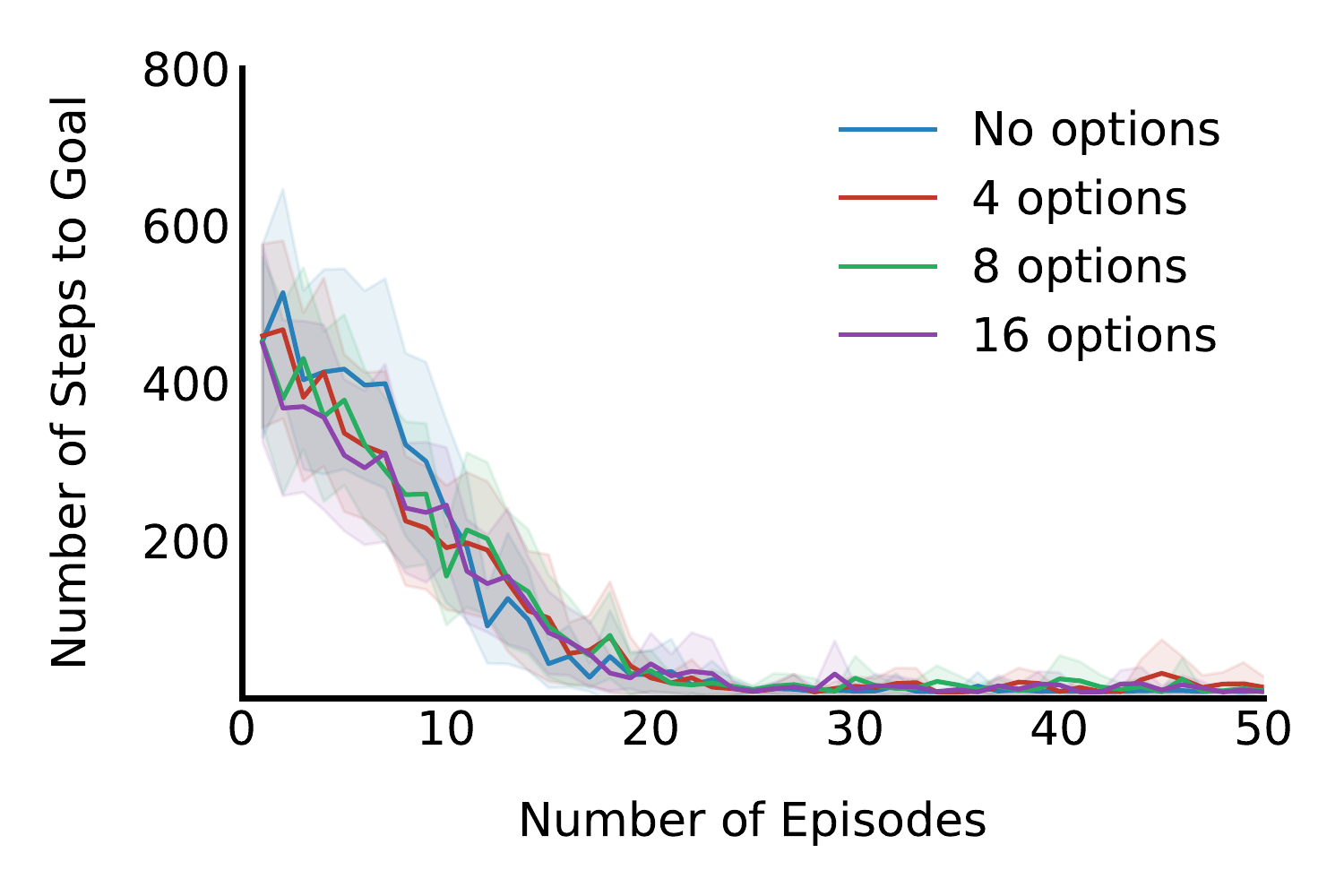}
         \caption{Covering options}
     \end{subfigure}
     \caption{Performance of Q-learning augmented with eigenoptions and covering options. In the task, \textsf{S} and \textsf{G} denote the start and goal state. Results are averaged across 50 runs and shaded regions denote a 99\% confidence interval. See text for details.}
     \label{fig:comparison_return}
\end{figure}

These results can also be interpreted in light of \citeauthor{Liu18}'s (\citeyear{Liu18}) work, which characterizes the properties of an environment that make exploration hard or easy. Informally, \citeauthor{Liu18} show that, if, asymptotically, a random walk can explore well, then it is possible to obtain a polynomial sample complexity bound for finite sample exploration. We conjecture eigenoptions transform the problem to one where random walks are more effective, making the problem more amenable to $\epsilon$-greedy exploration. Covering options, on the other hand, fail to make random walks more effective. This is  due to the fact that each option is available in a single state, and, when the agent happens to be in that state, it still needs to sample the option instead of primitive actions. Recent extensions of covering options to problems that require function approximation define the  initiation set as a region of the state space instead of a single state \citep{Jinnai20}.

\subsection{The Impact of Different Initiation Sets and Uses of the Eigenspectrum}~\label{sec:ablation_exploration}

In this section, we quantify the impact that different choices have when designing option discovery algorithms for temporally-extended exploration. We answer the following questions:
\begin{itemize}
  \setlength\itemsep{1pt}
    \item Is it better to have options that are broadly available, i.e., with large initiation sets?
    \item Is it beneficial to use more than one eigenvector of the SR when discovering options?
\end{itemize}

We ask these questions because eigenoptions and covering options are based on the same principles but, as previously discussed, their effectiveness is not the same. These questions are motivated by how these methods differ.
 
We use a factorial design to evaluate the impact of these different choices, which are outlined in Table~\ref{tab:variations}. We use the diffusion time induced by the discovered options as evaluation metric. As aforementioned, the diffusion time is task-agnostic and it concisely describes the effectiveness of different sets of options across multiple potential tasks by assessing how options capture relevant properties of the environment. We compare, given the same number of options, how the diffusion time induced by different methods varies.

\setlength{\dashlinedash}{0.5pt}
\setlength{\dashlinegap}{1pt}
\begin{table}[t]
    \small
    \centering
    \begin{tabular}{l c c}
    \toprule
    \multirow{2}{5cm}{\textbf{Algorithm description}} & \multirow{2}{2cm}{\textbf{Single iteration}} & \multirow{2}{3.7cm}{\textbf{Options broadly available (init. set)}}\\
                     & & \\
    \hline\\[-0.2cm] 
        
         Covering options (CO) & {\textsc{No}} &{\textsc{No}}\\[0.05cm]
         \hdashline\\[-0.35cm]
         CO w/ Broad Initiation Set & {\textsc{No}} &{\textsc{Yes}}\\[0.05cm]
         \hdashline\\[-0.35cm]
         Point-based Eigenoptions & {\textsc{Yes}} &{\textsc{No}}\\[0.05cm]
         \hdashline\\[-0.35cm]
         Eigenoptions & {\textsc{Yes}} &{\textsc{Yes}}\\[0.1cm]
    \bottomrule
    \end{tabular}
    \caption{The different dimensions of variation between eigenoptions and covering options. Covering options with a broad initiation set are defined such that the terminal state of these options is the same as in covering options, but the initiation set, instead of being a single state, is now every state that is not terminal. Conversely, point-based eigenoptions are eigenoptions that have a single start state, defined to be the state with smallest value in the corresponding eigenvector.}
    \label{tab:variations}
\end{table}

Figure~\ref{fig:comparison_diffusion_time_ablation} depicts the performance of the algorithms described in Table~\ref{tab:variations}. These results show that a broad initiation set eventually leads to a smaller diffusion time but, until a minimum number of options is available, they hinder performance. Eigenoptions, for example, obtain the smallest diffusion time but requires more options before doing so---the same applies, in a smaller scale, to using covering options with a large (broad) initiation set. There is a bigger difference in the median diffusion time, as point-based options always reduce the median diffusion time while broad initiation sets can create attractor states that are difficult to escape from. Additionally, another trade-off to be considered is that, depending on how the options are used, a small initiation reduces the likelihood agents will have access to the corresponding options.

In the setting we analyzed here, additional eigenvectors provide benefits when compared to using only the top eigenvector of the SR, even when we discard the cost of collecting more samples at each iteration. Notice that, because we use the closed-form solution, it is as if the agent had covered the whole state-space before starting the option discovery process. We explore the setting in which the agent does not cover the whole state-space before the option discovery step in Section~\ref{sec:illustration_rod_cycle}, which is a setting in which the benefit of multiple iterations of the ROD cycle is clearer.

\begin{figure}[t]
     \centering
     \includegraphics[width=\textwidth]{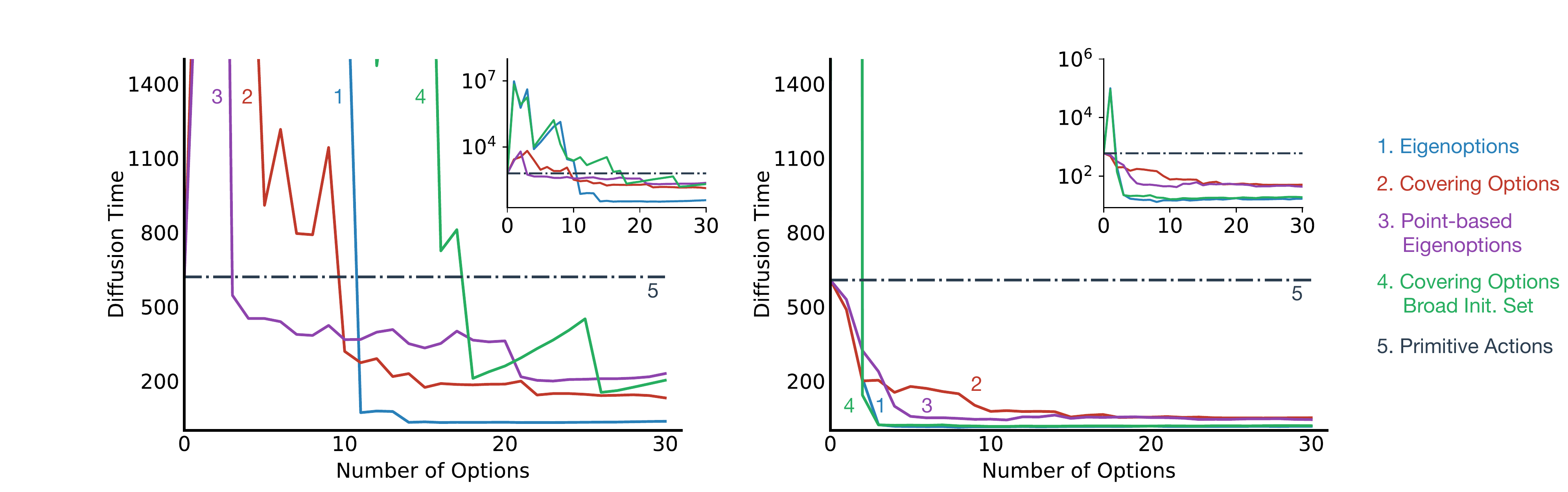}
     \caption{Average (left) and median (right) diffusion time in the four-room domain induced by the options generated with the algorithms in Table~\ref{tab:variations}. The inset plots depict, in log-scale, the range the diffusion time lies. See text for details.}
     \label{fig:comparison_diffusion_time_ablation}
\end{figure}

We can also analyze other successes of temporally extended-exploration in light of the insights we obtained here. \citet{Dabney20} recently introduced $\epsilon z$-greedy, an extended form of $\epsilon$-greedy exploration in which an agent, when taking an exploratory action, repeats the sampled action for a random duration, often sampled from a zeta distribution. Such an approach achieves remarkable success in standard benchmarks such as Atari 2600 games. Despite not being based on the successor representation, such an approach gives us evidence to support our conclusions. In $\epsilon z$-greedy exploration, the agent is allowed to sample an extended sequence of actions at any time, corroborating the intuition that restrictive initiation sets might prevent the agent from exploring the environment. Moreover, the duration in which the sampled action will be repeated is random, showing the benefits of varying time-scales, one of the features we get from using additional eigenvectors of the SR (c.f. Figure~\ref{fig:eigenoptions_length}). Similarly, the temporally-extended exploration used when learning to control superpressure balloons in the stratosphere~\citep{Bellemare20} is also based on options that are not constrained to a small region of the state space and they too have varying duration.

The results in this section also raise the question of whether multiple iterations of the ROD cycle improve an agent's ability to explore. As one can expect, this is indeed true, as we can see how, for covering options, additional iterations lead to better results. In Section~\ref{sec:illustration_rod_cycle}, we provide a different illustration where we use the insights gained here to design an algorithm that performs multiple iterations of the ROD cycle. Before we do so though, we assess the impact of discovering options online, which is an important aspect of any iterative cycle.

\subsection{Online Option Discovery}~\label{sec:online_experiments}

The results so far were obtained in closed-form, which can only be achieved after the agent has thoroughly explored the environment. This allowed us to evaluate algorithmic ideas in a conceptually simpler way, without approximation errors. However, this is not a realistic setting. In this section, we evaluate the impact of using options discovered from online estimates of the SR, instead of assuming access to an adjacency matrix representing the environment. We use TD learning to estimate the SR from samples, which allows us to compute options before the environment has been exhaustively explored.

Eigenoptions are robust to using online estimates of the SR, as one can see in Figure~\ref{fig:comparison_dt_eigenoptions_online}. This result is similar to \citeauthor{Machado18b}'s~(\citeyear{Machado18b}). To minimize the number of interactions with the environment, we re-use the data used to compute the SR to learn the eigenoptions' policies (c.f. Algorithm~\ref{alg:sr_closed_form}). We use Q-learning to learn these policies. The agent always starts at the bottom left corner. Episodes were 1,000 steps long and we used a step-size $\eta = 0.1$ and $\gamma=0.9$ to learn the SR.

\begin{figure}[t]
     \centering
     \includegraphics[width=\textwidth]{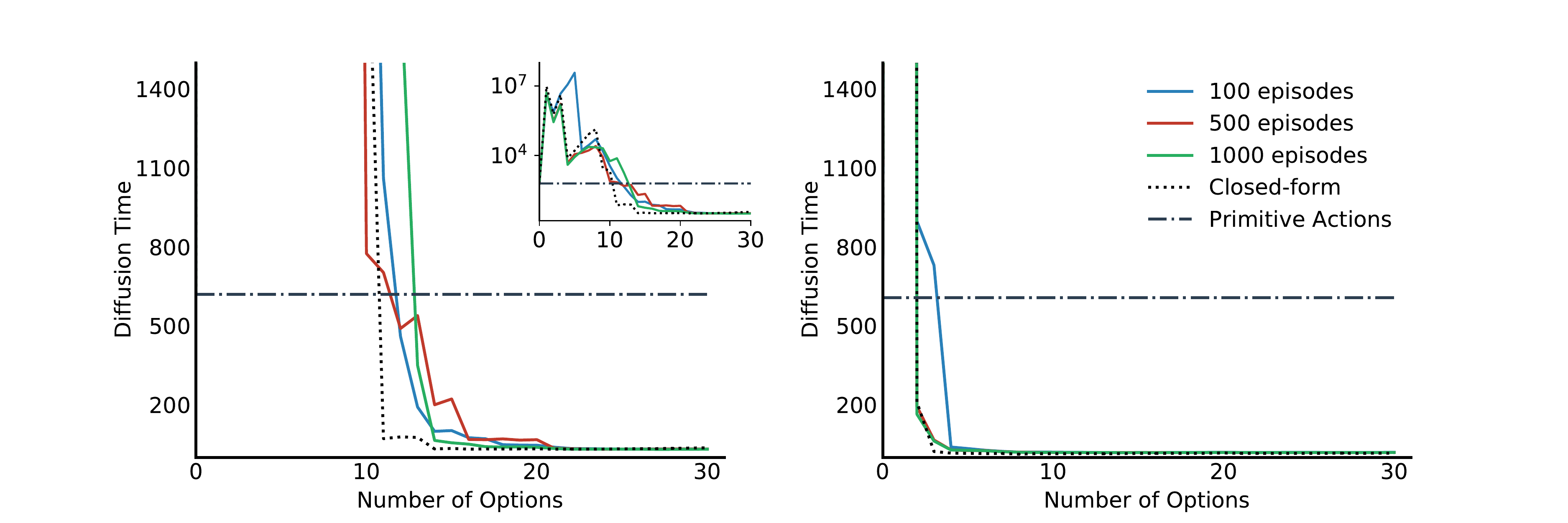}
     \caption{Average (left) and median (right) diffusion time, in the four-room domain, induced by eigenoptions computed with the SR, online and in closed-form. Episodes were 1,000 steps long. We report, averaged across 50 different runs, the impact of using different number of episodes for computing the SR online. The inset plot on the left figure depicts, in log-scale, the range the diffusion~time~lies.} %gamma = 0.9 for Q-learning, 0.99 for SR.
     \label{fig:comparison_dt_eigenoptions_online}
\end{figure}

Surprisingly, covering options are not as effective in the setting in which the representation is learned online. This is depicted in Figure~\ref{fig:comparison_dt_co_online}. While the results for the median diffusion time are consistent with the other results in this section, the results for the average diffusion time are not. Covering options are not able to reduce the average diffusion time in the environment; in fact, they increase it substantially. The reason becomes clearer when we look at the first options discovered. When estimating the SR online,\footnote{When the agent selects an option, we ignore individual transitions when estimating the SR, assuming the the agent ``teleports'' to the state in which the option terminates. This mimics the closed-form solution, which sees an option as creating an edge between two vertices. Moreover, we observed that estimating the SR from individual transitions leads to much worse results.} the agent rarely samples an option because it is available in a single state. This leads to similar options being discovered in multiple iterations. Moreover, covering options are less robust because they rely only on the maximum and minimum values of the top eigenvector of the SR to define the state in which an option is available. Using a single eigenvector adds to this brittleness because the agent cannot rely on other eigenvectors to capture different dimensions of the state space, or to correct for an option that captures the wrong timescale at which the agent should act. As an example, the top two eigenvectors of the SR capture the two diagonals of the four-room domain (c.f. Figure~\ref{fig:covering_option_example}). If an agent tries to use the top eigenvector of the SR to capture, in two different iterations, these two diagonals, it may fail to do so if the random walk of the agent in the second iteration does not sample the option that navigates the dimension the first iteration captured. The mismatch between the closed-form as online solutions only increases in later iterations---Appendix~\ref{sec:appendix_evolution} presents visualizations of eigenoptions and covering options learned online.

\begin{figure}[t]
     \centering
     \includegraphics[width=\textwidth]{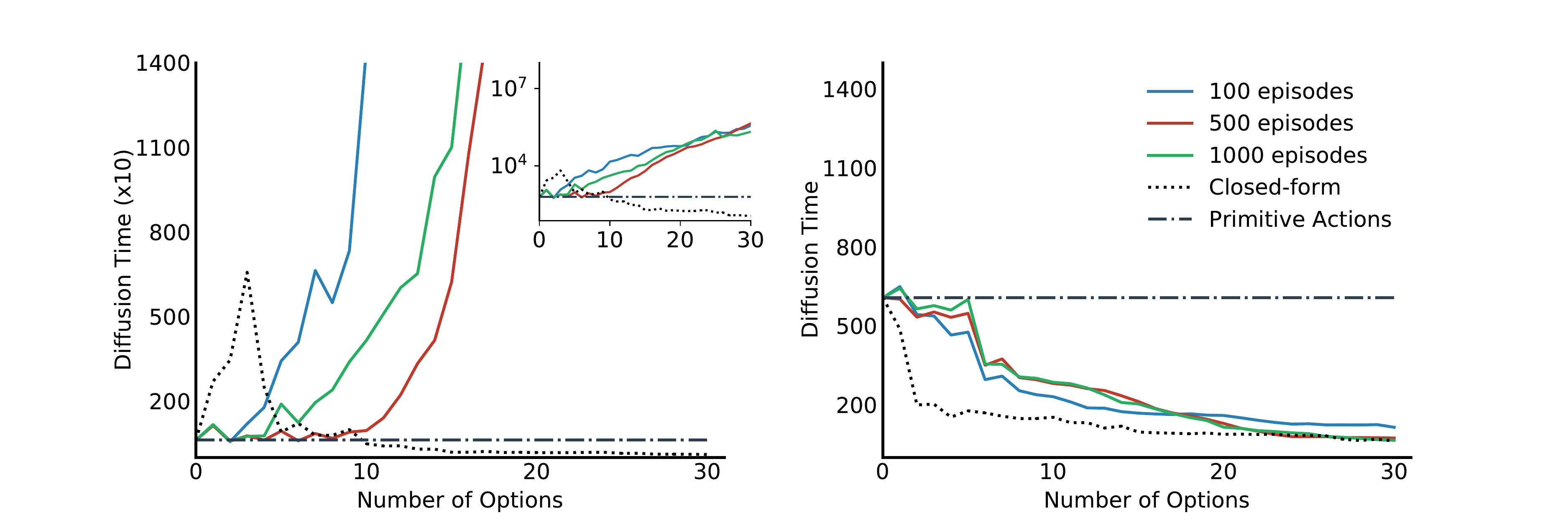}
     \caption{Average (left) and median (right) diffusion time, in the four-room domain, induced by covering options computed with the SR, online and in closed-form. Episodes were 1,000 steps long. We report, averaged across 50 different runs, the impact of using different number of episodes for computing the SR online. The inset plot on the left figure depicts, in log-scale, the range the diffusion time lies. The y-axis in the left plot is scaled by 10 in comparison to Figure~\ref{fig:comparison_dt_eigenoptions_online} to depict the average diffusion time induced by the covering options discovered online.}
     \label{fig:comparison_dt_co_online}
\end{figure}

Finally, using the SR to compute covering options can violate the symmetry assumption of Theorem~\ref{th:equivalence}. After the first iteration of covering options, in some states the agent has access to four primitive actions and one option, while in others only four actions are available. This mismatch is interesting from a theoretical point of view, but cannot be used to justify the poor performance we observed. At least in the four-room domain, this mismatch does not meaningfully impact the diffusion time induced by covering options discovered from the SR. The empirical analysis of the impact of this mismatch is available in Appendix~\ref{app:mismatch_co_sr_pvfs}, where we compare the diffusion time induced by covering options when using the eigenvectors of the SR and proto-value functions, and we show it is minimal.

\subsection{Summary}~\label{sec:summary_experiments_exploration}

In this section, we showed that eigenoptions and covering options reduce the diffusion time in a given environment, meaning that they capture properties of the environment that are useful when subsequent tasks are posed to the agent. Eigenoptions do so both when computed in closed-form and online, while covering options only do so when discovered in closed-form, the setting they were originally introduced. We also investigated the impact of the different choices these approaches make to understand their overall impact in temporally-extended exploration with options. Specifically, whether it is beneficial to discover more than one option per iteration, with our results showing that discovering multiple options leads to a lower diffusion time, more robust solutions, and a more judicious sample use. Moreover, we evaluated different approaches for defining options' initiation sets and termination conditions. Our results highlight an interesting trade-off: while making options available in fewer states avoids the creation of sink states that increase the diffusion time, these options end up not being so useful to the agent when learning to maximize rewards. They can also be quite detrimental in the online setting.

The results in this section also raise several interesting questions. In the next section we use the insights gained here to introduce an algorithm that illustrates multiple iterations of the ROD cycle, and in Sections~\ref{sec:options_keyboard} and~\ref{sec:experiments_ok} we explore the possibility of linearly combining eigenoptions, without additional sample complexity, in order to obtain more diverse behavior. Other questions we leave for future work revolve around, for example, combining the benefits of eigenoptions and covering options, maybe through an ensemble of both types of options; or on how to chain such options.

\section{Iterative Option Discovery with the ROD Cycle}~\label{sec:illustration_rod_cycle}

So far, we have investigated settings in which either the agent has access to the true SR or it is able to learn an accurate estimate of the SR at each iteration~of~the~ROD~cycle. It is hard to see the benefits of multiple iterations of the ROD cycle in these settings because a single iteration is already enough for the agent to learn the SR accurately. We now look at the more challenging (and realistic) setting in which it is impossible for the agent to visit every state of the environment in a single iteration. This allows us to validate the claim that options can be used to generate a different distribution of state visitation, which leads to different representations, which further impacts subsequently discovered options, in a virtuous, never-ending, cycle.

\subsection{Iterative Online Eigenoption Discovery}

Inspired by the results in the previous section, such as the importance of a large initiation set and the dangers of sink nodes, we introduce \emph{covering eigenoptions (CEO)}. CEO is an algorithm for option discovery that illustrates the benefits of multiple iterations of the ROD cycle. It combines the best design decisions of both eigenoptions and covering options~for~the~online~case.

CEO discovers one eigenoption per iteration, slowly increasing the size of the option set. It uses the top eigenvector of the SR to define such an eigenoption but, differently than covering options, the SR it uses is defined only over primitive actions.  At each time step, actions or options (if available) are randomly selected. Options are sampled with a smaller probability than primitive actions to ensure they do not dominate the agent's behavior. 

At first, discovering one option per iteration might seem surprising in light of the results in the previous section suggesting that discovering more options per iteration is beneficial. However, such a choice stems from the difference in problem setting: now that options are sampled, it is important to acknowledge that any algorithm based on the ROD cycle will end up with a growing set of options. If too many options are added at each iteration, the probability of sampling any individual option becomes vanishingly small. This is in contrast with the previous evaluation that was mostly focused on how the discovered options change the topology of the environment. Finally, by discovering eigenoptions at each iteration, due to their large initiation set, we increase the likelihood that the options will actually be sampled, minimizing the chance an iteration is wasted, as sometimes happens with covering options (c.f. Section~\ref{sec:online_experiments}). Algorithm~\ref{alg:rod_eigenoptions}, in the next page, depicts the pseudo-code for CEO. We explicitly discuss our choices for each each step of the ROD cycle below. \vspace{0.3cm}

\emph{Collect samples.} The agent collects samples by randomly interacting with the environment. In the first iteration, only primitive actions are available to the agent; later, options also become available. If an option is sampled, the agent acts according to that option's policy until termination. We save the observed transitions in a data set. We consider transitions in terms of primitive actions, even when they are induced by an option. Importantly, the probability of sampling an option, $p_{option}$, should be lower than the probability that a primitive action is sampled to account for the fact that options have a longer duration. \vspace{0.3cm}

\emph{Learn representation.} This step consists in learning the successor representation from the samples gathered by the agent. One important aspect to the success of the introduced algorithm is that it never throws away any data, meaning that it is constantly refining the SR instead of learning it from scratch at every iteration. \vspace{0.3cm}

\emph{Derive an intrinsic reward function from the learned representation.} Having access to the SR learned online, the agent now derives an intrinsic reward function it will use to learn the option's policy. We use the reward function defined by eigenoptions. Importantly, we only consider one direction for the eigenvector: to incentivize the agent to go to states that it has not visited much, it is important to pre-determine the direction of the eigenvector, otherwise one of the generated eigenoptions is the one that takes the agent to the place it has visited the most. Obviously, it is important to avoid such a behavior when options are sampled because we want to incentivize the agent to visit places it has not visited much. For this step, we choose the direction of the eigenvector such that $\sum_i \mathbf{e}(i) < 0$. Inspired by covering options, we use only the top eigenvector of the SR. \vspace{0.3cm}

\emph{Learn to maximize intrinsic reward.} We use Q-Learning to learn how to maximize the intrinsic reward defined above. Because the intrinsic reward is only defined in states which the agent has visited before, letting the agent learn from scratch how to maximize such reward might lead the agent to visit states in which the reward function is not defined. Thus, instead, we use the transitions in the data set we collected in the sample collection step. Such an approach saves agent-environment interactions and guarantees the agent will learn a policy taking only previously visited states into consideration.
\vspace{0.3cm}

\emph{Define option.} Finally, we define the option's initiation set and termination condition following the eigenoptions description in Section~\ref{sec:temporally_extended_exploration}, adding the new option to the option set.

\begin{algorithm}[p]
\caption{Covering Eigenoptions}\label{alg:rod_eigenoptions}
\KwInput{$\eta, \ \alpha_{o}$ \Comment*[r]{Step-sizes for learning the SR and the options' policies}}
\hspace{1.2cm} $\gamma_{SR}, \ \gamma_{o}$ \Comment*[r]{Discount factor for the SR and the options' policies}
\hspace{1.2cm} $p_{option}$ \Comment*[r]{Prob. of sampling an option instead of a primitive action}
\hspace{1.2cm} $N_{steps}$ \Comment*[r]{Maximum number of interactions with the environment}
\hspace{1.2cm} $N_{iter}$ \Comment*[r]{Number of iterations of the ROD cycle}

\vspace{0.5cm}

%interact with the environment
$\mathcal{D} \gets \emptyset$\\
$\Omega \gets \emptyset$\\
\For{$i \gets 0$ \KwTo $N_{iter}$}{
$\triangleright$ Collect samples\\
\For{$j \gets 0$ \KwTo $N_{steps}$}{
    With prob. $1 - p_{option}$ randomly sample, uniformly, a primitive action $a$; otherwise uniformly sample an option $\omega$ from $\Omega$\\
    \uIf{$\textrm{\normalfont primitive action was sampled}$}{
    In state $s$, take action $a$ and observe state $s'$ and reward $r$\\
    $\mathcal{D} \gets \mathcal{D} \ \| \ (s, a, r, s')$ \Comment*[r]{Append transition to data set $\mathcal{D}$}
    }
    \Else{
      \While{$\beta_\omega(s) \neq 1$ }{
      In state $s$, take action $\pi_\omega(s)$ and observe state $s'$ and reward $r$\\
    $\mathcal{D} \gets \mathcal{D} \ \| \ (s, \pi_\omega(s), r, s')$ \Comment*[r]{Append transition to data set $\mathcal{D}$}
    }
    }
}
$\triangleright$ Learn representation\\
$\Psi \gets \textrm{Successor Representation}(\eta, \gamma_{SR}, \mathcal{D})$\\

$\triangleright$ Derive an intrinsic reward function from the learned representation\\
$\mathbf{e} \gets \textrm{getTopEigenvector}(\Psi)$\\
$r^{\mathbf{e}}(s, s') \gets \textrm{\normalfont \textbf{e}}(s') - \textrm{\normalfont \textbf{e}}(s) \ \ \forall s, s' \in \mathscr{S}$ \Comment*[r]{Define eigenpurpose}
$\triangleright$ Learn to maximize intrinsic reward\\
 $Q \gets \textrm{Q-Learning}(r^{\mathbf{e}}, \alpha_o, \gamma_o, \mathcal{D})$ \Comment*[r]{Learn value function}
 $\triangleright$ Define option\\
 $\mathcal{I_\omega} \gets \emptyset$; $\pi_\omega(s) \gets \bot$, $\beta_\omega(s) \gets 1 \ \
 \forall s \in \mathscr{S}$ \Comment*[r]{Initialize option tuple}
 \For{$s$ \textrm{\normalfont \textbf{in}} $\mathscr{S}$}{
    $\triangleright$ \ If $s$ is not a terminal state for the eigenoption being learned:\\
    \uIf{$\exists a \in \mathscr{A}(s) \ s.t. \ Q(s, a) > 0 $}{
    $\mathcal{I}_\omega \gets \mathcal{I}_\omega \cup \{s\}$\\
    $\pi_\omega(s) \gets \argmax_a Q(s, a)$\\
    $\beta_\omega(s) \gets 0$
  }
 }
 $\Omega \gets \Omega \cup \langle \mathcal{I}_\omega, \pi_\omega, \beta_\omega \rangle$
}
\end{algorithm}

\subsection{Empirical Analysis}

As in the previous sections, we evaluate CEO in the four-room domain. However, instead of using episodes that are 1,000 steps long, we now consider episodes that are only 100 steps long. The agent starts every episode in the top right corner. In this setting, it is impossible for the agent to visit every state in the environment in a single episode. The agent interacts with the environment for a whole episode as part of the data collection step, with the agent using the collected data set to complete the other steps of the ROD cycle between episodes. In other words, each episode corresponds to a different iteration of the ROD cycle.

We performed numerical simulations to estimate how many time steps, on average, CEO needs to visit every state in the environment at least once.\footnote{If the agent visits a state for the first time at the second time step of the second iteration, we consider the agent has visited that state at time step 101 + 3 = 104.} This metric can be seen as a Monte Carlo estimate of the diffusion time when counting steps instead of decisions. We use this metric because it is not clear how to easily compute the diffusion time in closed form when taking into consideration the episodic nature of the problem. We use $\eta = \alpha_o = 0.1$, $\gamma_{SR} = \gamma_o = 0.99$, and we sample options with 5\% probability ($p_{option}$), which is similar to what we did in Section~\ref{sec:online_experiments}, where options were potentially sampled only in the exploration step of Q-Learning with $\epsilon$-greedy ($\epsilon = 0.05)$. We pass over $\mathcal{D}$ $100$ times when learning the SR, and $1,000$ when learning the option policy, leveraging the off-policy aspect of our problem formulation.

CEO needs, on average, 2,301.2 steps to visit every state at least once (n=100, SD=830.2, Mdn=2,069.5, Min=1,067, Max=5,793). In contrast, a uniform random policy needs 27,032.3 steps (n=100, SD=16,961.0, Mdn=22,397.0, Min=3,328, Max=95,118). CEO, by implementing multiple iterations of the ROD cycle, reduces the number of interactions an agent needs in order to visit every state by an order of magnitude! Importantly, when compared to a uniform random policy, CEO does not only visit every state quicker, but it also induces a more uniform visitation over the state space, as depicted in Figure~\ref{fig:comparison_exploration_iterative_exp}.

\begin{figure}[t]
     \centering
     \includegraphics[width=0.66\textwidth]{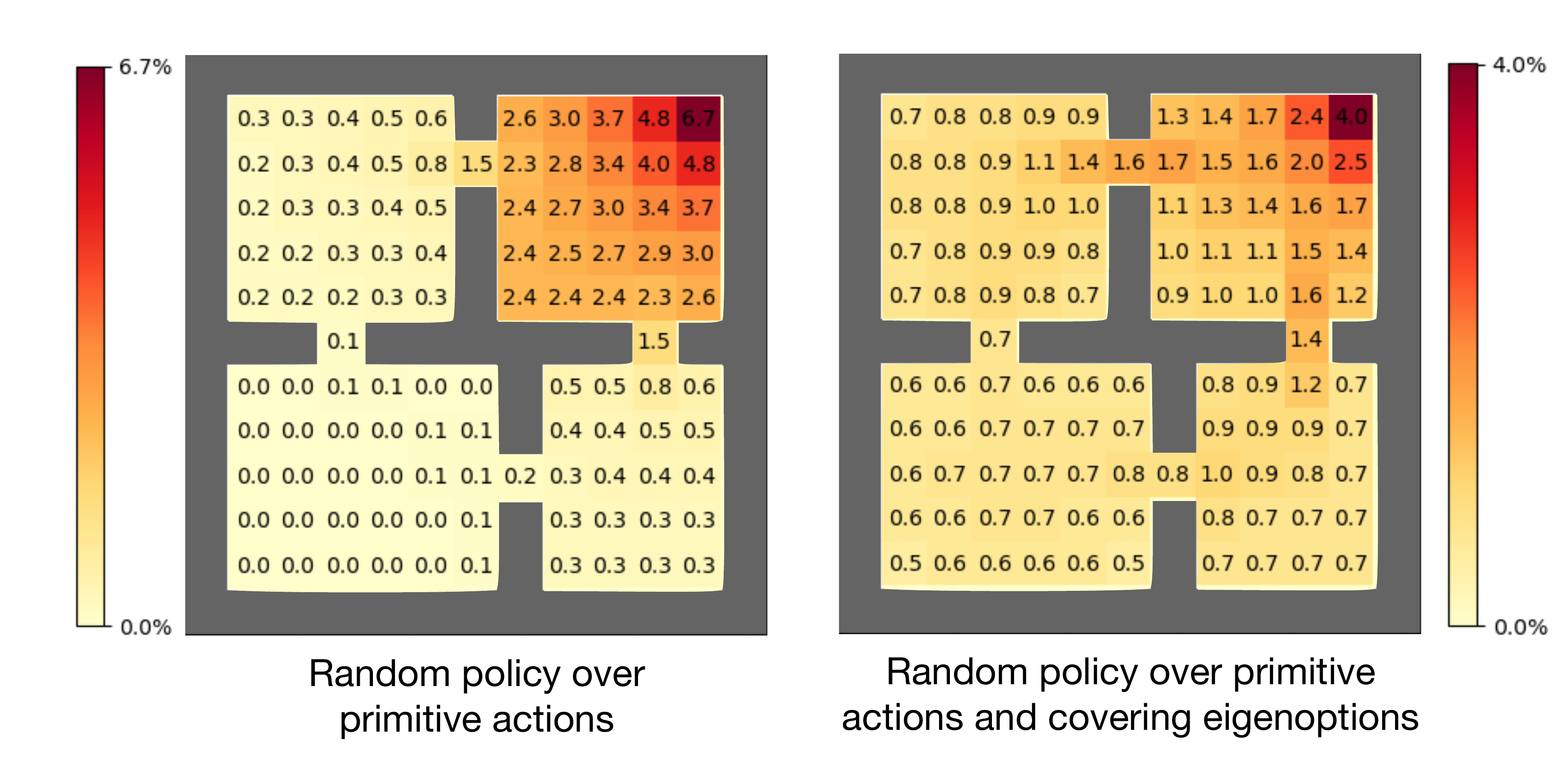}
     \caption{Proportion of time spent in each state when randomly selecting actions over primitive actions only, and over primitive actions and covering eigenoptions. Reported numbers are averaged over 100 seeds. The result of each seed is calculated at the end of the last episode, which is defined by the time in which the agent has visited every state in the environment. The proportion is computed as the number of time steps the agent was in that particular state divided by the total number of interactions with the environment.}
     \label{fig:comparison_exploration_iterative_exp}
\end{figure}

\begin{figure}
     \centering
     \includegraphics[width=0.95\textwidth]{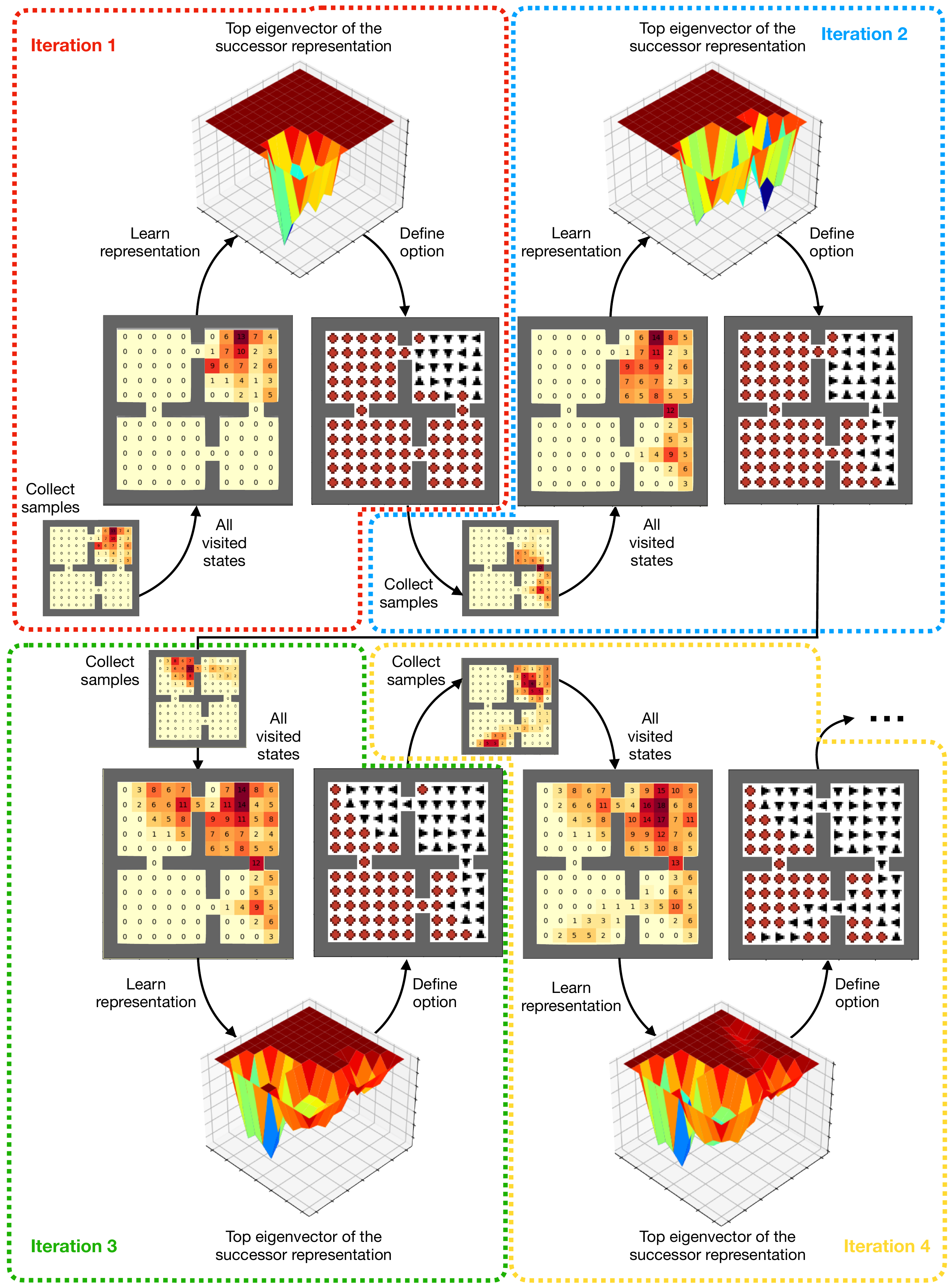}
     \caption{Illustration of multiple iterations of the ROD cycle when instantiated with covering eigenoptions. See text in Section~\ref{sec:illustration_rod_cycle} and Algorithm~\ref{alg:rod_eigenoptions} for details.}
     \label{fig:illustration_rod_cycle}
\end{figure}

The behavior induced by CEO indeed supports the claims around the benefits of multiple iterations of the ROD cycle. Figure~\ref{fig:illustration_rod_cycle} depicts the first four iterations of CEO with a particular seed. In iteration 1, the agent collects samples with a random walk and it learns a representation such that it is incentivized to visit states it has not visited often. This leads to an option that takes the agent closer to the hallway between the top and bottom rooms on the right. In iteration 2, the agent ends up closer to the referred hallway as it eventually samples the learned option; a random walk then leads the agent to the bottom right room. The learned representation still puts a lot of mass in the states often visited in the first iteration, but the reward function is now defined in more states. Aside from the option that takes the agent to the aforementioned hallway, the agent now also discovers an option that, when in the top right room, takes the agent to the hallway between the top right and left rooms or, when in the bottom right room, takes the agent close to the hallway between the bottom right and left room. At the third iteration the agent then has access to two options. By chance, the agent eventually samples an option that takes it to the state close to the hallway between the top and left room, with the random walk then, by chance, taking the agent to the top left room. The induced SR at the end of these three iterations leads to the discovery of a third option that, from most states, take the agent to the state close to the hallway between the bottom right and left room. By chance, the agent ends up sampling that option in the fourth iteration, visiting the room that had not been visited yet. This is a clear example of the benefits of the ROD cycle where options discovered from previous iterations act as a scaffold for more complex behaviors discovered in subsequent iterations. We chose to depict this particular behavior of CEO for pedagogical purposes, although the agent does not always visit the four rooms in the first four iterations, it eventually does so, generally much faster than a uniform random policy.

These results (and the proposed method) add to the growing list of approaches that can be seen as instantiations of the ROD cycle~\citep[e.g.][]{Erraqabi21,Jinnai20,Hoang21,Machado16}. The illustration in this section is particularly useful for making the multiple iterations of the cycle, each step, and their outcome, very explicit. We further discuss some of these methods in Section~\ref{sec:related_work}.

\section{Combining Options with the Option Keyboard}~\label{sec:options_keyboard}

In the previous sections we discussed the ROD cycle, a general framework for option discovery based on representation learning (c.f. Figure~\ref{fig:option_discovery_cycle}). We described two algorithms that instantiate the ROD cycle, one based on eigenoptions and one based on covering options, and we introduced an algorithm that combines both eigenoptions and covering options to benefit from multiple iterations of the ROD cycle. In all these cases, options are derived from the eigenvectors of the SR matrix.  

Options derived from the eigenvectors of the SR are particularly suitable for exploration because they induce complementary distributions that collectively cover the state space in a structured way. At first this suggests a straightforward strategy for exploration: compute one option per eigenvector and then use them together to explore the environment. The problem with this strategy is that, since each option must be \emph{learned}, there is an inherent cost associated with adding them to the library of available behaviors. Moreover, adding the options associated with \emph{all} the eigenvectors may be unnecessary, since some of them provide little benefit in terms of exploration in the presence of others. To illustrate this point, note that, even in a simple domain like  four-room, this exploration scheme would result in more than $100$ options. Our experiments clearly demonstrate that using this many options is not really necessary  (c.f. Figures~\ref{fig:comparison_diffusion_time_ablation}, \ref{fig:comparison_dt_eigenoptions_online}, and~\ref{fig:comparison_dt_co_online}), including those in Section~\ref{sec:illustration_rod_cycle}.

This leads us to the second useful property of options induced by the eigenvectors of the SR: their associated eigenvalues provide a natural ordering of the options according to their time scale---this roughly corresponds to the option's expected time before termination, as shown in Figure~\ref{fig:eigenoptions_length}. Based on this observation, we can improve the exploration strategy outlined above: instead of adding the options all at once, we rank them based on their eigenvalues and add them one by one until they collectively form a good basis for exploration. This is the basic recipe underlying both eigenoptions and covering options. But can we do better? Can we use these options to not only go to the far end of the environment, but in pretty much any state in the environment?

It has been argued that the ability to \emph{combine} options may be key to extend the range of available behaviors without incurring the otherwise inevitable cost in terms of sample transitions~\citep{sutton2016towards,hees2016learning,haarnoja2018latent}. By exploiting compositionality, one can potentially grow a finite number of options into a combinatorially larger counterpart without additional learning. 
In the context of eigenoptions and covering options, this benefit manifests itself in two complementary ways. First, by combining higher-order options, it may be possible to approximately emulate their lower-order counterparts, which will thus no longer need to be learned. Second, and perhaps more important, depending on how options are combined, they may give rise to new options whose behavior differ significantly from that of \emph{any} option induced by the eigenvectors of the SR---i.e., the combined options might effectively extend the SR behavioral basis used for exploration, even when thinking about multiple iterations of the ROD cycle.

The question then arises as to how to actually implement the combination of options. In this context, a natural choice is the \emph{option keyboard}~\citep{Barreto19}. The option keyboard is particularly suitable to be used with options induced by the SR because it is itself based on the concept of SR, making the integration between the two approaches natural and transparent. Given a set of options evaluated under different rewards, the option keyboard provides a way to instantaneously generate options induced by any linear combination of the rewards, without any learning involved. Although simple, this way of combining options provides all the benefits mentioned above in the context of options induced by the eigenvectors of the SR. In the next section we elaborate on how to build and use the option keyboard using options derived from the SR.

\subsection{Generalized Policy Evaluation and Generalized Policy Improvement}

The option keyboard is based on generalizations of the concepts of policy evaluation and policy improvement introduced in Section~\ref{sec:background}. Simply put, \emph{generalized policy evaluation}~(GPE) is the computation of a policy's value function under different reward functions. Given the value function of a set of policies under a specific reward function, \emph{generalized policy improvement} (GPI) is the computation of a new policy whose performance is no worse, and generally better, than that of the original policies. GPE and GPI are strict generalizations of their standard counterparts, which are recovered as special cases~\citep{Barreto20}.

To accommodate the generalization provided by GPE, we will use $v_\pi^r$ and $q_\pi^r$ to denote the state-value and action-value functions of policy $\pi$ under reward $r$. We start by noting that the SR provides a particularly efficient form of GPE: as shown in Eq.~\ref{eq:sr_reward}, once we have the SR of a policy $\pi$, $\Psi_\pi$, we can evaluate it under \emph{any} reward function $ {\bf r}$ by simply making $\mathbf{v}_{\pi}^{ \bf r} =  \Psi_\pi {\bf r}$. As discussed in Section~\ref{sec:successor_representation}, SFs generalize the SR by replacing its features $\bm{\phi}_i(s_j) =\mathbbm{1}_{\{i=j\}}$ with arbitrary features $\bm{\phi}: \mathscr{S} \times \mathscr{A} \mapsto \mathbb{R}^d$. Once we have the SFs of a policy $\pi$, $\bm{\psi}_{\pi}$, we can compute its value function under any linear combination of the features, $r(s,a) = \bm{\phi}(s,a)^{\top}\bm{w}$, by making $q_{\pi}^{r}(s,a) = \bm{\psi}_{\pi}(s,a)^\top \bm{w}$, where $\bm{w} \in \mathbb{R}^d$ (c.f. Eq.~\ref{eq:reward_features} and~\ref{eq:sfs_q}). Note that, because the features $\bm{\phi}(s,a)$ can be any function, one can in principle use an intrinsic reward $r_i(s,a)$ as the $i$-th feature: $\phi_i(s,a) = r_i(s,a)$. This will be important for the option keyboard.

GPI is the computation of a policy whose performance under a given reward is generally better than that of its precursors. The mechanics of GPI are actually very similar to that of its standard counterpart shown in Eq.~\ref{eq:policy_improvement}: given policies $\pi_1$, $\pi_2$, ..., $\pi_n$, and their value functions under reward $r$, $q_{\pi_1}^{r}$, $q_{\pi_2}^{r}$, ..., $q_{\pi_n}^{r}$, the GPI policy $\pi'$ is given by
\begin{equation}
 \label{eq:gpi}
\pi'(s) \in \mathrm{argmax}_{a} \max_i q_{\pi_i}^{r}(s,a) \text{ for all } s \in \mathscr{S}.
 \end{equation}
\citet{Barreto17} have shown that, starting from the execution of any action $a \in \mathscr{A}$ on any state $s \in \mathscr{S}$, the GPI policy $\pi'$ will do at least as well as, and generally better than any of its precursors $\pi_i$. More formally, we have that $q_{\pi'}^{r}(s,a) \ge  q_{\pi_i}^{r}(s,a)$ for all $i \in \{1, 2, ..., n\}$ and all $(s,a) \in \mathscr{S} \times \mathscr{A}$. This result can also be extended to the case in which GPI is applied with approximations $\tilde{q}_{\pi_i}^{r} \approx q_{\pi_i}^{r}$~\citep{Barreto17}.  

\subsection{Synthesizing Options with GPE and GPI}

Now that we have introduced the concepts of GPE and GPI, we describe how to use these operations to create options without any learning involved. For clarity, instead of presenting the option keyboard in its most general form, we will use the formalism of eigenoptions introduced in Section~\ref{sec:temporally_extended_exploration} (the adaptation to covering options is straightforward).

Let $\mathbf{e}_1$, $\mathbf{e}_2$, ..., $\mathbf{e}_d$ be eigenvectors of the SR, $\Psi_\pi$, induced by a policy $\pi$. As per Eq.~\ref{eq:eigenpurpose}, each $\mathbf{e}_i$ gives rise to a reward function $r^{\mathbf{e}_i}(s,s') =  \mathbf{e}_i^\top \big(\bm{\phi}(s') - \bm{\phi}(s)\big)$, which in turn gives rise to an eigenoption $\omega^{\mathbf{e}_i}$. As shown in Eq.~\ref{eq:extended_eigenoption_policy}, the eigenoption $\omega^{\mathbf{e}_i}$ can be compactly represented as a policy over an extended action space $\pi^{{\mathbf{e}_i}{+}}$. For ease of exposition, we will use $\pi^{{\mathbf{e}_i}{+}}$ to refer to $\omega^{\mathbf{e}_i}$ in this section. Suppose we have the value function of all the eigenoptions $\pi^{{\mathbf{e}_i}{+}}$ under all the rewards $r^{\mathbf{e}_j}$, that is, we have $q_{\pi^{{\mathbf{e}_i}{+}}}^{r^{\mathbf{e}_j}}$ for $i,j \in \{1, 2, ..., d\}$.
We will now show how to instantaneously generate options associated with linear combinations of the eigenvectors $\mathbf{e}_i$ without any learning involved. 

Let $\bm{w} \in \mathbb{R}^d$ and let
\begin{equation}
\label{eq:eigen_linear_combination}
 \mathbf{c} = \sum_i w_i \mathbf{e}_i.
\end{equation}
We want to compute an approximation of the option induced by the reward 
$r^{\bf{c}}(s, s') = {\bf c}^\top \big(\bm{\phi}(s') - \bm{\phi}(s)\big)$
without resorting to learning. First, we note that 
\begin{equation*}
r^{\bf{c}}(s, s')
= \sum_i w_i \mathbf{e}_i^\top \big(\bm{\phi}(s') - \bm{\phi}(s)\big) 
= \sum_i w_i r^{\mathbf{e}_i}(s, s').
 \end{equation*}
Connecting the above with Eq.~\ref{eq:reward_features}, we see that here we are using the intrinsic rewards $r^{\mathbf{e}_i}$ as features. This view allows us to resort to Eq.~\ref{eq:sfs_q} to compute the value function of the eigenoptions $\pi^{{\mathbf{e}_i}{+}}$ under the reward $r^{\bf{c}}$ as
\begin{equation}
 \label{eq:eigen_gpe}
 q_{\pi^{{\mathbf{e}_i}{+}}}^{r^{\bf{c}}}(s,a) = \sum_j w_j q_{\pi^{{\mathbf{e}_i}{+}}}^{r^{\mathbf{e}_j}}(s,a).
\end{equation}
Once we have $q_{\pi^{{\mathbf{e}_i}{+}}}^{r^{\bf{c}}}$ for all $i \in \{1, 2, ..., d\}$, we can use GPI defined in Eq.~\ref{eq:gpi} to compute 
\begin{equation}
 \label{eq:eigen_gpi}
  \tilde{\pi}^{{\mathbf{c}}+}(s) \in \mathrm{argmax}_{a \in \mathscr{A} \cup \{\bot\}} \max_i q_{\pi^{{\mathbf{e}_i}{+}}}^{r^{\bf{c}}}(s,a) \text{ for all } s \in \mathscr{S}.
\end{equation}
Note that, since $\tilde{\pi}^{{\mathbf{c}}+}$ is defined over an extended action space which also includes the terminate action $\bot$, it gives rise to a well-defined option, including the initiation and termination sets (see discussion in Section~\ref{sec:temporally_extended_exploration}). 

The option computed in Eq.~\ref{eq:eigen_gpi} is an approximation of the option $\pi^{{\mathbf{c}}+}$ induced by $\mathbf{c}$, that is, $\tilde{\pi}^{{\mathbf{c}}+} \approx \pi^{{\mathbf{c}}+}$. The advantage of using $\tilde{\pi}^{{\mathbf{c}}+}$ is that, unlike $\pi^{{\mathbf{c}}+}$, it can be obtained without any learning involved.
Based on the results regarding GPI, we know that  $\tilde{\pi}^{{\mathbf{c}}+}$ will perform at least as well as, and generally better than, any of the options $\pi^{{\mathbf{e}_i}{+}}$ under the reward $r^{\mathbf{c}}$. This means that the larger the number of options $\pi^{{\mathbf{e}_i}{+}}$ used in Eq.~\ref{eq:eigen_gpi} the closer $\tilde{\pi}^{{\mathbf{c}}+}$ will be to ${\pi}^{{\mathbf{c}}+}$. In any case, it is worth noting that, since we are using options mostly to generate diverse behavior, an eventual sub-optimal performance of   $\tilde{\pi}^{{\mathbf{c}}+}$ should not have a  catastrophic effect.

\emph{Putting it all together.} We now summarize how the procedure above can be combined with eigenoptions to considerably enlarge the number of options used for exploration. Given a set of eigenvectors $\mathbf{e}_1$, $\mathbf{e}_2$, ..., $\mathbf{e}_d$, the first thing we do is to compute the induced eigenoptions $\omega^{\mathbf{e}_1}, \omega^{\mathbf{e}_2}, ..., \omega^{\mathbf{e}_d}$, which here we represent as policies defined over an augmented action space: $\pi^{{\mathbf{e}_1}{+}}, \pi^{{\mathbf{e}_2}{+}}, ..., \pi^{{\mathbf{e}_d}{+}}$. Then, we evaluate each $\pi^{{\mathbf{e}_i}{+}}$ under the reward functions induced by the eigenvectors, that is, we compute $q_{\pi^{{\mathbf{e}_i}{+}}}^{r^{\mathbf{e}_j}}$ for $i,j \in \{1, 2, ..., d\}$ (obviously, we can compute $q_{\pi^{{\mathbf{e}_i}{+}}}^{r^{\mathbf{e}_j}}$ \emph{while} we learn the options $\pi^{{\mathbf{e}_i}{+}}$---see, for example, \citeauthor{Barreto19}'s (\citeyear{Barreto19}) Algorithm~3). Once we have the value functions $q_{\pi^{{\mathbf{e}_i}{+}}}^{r^{\mathbf{e}_j}}$, the successive application of Eq.~\ref{eq:eigen_linear_combination}, \ref{eq:eigen_gpe} and~\ref{eq:eigen_gpi} with \emph{any} $\bm{w} \in \mathbb{R}^d$ results in an option $\tilde{\pi}^{{\mathbf{c}}+}$ that approximates the option $\pi^{{\mathbf{c}}+}$ induced by $\mathbf{c} = \sum_i w_i \mathbf{e}_i$. See Algorithm~\ref{alg:ok} for a presentation of this discussion in pseudo-code.

\begin{algorithm}[t]
\caption{Option Keyboard}\label{alg:ok}
\KwInput{$Q_{\pi_{\omega^{\mathbf{e}}}}^{r^{\mathbf{e}}}$ \Comment*[r]{Matrix of Q-values for each reward-policy combination}}
\hspace{1.2cm} $\omega^{\mathbf{e}}$\Comment*[r]{Set of options generated from intrinsic rewards $\mathbf{e}_i$}
\hspace{1.2cm} \textbf{w} \Comment*[r]{Weights to be used to combine behaviors}
\KwOutput{$\tilde{\omega}$ \Comment*[r]{Options generated by the OK from the base options}}

\vspace{0.2cm}

$\triangleright$ GPE\\

%\For{$(w_i, Q_\omega) \textrm{\normalfont \textbf{ in }} (\mathbf{w}, Q_{\pi_{\omega^{\mathbf{e}}}}^{r^{\mathbf{e}}})$}{

\For{${\omega^{\mathbf{e}_i}} \textrm{\normalfont \textbf{ in }} {\omega^{\mathbf{e}}} $}{
    $Q_{\pi_{\omega^{\mathbf{e}_i}}}^{r^{\mathbf{c}}} \gets \mathbf{w}^\top Q_{\pi_{\omega^{\mathbf{e}_i}}}^{r^{\mathbf{e}}}$\\
}

\vspace{0.2cm}

$\mathcal{I}_{\tilde{\omega}} \gets \mathscr{S}$ \\ 
\For{$s \textrm{\normalfont \textbf{ in }} \mathscr{S}$}{
    $\triangleright$ GPI\\
    $\pi^+_{\tilde{\omega}}(s) \gets \argmax_{a} \max_{i} Q_{\pi_{\omega^{\mathbf{e}_i}}}^{r^{\mathbf{c}}}(s,a)$\\
    $\triangleright$ Converts $\pi^+_{\tilde{\omega}}(s)$ into a regular option \\    
  \If{$\pi^+_{\tilde{\omega}}(s) = \bot$}{
        $\pi_{\tilde{\omega}}(s) \gets$ random action in $\mathscr{A}$ \\
        $\mathcal{I}_{\tilde{\omega}} \gets \mathcal{I_{\tilde{\omega}}} \setminus \{ s \}$\\
        $\beta_{\tilde{\omega}}(s) \gets 1$\\
        }
    \Else{
        $\pi_{\tilde{\omega}}(s) \gets \pi^+_{\tilde{\omega}}(s)$\\
        $\beta_{\tilde{\omega}}(s) \gets 0$\\
        }
}
$\tilde{\omega} \gets \langle \mathcal{I}_{\tilde{\omega}}, \pi_{\tilde{\omega}}, \beta_{\tilde{\omega}}\rangle$

\end{algorithm}

We can think of the process above as implementing a mapping from $\bm{w} \in \mathbb{R}^d$ to an approximation of the corresponding option $\pi^{{\mathbf{c}}+}$, where $\bm{c} = \sum_i w_i \mathbf{e}_i$. This means that we immediately have at our disposal a potentially very large set of options induced by all possible instantiations of the vector $\bm{w}$.  If $\bm{w}$ happens to only have one non-zero element that is positive, we recover one of the eigenoptions $\pi^{{\mathbf{e}_i}{+}}$ induced by the eigevectors $\mathbf{e}_i$. For other instantiations of $\bm{w}$ we have options whose behavior can significantly deviate from that of the original eigenoptions~\citep{Barreto19}. In the next section we use experiments to illustrate the benefits of combining the option keyboard with eigenoptions.

\section{Combining Eigenoptions with the Option Keyboard}~\label{sec:experiments_ok}

In this section, we demonstrate the synergy between eigenoptions and the option keyboard. As mentioned above, the option keyboard allows one to extend a finite set of options to a combinatorially large counterpart without additional learning. Nevertheless, the option keyboard assumes an initial set of basis options is available beforehand, with existing results in the literature relying on handcrafted options as basis options. When considering discovered options, eigenoptions are natural candidates as basis options. They are autonomously discovered from the SR, the same object that makes the option keyboard computationally efficient, and they are generated by orthogonal vectors obtained from the agent's behavior. Intuitively, using the option keyboard to combine eigenoptions is the equivalent of computing linear combinations of orthogonal bases of behaviors.
 
We first present a qualitative analysis of the options generated by combining eigenoptions with the option keyboard. We present multiple eigenoptions and the options the option keyboard generates from them. We discuss the number of unique combinations and how diverse the generated options are. We then present a quantitative analysis focused on the diffusion time induced by these new options generated by the option keyboard. We conclude this section with a higher-level discussion about the benefits of combining eigenoptions with the option keyboard, and potential avenues for future work.

\subsection{Options Combined through the Option Keyboard are Diverse} \label{subsec:visualize_ok}

We define the set of basis options to be the first ten eigenoptions and we linearly combine these options with the option keyboard. Because there are infinite possible weight combinations, we constrain ourselves to $\{0,1\}$ combinations at first, meaning weights can be either $0$~or~$1$, and later, to $\{-1, 0, 1\}$ combinations, as shown in Algorithm~\ref{alg:ok_eigenoptions}. We perform our experiments in the four-room domain and in an open $10\times10$ gridworld, which we name open-room (see Figure~\ref{fig:eigenoptions_openroom}). The latter is particularly useful for this set of experiments because it allows us to build intuitions about the algebra of the options without being distracted by walls and asymmetries. For reference, some of the eigenoptions used as basis options are depicted in Figures~\ref{fig:eigenoptions_openroom} and~\ref{fig:evolution_policies_eigenoptions}, the latter in Appendix~\ref{sec:appendix_evolution}.

\begin{algorithm}[t]
\caption{OK-Eigenoptions}\label{alg:ok_eigenoptions}
\KwInput{$\Omega$ \Comment*[r]{Base options}}
\hspace{1.2cm} $C \in \mathbb{R}^{|\Omega| \times |\mathscr{S}| \times |\mathscr{A}| }$ \Comment*[r]{Reward function used to generate each option in $\Omega$}
\hspace{1.2cm} $\gamma$ \Comment*[r]{Discount factor}
\KwOutput{$\Omega_{\mathbf{w}}$ \Comment*[r]{Options generated by the OK from the base options}}

\vspace{0.2cm}
$\triangleright$ Compute the matrix of Q-values induced by each intr. reward $r^{\mathbf{e}_j}$ and option $\omega^{\mathbf{e}_i}$\\

\For{$\omega^{\mathbf{e}_i} \textrm{\normalfont \textbf{ in }} \Omega$}{
    \For{$r^{\mathbf{e}_j} \textrm{\normalfont \textbf{ in }} C$}{
        $Q_{\pi_{\omega^{\mathbf{e}_i}}}^{r^{\mathbf{e}_j}} \gets \textrm{PolicyEvaluation}(\pi_{\omega^{\mathbf{e}_i}}, r^{\mathbf{e}_j}, \gamma)$
    }
}

\vspace{0.2cm}
$\triangleright$ Using the option-keyboard, generate combination of eigenoptions with weights $\mathbf{w}$\\
$\Omega_{OK} \gets \emptyset$\\
\For{$\textrm{\normalfont all permutations } \mathbf{w} \textrm{\normalfont \textbf{ in }} [-1, 0, 1]^{|\Omega|}$}{
    $\Omega_{OK} \gets \Omega_{OK} \ \cup$ Option Keyboard($Q_{\pi_{\omega^{\mathbf{e}}}}^{r^{\mathbf{e}}}$, \textbf{w})
}
\end{algorithm}

\begin{figure}[t]
     \centering
     \includegraphics[width=\textwidth]{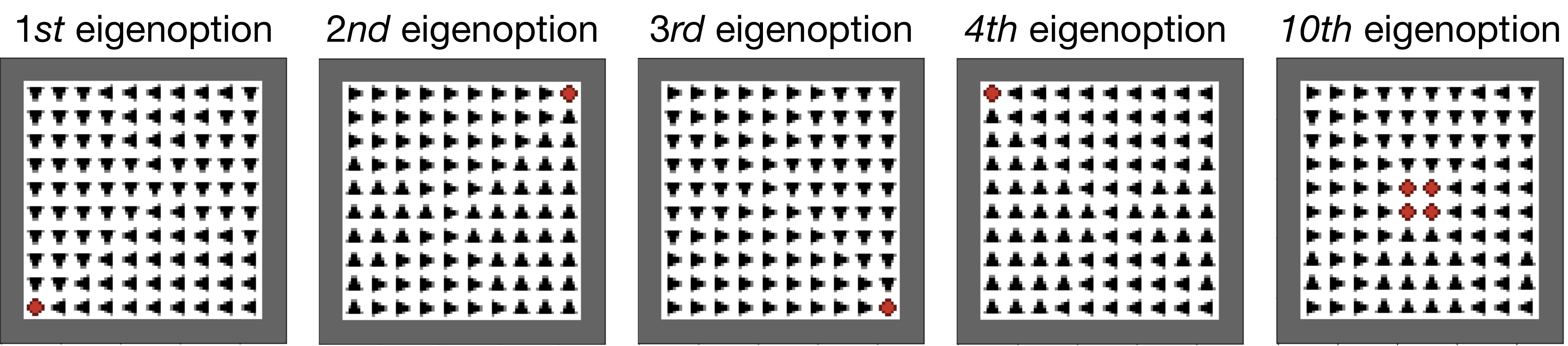}
     \caption{Eigenoptions discovered in the open-room. As in previous results, their ordering is defined by the eigenvalues corresponding to the eigenvectors that gave rise to each option. Both directions of the eigenvectors are taken into consideration.}
     \label{fig:eigenoptions_openroom}
\end{figure}

\begin{figure}[t]
     \centering
         \centering
         \includegraphics[width=\textwidth]{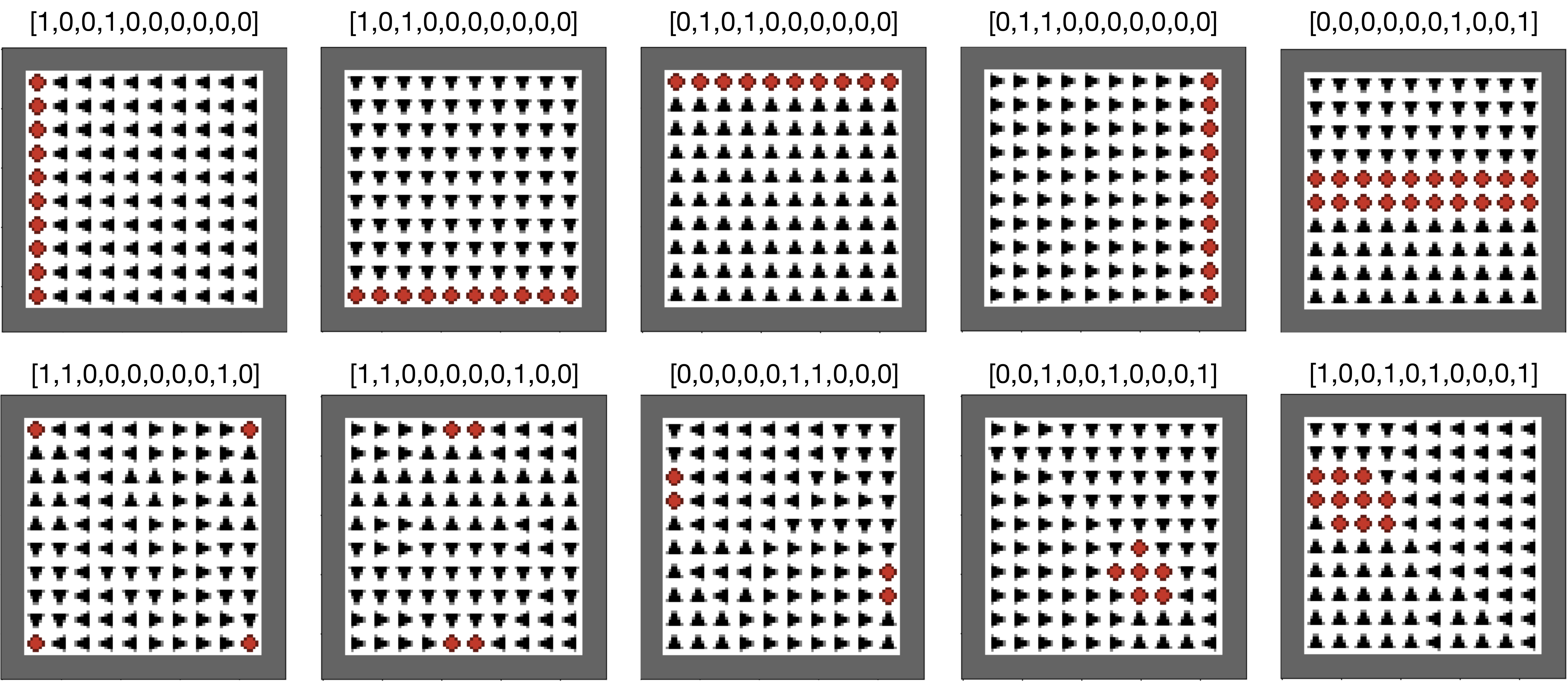}
     \caption{Options obtained by combining eigenoptions with the option keyboard. The weights used to generate them are above the option, where the $i$-th entry corresponds to the weight given to the $i$-th eigenoption.}
     \label{fig:ok_eigenoptions_openroom}
\end{figure}

In the open-room, when combining the first four eigenoptions in sets of two, the agent recovers cardinal directions, as shown in Figure~\ref{fig:ok_eigenoptions_openroom}. This is an interesting result because it shows a general method naturally discovering important abstractions. Besides that, we also observe the union of different options (e.g., going to the closest corner) and, once an eigenoption that takes the agent to the center of the room is available, the option keyboard generates combinations that together terminate in most states in the environment. In the four-room domain, we can draw similar conclusions, despite the asymmetric walls preventing behaviors as interpretable as those in the open-room. We see options that take the agent to specific walls/directions, to specific rooms, and to bottleneck states, as shown in Figure~\ref{fig:ok_eigenoptions_four_rooms}.

\begin{figure}[t]
     \centering
         \centering
         \includegraphics[width=\textwidth]{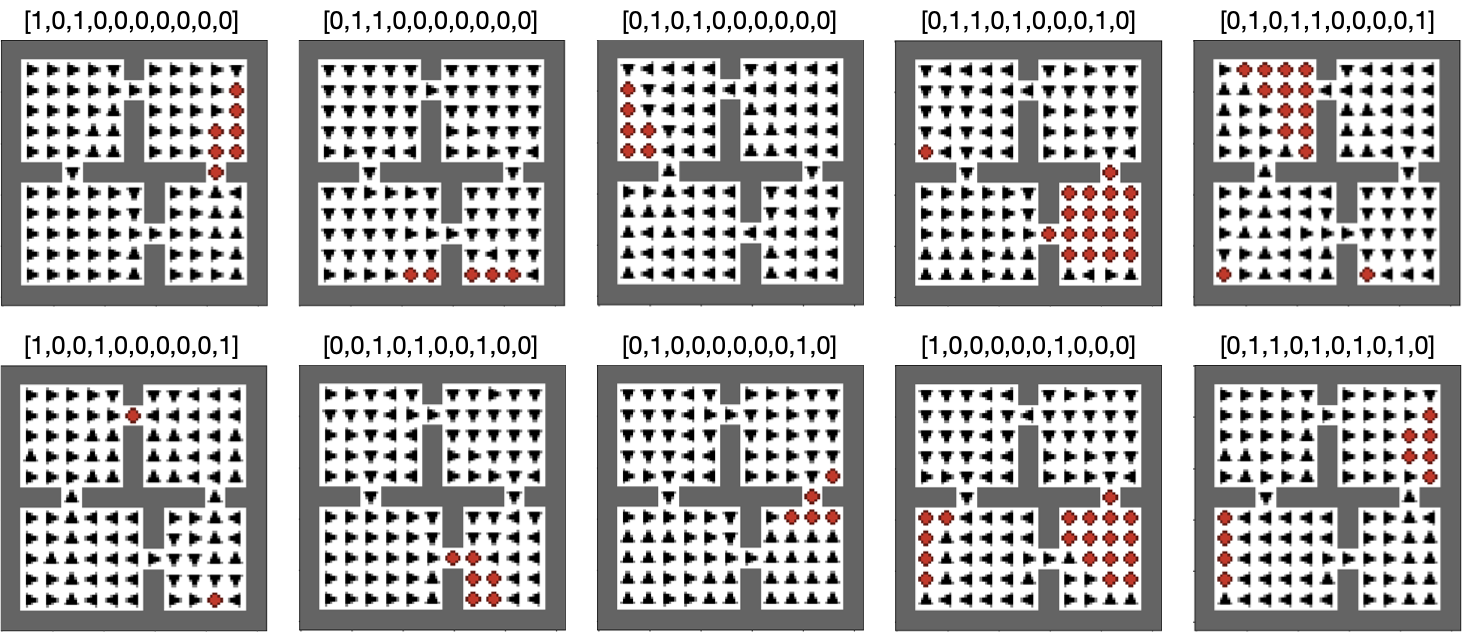}
     \caption{Options obtained by combining eigenoptions with the option keyboard. The weights used to generate them are above the option, where the $i$-th entry corresponds to the weight given to the $i$-th eigenoption.}
     \label{fig:ok_eigenoptions_four_rooms}
\end{figure}

\clearpage

The option keyboard also leads to a combinatorial explosion of new options, even when we only consider $\{0,1\}$ combinations. This is shown in Figure~\ref{fig:num_new_options_openroom}, which reports the number of \emph{unique} options generated by the option keyboard. The number of new options is not $2^n$, where $n$ is the number of basis options, because two options can be added together to cancel each other, for example, when they are derived from both directions of the same eigenvector. Additionally, different combinations can lead to the same option.

\begin{figure}[t]
     \centering
         \begin{subfigure}[b]{0.49\textwidth}
            \includegraphics[width=\textwidth]{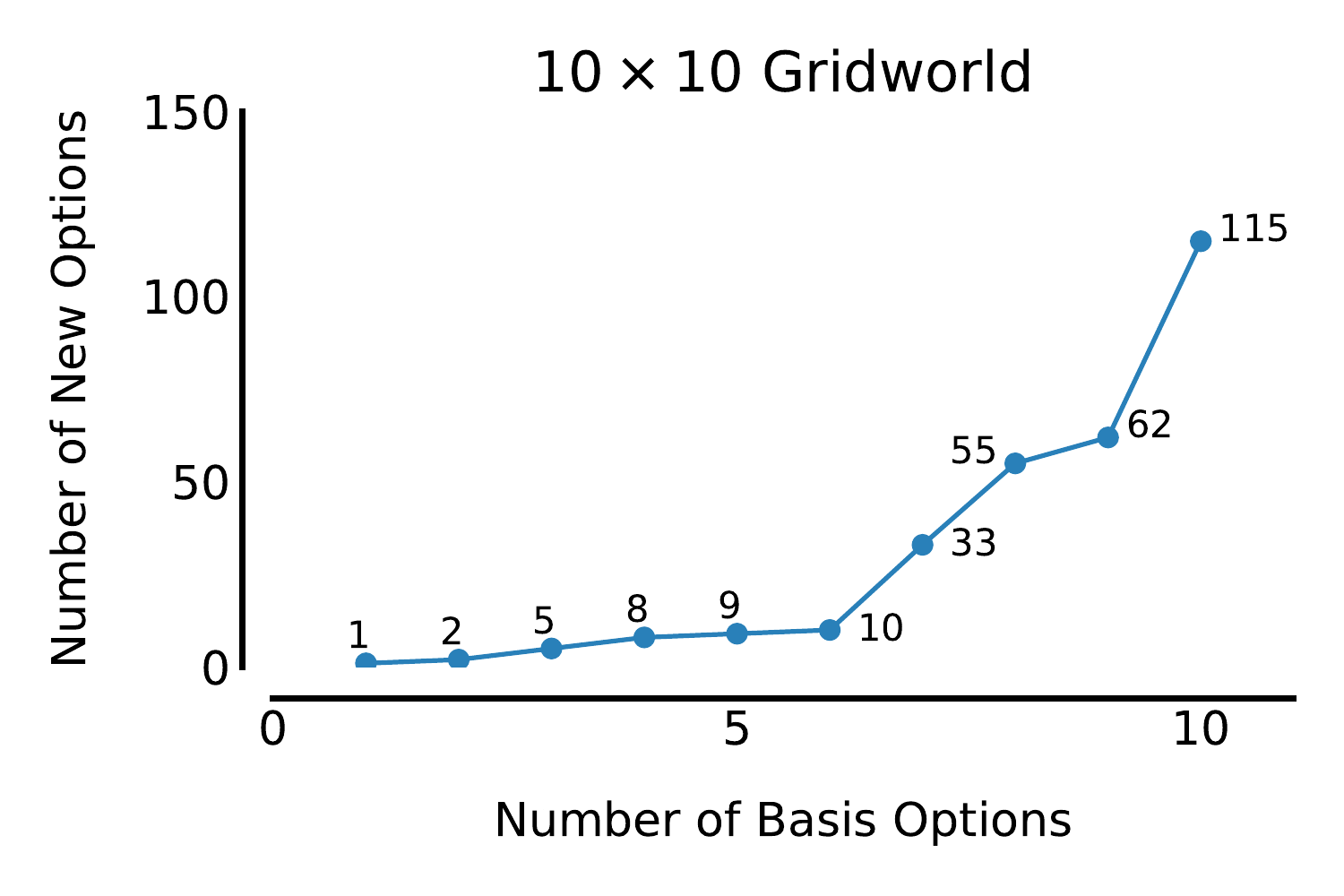}
         \end{subfigure}
         \begin{subfigure}[b]{0.49\textwidth}
            \includegraphics[width=\textwidth]{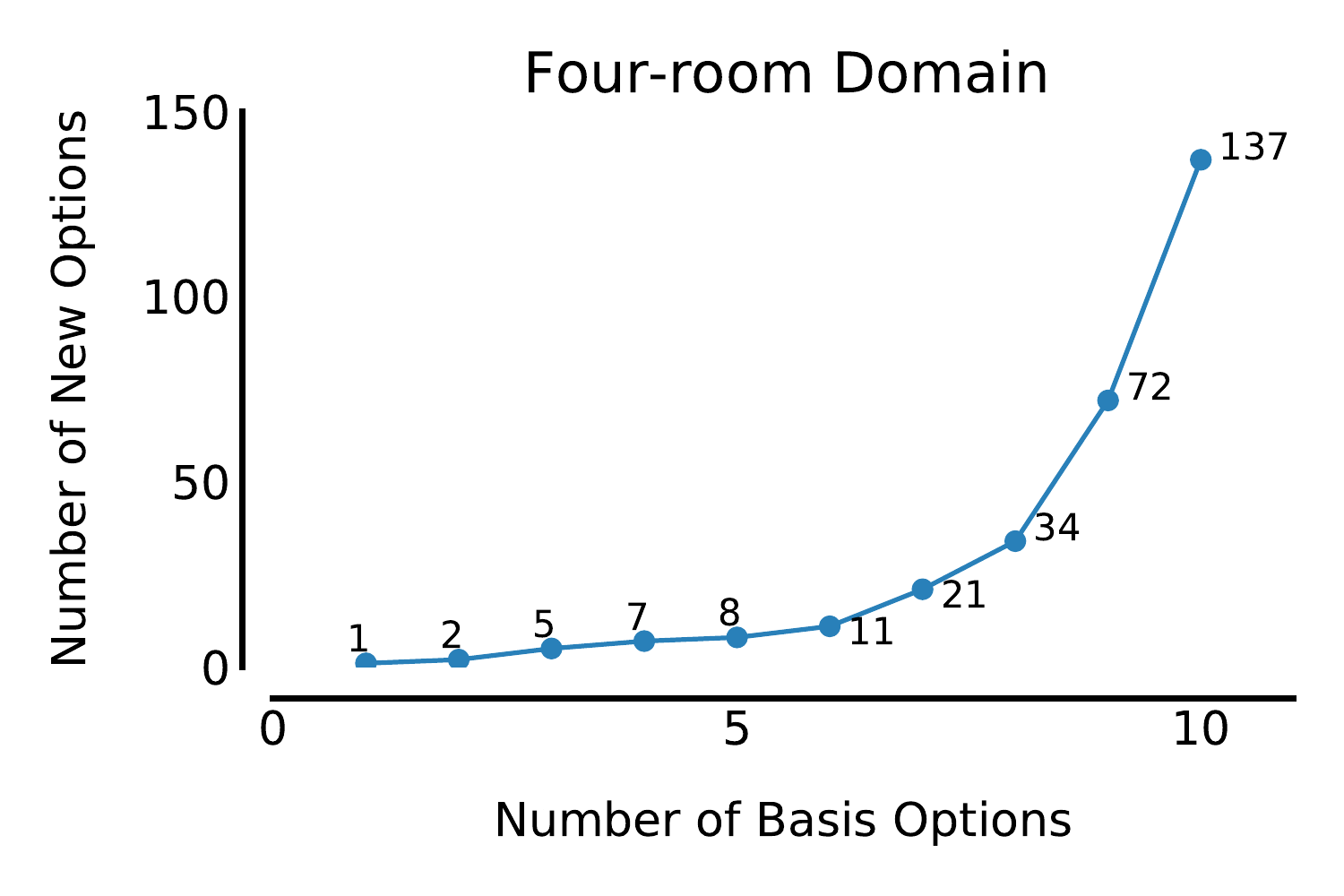}
         \end{subfigure}
     \caption{Number of unique options generated by combining eigenoptions. We consider two options to be the same if they have the same set of terminal states. In the open-room, for example, the $1st$ eigenoption leads to a single option: itself. Adding a $2nd$ eigenoption only leads to an extra option because the $1st$ and $2nd$ eigenoptions stem from the same eigenvector, thus they cancel each other when combined. Adding a $3rd$ eigenoption leads to two extra options (combination of the $1st$ and $3rd$, as well as the $2nd$ and $3rd$), totalling $5$ options (instead of $3$). This difference becomes starker as more basis options become available.}
     \label{fig:num_new_options_openroom}
\end{figure}

So far we have shown that combining eigenoptions with the option keyboard leads to interesting, semantically meaningful, options, as well as a large number of options.  We conclude this section showing these options are diverse. We do so by looking at the frequency at which these options terminate in different states. Figures~\ref{fig:heatmap_openroom} and~\ref{fig:heatmap_four_rooms} depict heatmaps contrasting the set of terminal states induced by eigenoptions and those induced by the combinations of those eigenoptions. The numbers in each tile report the number of times, across the considered options, that the corresponding state is a terminal state. Colors represent the relative frequency of termination across all states.

It is reassuring to see that the few basis options depicted on the left of Figures~\ref{fig:heatmap_openroom} and~\ref{fig:heatmap_four_rooms} can be combined to generate the diverse behaviors depicted on the right. In these environments, even when considering only $\{0,1\}$ combinations, ten options are enough for the option keyboard to generate options that visit most states in the environment. Eigenoptions are orthogonal \emph{bases of behavior} that span most of the behaviors one would be interested in. In the open-room, for example, eigenoptions terminating in $16$ states end up being combined to terminate in $96$ states. In the next section we show this diversity in fact allows the agent to better explore the environment.

\subsection{Options Combined through the Option Keyboard Improve Exploration} \label{subsec:ok_diffusion_time}

In this section, we show that the diversity introduced by the option keyboard in fact impacts an agent's ability to explore the environment. Figures~\ref{fig:dt_openroom_ok_eigenoptions} and~\ref{fig:dt_four_rooms_ok_eigenoptions} depict the diffusion time induced by the first ten eigenoptions and by the options generated by the option keyboard  when using the same eigenoptions as basis options. We term \emph{OK-Eigenoptions [0, 1]} the set of options generated by $\{0,1\}$ combinations of eigenoptions, and \emph{OK-Eigenoptions [-1, 0, 1]} the set of options generated by $\{-1, 0, 1\}$ combinations, which we discuss later.

We compare the diffusion time induced by these different option sets based on the number of basis options (i.e., eigenoptions) in them. This might seem unfair at first, as more options are used by OK-Eigenoptions, but this is exactly one of the benefits of the option keyboard: with an almost negligible computational cost, and no additional agent-environment interactions, a large combinatorial counterpart of the original options becomes available to the agent.

\begin{figure}[t]
     \centering
         \centering
         \includegraphics[width=0.65\textwidth]{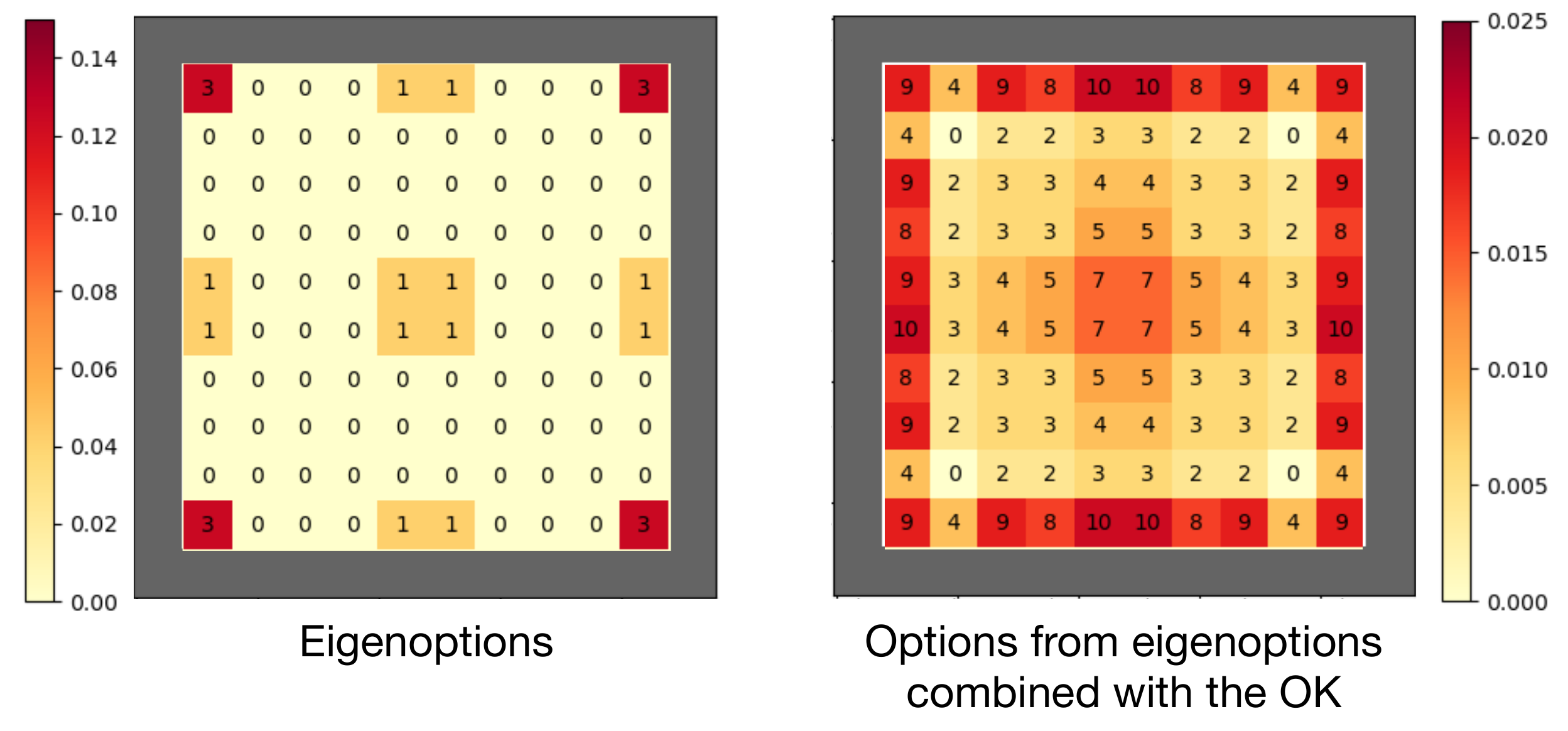}
     \caption{Frequency options terminate in each state in the open-room domain.}
     \label{fig:heatmap_openroom}
\end{figure}

\begin{figure}[t]
     \centering
         \centering
         \includegraphics[width=0.65\textwidth]{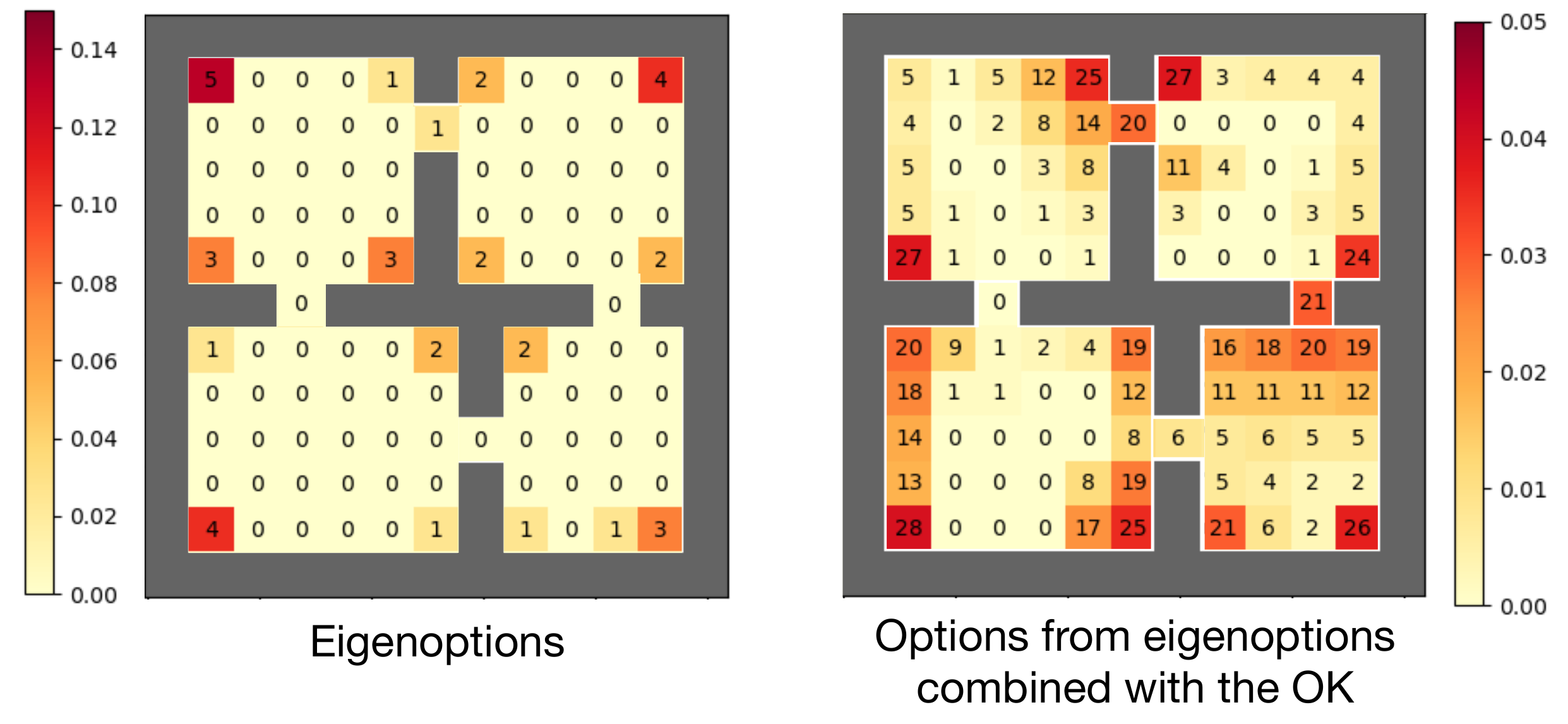}
         \caption{Frequency options terminate in each state in the four-room domain.}
         \label{fig:heatmap_four_rooms}
\end{figure}

Shortly, using the option keyboard to augment an agent's option set can drastically reduce the number of eigenoptions the agent needs to effectively explore the environment. In the open-room and four-room domains, the agent needs $10$ and $12$ eigenoptions to make the induced exploration more effective than that obtained by a uniform random policy. When augmenting the agent with the options generated by the option keyboard, this number is reduced by $30\%$ and $42\%$, respectively. Also, the median diffusion time is not affected when adding the options generated by the option keyboard, suggesting that these additional options improve exploration by making eigenoptions more robust.

\begin{figure}[t]
     \centering
     \includegraphics[width=\textwidth]{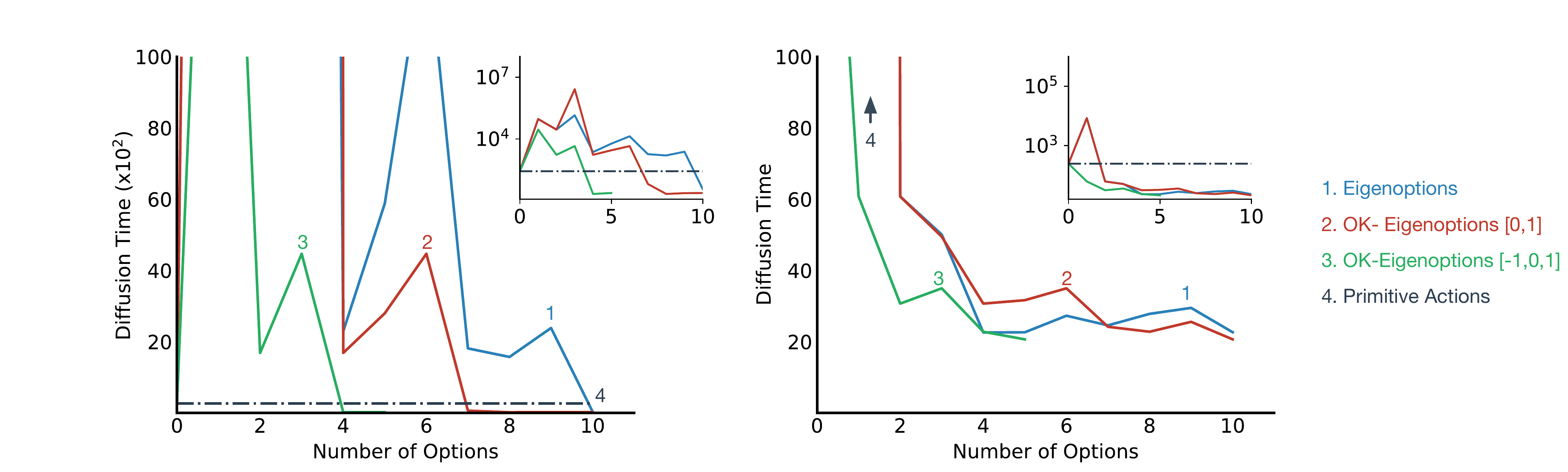}
     \caption{Average and median diffusion time in the open-room domain. Each curve depicts a function of the number of primitive options the agent has access to, but the actual number of options used may vary, as described in the text. The performance of primitive actions is not depicted in the plot on the right because it is out of the reported range. The scale on the left is $100$ times bigger than in the other plots due to the big impact individual options have at first.}
     \label{fig:dt_openroom_ok_eigenoptions}
\end{figure}

\begin{figure}[t]
     \centering
     \includegraphics[width=\textwidth]{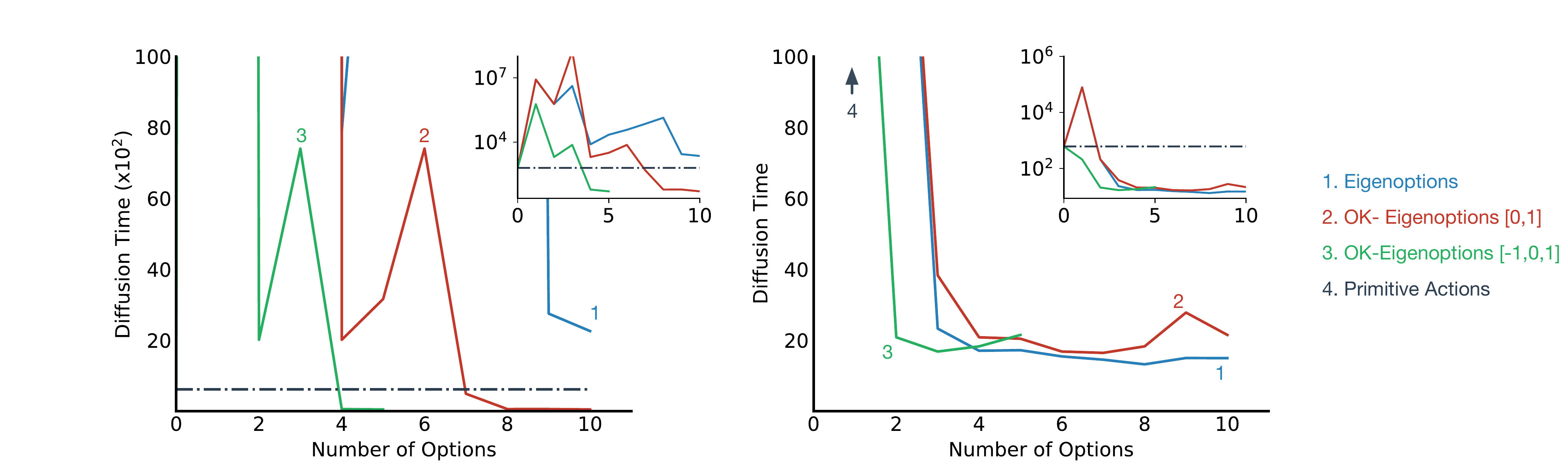}
     \caption{Average and median diffusion time in the four-room domain. Each curve depicts a function of the number of primitive options the agent has access to, but the actual number of options used may vary, as described in the text. The performance of primitive actions is not depicted in the plot on the right because it is out of the reported range. The scale on the left is $100$ times bigger than in the other plots due to the big impact individual options have at first.}
     \label{fig:dt_four_rooms_ok_eigenoptions}
\end{figure}

Finally, it is important to stress that we obtained these results using only $\{0,1\}$ weights. A trivial improvement to what we have done so far is to generate only one (instead of two) eigenoption from each eigenvector of the SR. The option corresponding to the opposite direction of each eigenvector can be obtained with the option keyboard by assigning a $-1$ weight to the corresponding eigenoption. This immediately reduces the number of options to be learned by half! We use the green line in Figures~\ref{fig:dt_openroom_ok_eigenoptions} and~\ref{fig:dt_four_rooms_ok_eigenoptions} to make this point obvious. Shortly, by leveraging the option keyboard in the simplest possible way we are able to show, in both environments, that $4$ eigenoptions are enough to improve exploration beyond what the uniform random policy does. This is an important result because, for the first time, we are able to obtain results that match the intuition that four options should suffice in these environments. It was not clear how to discover (and combine) such options until now.

\subsection{A Big Picture View of Combining Eigenoptions with the Option Keyboard} \label{subsec:conclusion_exp_ok}

The combination of eigenoptions and the option keyboard allow these approaches to heavily benefit from each other. Eigenoptions, by construction, provide a diverse set of behaviors to the agent, but dozens of options are necessary for better exploration. The option keyboard, on the other hand, combines different options to generate a large set of new behaviors, but only once basis options are provided. The combination of the eigenoptions and the option keyboard leads to more diverse behavior that improves exploration while drastically reducing the number of options that need to be discovered. The option keyboard is also able to drastically reduce the computational cost of learning eigenoptions.

Our experiments demonstrated the potential of combining an option discovery method based on the SR with the option keyboard. There are multiple directions one could pursue in a future work. An interesting direction is how to define the weights used by the option keyboard. While we used $\{-1, 0, 1\}$ weights, different weights are likely to be more expressive, generating even more diverse behaviors. From a theoretical perspective, it would be interesting to characterize the number of basis options needed by an agent in order to, for example, ensure that there is at least one combined option that terminates in each state. This number could also be used as a measure of task complexity (or similarity).

From an algorithmic perspective, it would be valuable to understand the impact of using the option keyboard to combine options estimated online, and how it can be used with multiple iterations of the ROD cycle. It would also be useful to see these ideas evaluated in the function approximation case. An immediate challenge in this case would be to identify equivalent options, something trivially done in the tabular case. Scaling these ideas to more challenging domains might also require non-linear function approximation and representation learning, problems tackled by deep reinforcement learning. As previously mentioned, for clarity purposes, we decided to not discuss these ideas in this paper, but there are several advances and successes of this line of work in the deep reinforcement learning case, such as the introduction of optimization objectives that allow non-linear function approximators such as neural networks to approximate the eigenvectors of the graph Laplacian~\citep{Pfau18,Wang21,Wu18}, and their use for option discovery in problems which require function approximation~\citep{Jinnai20}. We discuss some of these ideas in the next section.

\section{Related Work}~\label{sec:related_work}

We have discussed the role of the SR for temporal abstraction in reinforcement learning, focusing on options for temporally-extended exploration. Additionally, we explored how the SR can be used to seamlessly combine different options without additional learning. It is unfeasible to thoroughly discuss even all the methods for options discovery that aim at exploration, let alone the diverse set of methods that do not. Instead, here we outline only the main ideas behind these other approaches. Besides additional approaches, we also discuss existing extensions of the ideas we presented to the function approximation setting, and we conclude by drawing connections between the topics we discussed and recent results in neuroscience.

Naturally, there are also other approaches, not based on the SR, that share similar intuitions to those described here, such as discovering task-agnostic options that terminate in a diverse set of states. They achieve this in various ways, such as maximizing metrics like empowerment, diversity, and entropy~\citep[e.g.,][]{Eysenbach19,Gregor16,Hansen20}. However, to focus on approaches based on the SR, we do not further discuss them here.

\subsection{Additional Option Discovery Methods based on the SR}

Planning and faster credit assignment are common use cases for options. These ideas are often associated with bottleneck options, that is, options that lead to states that connect different closely connected regions of the environment. The SR has also been used by option discovery methods in this setting. \cite{Stachenfeld14,Stachenfeld17}, for example, explicitly searched bottleneck states by using the eigenvectors of the SR to obtain an approximate solution to the $k$-way normalized min-cut problem. In this case, the eigenvectors of the SR are used in a different way from what we described so far, with positive and negative elements of an eigenvector approximating different partitions, a known result from spectral graph theory~\citep{Shi00}. \cite{Kulkarni16a}, in the context of deep reinforcement learning, also discovered bottleneck options through normalized cuts. Another relevant idea in this context is the algorithm named \emph{successor options}~\citep{Ramesh19}. This method clusters the SR vectors and defines the clusters centroids as the options' terminal states. In this case, even though the method is not explicitly designed to seek for bottleneck states, the empirical evidence provided shows that successor options tend to find bottleneck states much more often than eigenoptions~\citep[see Figure~5 by][]{Ramesh19}.

In the context of bottleneck options and the framework presented in Section~\ref{sec:framework}, Continual Curiosity-driven Skill Acquisition~\citep[CCSA;][]{Kompella17} is also relevant to our discussion. CCSA also discovers options that maximize an intrinsic reward obtained from a learned representation, which in this case is obtained from Slow Feature Analysis~\citep[SFA;][]{Wiskott02}. Importantly, \cite{Sprekeler11} has shown that, given a specific choice of adjacency function, proto-value functions are equivalent to SFA. SFA becomes an approximation of proto-value functions if the function space used in the SFA does not allow arbitrary mappings from the observed data to an embedding. Recall that one can see proto-value functions as a special case of the eigenvectors of the SR~\citep{Machado18b}, making SFA, by transitivity, connected to the ideas discussed in this paper.

In the context of options for temporally-extended exploration, \cite{Bar20} recently introduced \emph{diffusion options}, which were inspired by the ideas discussed here. Instead of looking at individual eigenvectors, this approach sets out to use the full eigenspectrum of the adjacency matrix. Specifically, it uses the eigenvalues as weights to linearly combine the right and left eigenvectors of the non-symmetric lazy random walk matrix, a variation of the graph Laplacian matrix we discussed in Sections~\ref{sec:background} and~\ref{sec:successor_representation}. Despite the intuitive similarity between the eigenvectors of the non-symmetric lazy random walk matrix and the eigenvectors of the SR, there is still not a formal connection between these two objects. However, it is interesting to note that \citeauthor{Bar20}'s empirical results corroborate what we present in this paper. Shortly, (1) these options provide temporal-extended exploration, (2) eigenoptions tend to be more effective than covering options, and (3) diffusion options also needs to use $10-20$ eigenvectors to be effective in gridworld domains, similar to our results in Section~\ref{sec:experiments_temporally_extended_exploration}. \looseness=-1

Finally, the eigenoption-critic \citep{Liu17} is also relevant to this overview. It combines eigenoptions to the option-critic architecture~\citep{Bacon17} for a single online phase for option discovery and option learning, while also taking the reward generated by the environment into consideration. Moreover, it is applicable to settings in which non-linear function approximation is required, which we discuss in the next section. Importantly, it also extends eigenoptions to continuous state spaces with the Nystr\"om approximation.

Successor features and the other extensions to the function approximation case, which we discuss in the next section, naturally allow us to extend ideas such as eigenoptions and covering options to continuous state spaces. The results in Figure~\ref{fig:results_cycle_exploration}, for example, were generated in a continuous state space. But the Nystr\"om method is another interesting approach that has been used in the past to extend eigenoptions to the continuous case~\citep[c.f.][for details]{Mahadevan07}. Shortly, it interpolates the value of eigenvectors computed on observed states to novel states, addressing the issue that the exact same state is rarely visited twice in continuous state spaces. For the normalized Laplacian, this is achieved with $$\mathbf{e}_i(s) = \frac{1}{1-\lambda_i} \sum_{i : |s - s_k| < \epsilon} \frac{w(s, s_k)}{\sqrt{d(s)d(s_k)}} \mathbf{e}_i(s_k),$$
where $\lambda_i$ and $\mathbf{e}_i$ denote the $i$-th eigenvalue and eigenvector, $\epsilon$ is a threshold parameter, $s_k$ is the state observed in the past that is closest to the new observation $s$; $w(s, s_k)$ are the weights from $s$ to $s_k$, and $d(s) = \sum_{k:|s-s_k| < \epsilon} w(s, s_k)$, is the degree of $s$. Note that $d(s)$ is determined when a new state $s$ is encountered while $d(s_k)$ is computed in the graph construction phase~\citep{Liu17}.

\subsection{Function Approximation and the SR}~\label{sec:related_work_fa}

In this paper, we focused on the tabular case for conceptual clarity. Nevertheless, in several problems one cannot uniquely identify the states in the environment and function approximation is required. Scaling these ideas up has been an active topic of research, with solutions proposed to both the linear and non-linear function approximation cases. 

In the \textbf{linear function approximation} case, given a set of features $\bm{\phi}(\cdot)$, \cite{Machado17} proposed to generate the options' intrinsic reward function with the singular vectors of the matrix obtained from the difference between current and previous observations. If all transitions in the graph are sampled once, for tabular representations, this matrix recovers the same options we obtain with the combinatorial Laplacian. This is formalized in Theorem~\ref{thlinearfa} in Appendix~\ref{app:proofs}.

In the \textbf{non-linear function approximation} case, when neural networks are used to approximate the value function, there are many more solutions to the problem of scaling up the ideas we presented here. A direct way of doing so is by having the network also outputting successor features~\citep[e.g.][]{Barreto20,Borsa19,Hansen20,Hoang21,Kulkarni16a,Liu21,Machado18b,Machado20}. This is often done by training the neural network to also minimize the loss
\begin{eqnarray}
\mathcal{L}(s, s') = \mathbb{E} \Bigg[ \bigg(\bm{\phi}(s) + \gamma \bm{\psi}\big(\bm{\phi}(s')\big) - \bm{\psi}\big(\bm{\phi}(s)\big) \bigg)^2\Bigg]\label{eq:loss_sf},
\end{eqnarray}
or one of its variations, where $\bm{\psi}$ denotes successor features. Sometimes it is assumed the representation $\bm{\phi}(\cdot)$ is known beforehand~\citep[e.g.][]{Borsa19}, while other times it is defined as the output of an inner layer of the neural network~\citep[e.g.,][]{Hoang21,Machado18b}. The latter can be unstable and, because the representation is also being learned, one needs to be careful since $\bm{\phi}(\cdot) = \mathbf{0}$ is a fixed point of this minimization problem. In this case, it is common to prevent gradients to backpropagate to the representation layer, and to introduce auxiliary tasks to shape the representation learning process~\citep{Jaderberg17}.

Alternatively, researchers have also been using results from graph drawing theory~\citep{Koren03} to design loss functions that allow the network to explicitly estimate the eigenvectors of the Laplacian~\citep[e.g.,][]{Erraqabi21,Wang21,Wu18}. The loss function originally proposed, with later methods introducing variations, was
\begin{eqnarray}
G(\mathbf{e}_1, \ldots \mathbf{e}_d) &=& \frac{1}{2}\mathbb{E} \Bigg[\sum_{k=1}^d \Big(\mathbf{e}_k(u) - \mathbf{e}_k(v)\Big)^2\Bigg] \nonumber\\
& & \ \ \ \ \ \ \ \ \ \ \ \ \ + \beta \mathbb{E} \Bigg[\sum_{j,k} \Big(\mathbf{e}_j(u)\mathbf{e}_k(u) - \delta_{jk} \Big)\Big(\mathbf{e}_j(v)\mathbf{e}_k(v) - \delta_{jk} \Big) \Bigg],
\end{eqnarray}
where $\mathbf{e}_1, \ldots, \mathbf{e}_d$ are the first $d$ eigenvectors, $u$ and $v$ are different states, $\beta$ is a penalty weight parameter, and $\delta_{jk}$ is a soft constraint used to capture the orthogonality constraint between different eigenvectors.

Intuitively, the loss function above has an attractive and a repulsive term, which is common in contrastive losses~\citep[e.g.,][]{Li21}. The first term, the attractive one, tries to ensure consecutive states are put together in the embedding, as in the SR. The repulsive term repels the embedding of states independently sampled from the stationary distribution of the induced Markov chain, trying to make eigenvectors orthogonal to each other. In the first term, the state $u$ is assumed to be (uniformly) sampled from the state distribution and $v$ is defined by the environment's dynamics. In the second, $v$ is randomly sampled. We did not depict this in the expectations to avoid cluttering the equations. \citeauthor{Wu18} presents a precise discussion and the accompanying derivations. This approach, based on the graph drawing objective, has also been used to scale covering options to the deep reinforcement learning case~\citep{Jinnai20}. This is what was used to generate Figure~\ref{fig:results_cycle_exploration}.

Finally, there have also been attempts to scale up the ideas discussed here by generating abstract transition graphs, with a neural network being used to generate the state abstractions~\citep{Hoang21,Mendonca19}. The work by \cite{Hoang21} is particularly interesting as it can be seen as instantiating multiple iterations of the ROD cycle while dealing with function approximation. Shortly, \citeauthor{Hoang21} use a neural network to learn successor features, as outlined in Eq.~\ref{eq:loss_sf}, and they define successor feature similarity, which is the dot-product between two SFs, to measure the distance between two different states; using those to define landmark states that the agent plans to visit. When in a landmark (abstract) state considered to be in the frontier of the agent's experience, the agent starts to act randomly to further explore the environment.

\subsection{The Relationship of the SR to Other Ideas in RL and Other Fields}

\vspace{0.3cm}

As mentioned in Section~\ref{sec:successor_representation}, the SR is directly related to several other ideas in reinforcement learning. As already discussed, it is related to proto-value functions~\citep{Machado18b} and slow-feature analysis~\citep{Sprekeler11}, and it can be seen as encoding the LSTD matrix~\citep{Lagoudakis03}. Moreover, the SR can be seen, for example, as a form of dual approach to value-function based methods \citep{Wang07}. In this formalism, which has been shown to avoid the risk of divergence, one maintains an explicit representation of visit distributions instead of value functions. For exploration, it has been shown that the norm of the SR implicitly encodes state-visitation counts~\citep{Machado20}. 

It is particularly interesting to observe that methods introduced for computational reinforcement learning are related to recent results in neuroscience and human behavior. From a neuroscience perspective, \cite{Stachenfeld14,Stachenfeld17} have suggested the SR is able to model activations in the hippocampus. Specifically, they argue that hippocampal place fields encode a predictive representation of the current and future states under the current transition distribution. Additionally, and directly related to the option discovery methods we presented here, \cite{Stachenfeld14,Stachenfeld17} present results demonstrating that the eigenvectors of the SR model grid cells activations, also suggesting that the entorhinal cortex may use that information to aid hierarchical reinforcement learning. This result is motivated by the fact that the eigenvectors of the SR allow one to capture natural boundaries, potentially enconding metric information about the space. From a planning perspective, they show how one can use the eigenvectors of the SR to identify bottleneck states, which are convenient waypoints for likely being traversed along many optimal trajectories~\citep[see][]{Solway14}.

On a higher level of abstraction, the SR has also been used to explain human behavior, specifically the notion of choice in humans. \cite{Momennejad17}, for example, has suggested the SR can be used to introduce a ``subtler, more cognitive notion of habit''. They argue that the SR can be seen as an underlying computational procedure that allows humans to balance the flexibility of model-based methods and the efficiency of model-free methods. Even more directly related to the concepts we presented, \cite{Tomov21} also discusses the SR in the context of human decision making, suggesting it can be used, alongside GPI, as a model for human decision making in multi-task scenarios.

In summary, the ideas discussed in this paper around the SR, such as using it as a representation, for efficient transfer learning, or for option discovery, seem to model well data collected from intelligent, biological animals. Although this does not necessarily mean that these ideas are the correct way of tackling the problems discussed here, these results are encouraging evidence that it might be worth investigating these ideas further.

\clearpage

\section{Conclusion}~\label{sec:conclusion}

In this paper, we discussed the role of the successor representation (SR) when using temporal abstractions in reinforcement learning. We presented a general framework for option discovery, termed Representation-driven Option Discovery cycle, which follows a constructivist approach, and we examined two instantiations of this cycle, executed for different numbers of iterations. These instantiations use the eigenvectors of the SR to discover options for temporally-extended exploration, thoroughly evaluating them and shedding light on decisions every option discovery method should make, such as on how to define the initiation set and termination condition of each option. Moreover, we discussed how the SR can also be used to address the inevitable trade-off every agent is subject to: to discover more options in order to have access to a more expressive set of behaviors, or to restrict them to facilitate learning. This is done through the option keyboard, which extends, without additional learning, a finite set of options to a combinatorially large counterpart. The empirical evaluation of using the option keyboard to combine options discovered with the SR provides encouraging evidence of the synergy of these approaches, with this combination drastically reducing the computational cost of option learning while increasing the expressivity of the options available to the agent.

The goal of this paper was to provide a unified perspective over different approaches for temporal abstraction in reinforcement learning, showing how they can all be cast as discovering options from the SR, and how these ideas have evolved throughout the time. Naturally, there are several extensions of these approaches in order to scale them up to problems with large state spaces. Moreover, there is increasing evidence that the SR also plays an important role, at different levels of abstraction, in human and animal decision-making. We believe using the SR as the main substrate for temporal abstraction is a promising research direction and we hope this paper serves as a catalyst to this line of work. In particular, we believe that a cycle such as the one described here might end up being a core component of agents that are capable to continually learn and to acquire increasingly complex skills. Intelligent agents should be constantly acquiring new data, and this data should allow the agent to refine its representation of the world, which in turn should serve to inform which temporal abstractions an agent should learn, further empowering the agent's ability to collect new data, in a virtuous cycle of discovery.

\section*{Acknowledgements}

This work was partially developed while Marlos C. Machado was at Google Research, Brain Team. The authors would like to thank Tom Schaul, Adam White, and the anonymous reviewers for their thorough feedback on an earlier draft; and Dale Schuurmans, Yuu Jinnai, Marc G. Bellemare, and Patrick Pilarski for useful discussions. Marlos C. Machado, Doina Precup, and Michael Bowling are supported by a Canada CIFAR AI Chair.

\vfill

\appendix

\section{Proto-value Functions and their Equivalence to the Eigenvectors of the SR}\label{app:pvfs}

As mentioned in the main paper, the SR is present in several RL algorithms, either explicitly or implicitly. An important result for this paper is that the eigenvectors of the SR are equivalent to proto-value functions \citep[PVFs;][]{Mahadevan05,Machado18b}.\footnote{This holds when the SR is defined w.r.t. the uniform random policy in a deterministic and symmetric environment. The cardinality of the action set should also be the same across all states.} Because these properties were first discussed in the PVFs literature, we further discuss PVFs and their properties here.

Proto-value funtions~\citep[PVFs;][]{Mahadevan05} were originally introduced as representations that reflect the geometry of the environment. Specifically, they are basis functions based on the notion of diffusion models~\citep{Coifman05,Kondor02}, which capture how information flows in the environment by modeling it as a graph connecting states that are one action away from each other. This is motivated by the fact that value functions can be seen as the result of rewards diffusing through the state space, governed by the environment dynamics~\citep{Mahadevan07}. Formally, PVFs are the eigenvectors of a symmetric diffusion operator such as the \emph{normalized Laplacian},

\begin{eqnarray}
\mathbf{L} = \mathbf{D}^{-\frac{1}{2}}(\mathbf{D} - \mathbf{W})\mathbf{D}^{-\frac{1}{2}},
\end{eqnarray}
where $\mathbf{W}$ is the graph's adjacency matrix and $\mathbf{D}$ the diagonal matrix whose entries are the row sums of $\mathbf{W}$. In the simplest setting, for states $s_i$ and $s_j$, the $ij$-th entry of matrix $\mathbf{W}$ is
$$\mathbf{W}_{ij}=\left\{\begin{array}{rl}
1,&\mbox{if}\quad p (s_j | s_i, a) > 0,\\
0, &\mbox{otherwise},
\end{array}\right.
$$
for any action $a \in \mathscr{A}$. Notice the matrix $\mathbf{W}$ can potentially be extended to a weight matrix.

These diffusion models are tightly related to the random walk diffusion model $\mathbf{D}^{-1}\mathbf{W}$. A diffusion model works as a surrogate that is easier to estimate than the full transition matrix, while being useful for value function approximation because we can represent the value function as a linear combination of the eigenvectors of the transition matrix, as shown in Eq.~\ref{eq:value_function}. See the work by \citeauthor{Mahadevan07}~(\citeyear{Mahadevan07}) for a detailed discussion.

The formal result of equivalence between PVFs and the eigenvectors of the SR is below.

\begin{restatable}[\citeauthor{Machado18b} \citeyear{Machado18b}]{theorem}{thequivalence} Let $n$ be the  number of rows (and columns) of matrix $\mathbf{P_\pi}$, and let $i$ and $j$ denote indices such that $i + j = n + 1$. The $i$-th eigenvalue of the SR, defined w.r.t. a uniform random policy, and the $j$-th eigenvalue of the normalized Laplacian are related as follows when in a symmetric and deterministic environment:
$$\lambda_{\mbox{\tiny{PVF}}, j} = \Big[1 - (1 - {\lambda^{-1}_{\mbox{\tiny{SR}}, i}}) \gamma^{-1}\Big].$$
The $i$-th eigenvector of the SR, ${\bf e}_{\mbox{\tiny{SR}}, i}$, and the $j$-th eigenvector of the normalized Laplacian, ${\bf e}_{\mbox{\tiny{PVF}}, j}$, are related as follows:
$${\bf e}_{\mbox{\tiny{PVF}}, j} = (\gamma^{-1} \mathbf{D}^{1/2}) {\bf e}_{\mbox{\tiny{SR}}, i}.$$
\label{th:equivalence}
\end{restatable}
\begin{proof}
See Appendix~\ref{app:proofs}.
\end{proof}

Notice that while the eigenvectors of the normalized Laplacian and the eigenvectors of the SR are the same, the order in which they appear is flipped. The eigenvectors with corresponding lowest eigenvalues, when using PVFs, are equivalent to the eigenvectors with corresponding largest eigenvalues when using the SR. Importantly, the eigenvectors of the SR are equivalent to PVFs in a symmetric and deterministic MDP, that is, when every transition is reversible and action outcomes are deterministic. Thus, while the SR reduces to PVFs in a more restrictive case, it is also applicable to the more general case, allowing us to more easily capture stochasticity and asymmetries in the environment.

As discussed in the main paper, PVFs capture properties of the dynamics environment and this is the main motivation behind several option discovery methods. Importantly, PVFs were originally introduced as basis functions for function approximation and they ended up not being widely adopted. This paper did not study the use of PVFs as basis functions. This is a major difference from how PVFs were used in the past. In this paper we advocate for the use of these concepts for the discovery and combination of temporal abstractions. Moreover, the singular vectors of the SR can already be seen as a generalization of PVFs. Thus, the shortcomings of PVFs in the past should not be carried over when considering their use for temporal abstraction.

\section{Theoretical Results}~\label{app:proofs}

\vspace{-0.3cm}

We present the theoretical results below for completeness. These results were obtained in other works and we present them, as well as their proofs, in their original version. We cite them accordingly before each theorem. We refer the reader to the works by \cite{Machado17,Machado18b} for further details.

\begin{restatable}[\citeauthor{Machado17} \citeyear{Machado17}]{lem}{lemma_pvf1}\label{lemma:pvf1}
 Suppose \ $(\mathbf{I} + \mathbf{A})$ \ is \ a \ non-singular \ matrix, \ with $||\mathbf{A}|| \leq 1$. We have:
$$||(\mathbf{I} + \mathbf{A})^{-1}|| \leq \frac{1}{1 - ||\mathbf{A}||}.$$
\end{restatable}

\begin{proof}
\begin{align*}
(\mathbf{I} + \mathbf{A})(\mathbf{I} + \mathbf{A})^{-1} &= \mathbf{I}\\
\mathbf{I}(\mathbf{I} + \mathbf{A})^{-1} + \mathbf{A}(\mathbf{I} + \mathbf{A})^{-1} &= \mathbf{I}\\
(\mathbf{I} + \mathbf{A})^{-1} &= \mathbf{I} -  \mathbf{A}(\mathbf{I} + \mathbf{A})^{-1}\\
||(\mathbf{I} + \mathbf{A})^{-1}|| &= ||\mathbf{I} -  \mathbf{A}(\mathbf{I}+\mathbf{A})^{-1}||\\
                 &\leq ||\mathbf{I}|| + ||\mathbf{A} (\mathbf{I} + \mathbf{A})^{-1}||\\
                 &\leq 1 + ||\mathbf{A}||||(\mathbf{I} + \mathbf{A})^{-1}||
\end{align*}
\begin{align*}
||(\mathbf{I} + \mathbf{A})^{-1}|| - ||\mathbf{A}||||(\mathbf{I} + \mathbf{A})^{-1}||&\leq 1\\
(1-||\mathbf{A}||) ||(\mathbf{I} + \mathbf{A})^{-1}|| &\leq 1\\
||(\mathbf{I} + \mathbf{A})^{-1}|| &\leq \frac{1}{1 - ||\mathbf{A}||} && \text{if} \ \ ||\mathbf{A}|| \leq 1. \\
\end{align*}

Where the first inequality is due to the fact that $||\mathbf{A} + \mathbf{B}|| \leq ||\mathbf{A}|| + ||\mathbf{B}||$ and the second inequality comes from the fact that $||\mathbf{A B}|| \leq ||\mathbf{A}|| \cdot ||\mathbf{B}||$.
\end{proof}

\begin{restatable}[\citeauthor{Machado17} \citeyear{Machado17}]{lem}{lemma_pvf2}
\label{lemma:pvf2}
The induced infinity norm of $(\mathbf{I} -  \gamma \mathbf{T})^{-1}\mathbf{T}$ is bounded by
$$||(\mathbf{I} -  \gamma \mathbf{T})^{-1}\mathbf{T}||_\infty \leq \frac{1}{(1 - \gamma)}.$$
\end{restatable}

\begin{proof}
\begin{align*}
||(\mathbf{I} -  \gamma \mathbf{T})^{-1}\mathbf{T}||_\infty & \leq ||(\mathbf{I} -  \gamma \mathbf{T})^{-1}||_\infty||\mathbf{T}||_\infty  && \text{because} \ \  ||\mathbf{AB}||_{\infty} \leq ||\mathbf{A}||_{\infty} \cdot ||\mathbf{B}||_{\infty}\\
||(\mathbf{I} -  \gamma \mathbf{T})^{-1}\mathbf{T}||_\infty & \leq \frac{1}{1 - ||-\gamma \mathbf{T}||_\infty} ||\mathbf{T}||_\infty && \text{Lemma~\ref{lemma:pvf1}} \\
||(\mathbf{I} -  \gamma \mathbf{T})^{-1}\mathbf{T}||_\infty & \leq \frac{1}{1 - \gamma ||\mathbf{T}||_\infty} ||\mathbf{T}||_\infty && \text{because}  \ \ ||\lambda \mathbf{B}s|| = |\lambda| ||\mathbf{B}|| \\
||(\mathbf{I} -  \gamma \mathbf{T})^{-1}\mathbf{T}||_\infty & \leq \frac{1}{(1 - \gamma)}
\end{align*}
\end{proof}

\thtermination*

\begin{proof}
We can write the Bellman equation in the matrix form: ${\bf v} = {\bf r} + \gamma \mathbf{P_\pi} \bf{v}$, for a fixed policy $\pi$, where $\bf{v}$ is a \emph{finite} column vector with one entry per state encoding its value function. From Equation \ref{eq:eigenpurpose} we have ${\bf r} = \mathbf{P_\pi}{\bf w - w}$ with ${\bf w} = \Phi {\bf e}$, where ${\bf e}$ denotes the eigenvector of interest and $\Phi \in \mathbb{R}^{|\mathscr{S}| \times d}$ denotes the matrix representing the $d$-dimensional feature representation for each state. We use $r: \mathscr{S} \rightarrow \mathbb{R}$ for simplicity. This function can be seen as the expected reward in a given state. With that we have:
\begin{align*}
{\bf v}                                      &=      \mathbf{P_\pi}{\bf w - w} + \gamma \mathbf{P_\pi} {\bf v}\\
{\bf v + w}                                  &=      \mathbf{P_\pi}{\bf w} + \gamma \mathbf{P_\pi} {\bf v}\\
                                             &=      \mathbf{P_\pi}{\bf w} + \gamma \mathbf{P_\pi} {\bf v} + \gamma \mathbf{P_\pi} {\bf w} - \gamma \mathbf{P_\pi} {\bf w}\\
                                             &=      (1- \gamma) \mathbf{P_\pi} {\bf w} + \gamma \mathbf{P_\pi} ({\bf v + w})\\
{\bf v + w} - \gamma \mathbf{P_\pi} ({\bf v + w})     &=      (1 - \gamma) \mathbf{P_\pi} {\bf w}\\
(\mathbf{I} -  \gamma \mathbf{P_\pi}) ({\bf v + w})             &=      (1 - \gamma) \mathbf{P_\pi} {\bf w}\\
 \bf{v} + \bf{w}                             &=      (1 - \gamma) (\mathbf{I} -  \gamma \mathbf{P_\pi})^{-1} \mathbf{P_\pi} {\bf w} \
\end{align*}
where the last step is true because $(\mathbf{I} -  \gamma \mathbf{P_\pi})^{-1}$ is guaranteed to be nonsingular since $||\mathbf{P_\pi}|| \leq 1$, where $||\mathbf{P_\pi}|| = \sup_{\mathbf{v}:||\mathbf{v}||_\infty = 1} ||\mathbf{P_\pi} {\bf v}||_\infty$. By the Neumann series we have $(\mathbf{I} -  \gamma \mathbf{P_\pi})^{-1} = \sum_{n=0}^\infty \gamma^n\mathbf{P_\pi}^n$. Using the induced norm we have:
\begin{align*}
 ||{\bf v + w}||_\infty                      &=      (1 - \gamma)||(\mathbf{I} -  \gamma \mathbf{P_\pi})^{-1} \mathbf{P_\pi} {\bf w}||_\infty\\
 ||{\bf v + w}||_\infty                      &\le    (1 - \gamma)||(\mathbf{I} -  \gamma \mathbf{P_\pi})^{-1} \mathbf{P_\pi}||_\infty ||{\bf w}||_\infty && \text{because $||A{\bf x}||\leq||A||\cdot||{\bf x}||$}\\
 ||{\bf v + w}||_\infty                      &\le    (1 - \gamma) \frac{1}{(1-\gamma)} ||{\bf w}||_\infty && \text{Lemma}~\ref{lemma:pvf2}\\
 ||{\bf v + w}||_\infty                      &\le    ||{\bf w}||_\infty
\end{align*}
We can shift ${\bf w}$ by any finite constant without changing the reward, that is, $\mathbf{P_\pi}{\bf w - w} = \mathbf{P_\pi}({\bf w} + \bm{\delta}) - ({\bf w} + \bm{\delta})$ because $\mathbf{P_\pi}{\bf 1}\bm{\delta} = {\bf 1}\bm{\delta}$ since $\sum_j P_{\pi_{i,j}} = 1$. Therefore, we can assume ${\bf w} \ge {\bf 0}$. Let $s^* = \argmax_s {\bf w}_{s^*}$, so that ${\bf w}_{s^*} = ||{\bf w}||_\infty$. Clearly ${\bf v}_{s^*} \le {\bf 0}$, otherwise $||{\bf v + w}||_\infty \ge |{\bf v}_{s^*} + {\bf w}_{s^*}| = {\bf v}_{s^*} + {\bf w}_{s^*} > {\bf w}_{s^*} = ||{\bf w}||_\infty$, arriving at a contradiction.
\end{proof}

\thequivalence*
\begin{proof}
Let $\lambda_i$, ${\bf e}_i$ denote the $i$-th eigenvalue and eigenvector of the SR, respectively. Using the fact that the SR converges to $(\mathbf{I} -  \gamma \mathbf{P_\pi})^{-1}$ (through the Neumann series), we have:

\begin{eqnarray}
(\mathbf{I} -  \gamma \mathbf{P_\pi})^{-1} {\bf e}_i                        &=&      \lambda_i {\bf e}_i \nonumber \\
%(\mathbf{I} -  \gamma \mathbf{P_\pi}) (\mathbf{I} -  \gamma \mathbf{P_\pi})^{-1} {\bf e}_i    &=&      \lambda_i (\mathbf{I} -  \gamma \mathbf{P_\pi}) {\bf e}_i \nonumber \\
{\bf e}_i                                                &=&      \lambda_i (\mathbf{I} -  \gamma \mathbf{P_\pi}) {\bf e}_i \nonumber \\
(\mathbf{I} -  \gamma \mathbf{P_\pi}) {\bf e}_i                             &=&      \lambda_i^{-1} {\bf e}_i \nonumber \\
(\mathbf{I} -  \gamma \mathbf{P_\pi}) \gamma^{-1} {\bf e}_i                 &=&      \lambda_i^{-1} \gamma^{-1} {\bf e}_i \nonumber \\
\gamma^{-1}{\bf e}_i -  \mathbf{P_\pi} {\bf e}_i                   &=&      \lambda_i^{-1} \gamma^{-1} {\bf e}_i \nonumber \\
\mathbf{P_\pi} {\bf e}_i                                          &=&      \gamma^{-1} {\bf e}_i -  \lambda_i^{-1} \gamma^{-1} {\bf e}_i \nonumber \\
                                                         &=&      (1 - \lambda_i^{-1}) \gamma^{-1} {\bf e}_i \nonumber \\
\mathbf{I}{\bf e}_i -  \mathbf{P_\pi} {\bf e}_i                             &=&      \mathbf{I}{\bf e}_i -  (1 - \lambda_i^{-1}) \gamma^{-1} {\bf e}_i \nonumber \\
(\mathbf{I} -  \mathbf{P_\pi}) {\bf e}_i                                    &=&      [\gamma - (1 - \lambda_i^{-1})] \gamma^{-1} {\bf e}_i \nonumber \\
(\mathbf{I} -  \mathbf{P_\pi}) \gamma^{-1} {\bf e}_i                        &=&      [1 - (1 - \lambda_i^{-1}) \gamma^{-1}] \gamma^{-1} {\bf e}_i \nonumber \\
(\mathbf{I} -  \mathbf{P_\pi}) \gamma^{-1} {\bf e}_i                        &=&      \lambda_j' \gamma^{-1} {\bf e}_i \label{eq:change_variables_sr}
\end{eqnarray}
\begin{eqnarray}
(\mathbf{I} -  \mathbf{D}^{-1}\mathbf{W}) \gamma^{-1} {\bf e}_i                      &=&      \lambda_j' \gamma^{-1} {\bf e}_i \nonumber \\
(\mathbf{D}^{-1}(\mathbf{D} - \mathbf{W})) \gamma^{-1} {\bf e}_i                    &=&      \lambda_j' \gamma^{-1} {\bf e}_i \nonumber \\
\mathbf{D}^{1/2}(\mathbf{D}^{-1}(\mathbf{D} - \mathbf{W})) \gamma^{-1} {\bf e}_i             &=&      \lambda_j' \gamma^{-1} \mathbf{D}^{1/2} {\bf e}_i \nonumber\\
\mathbf{D}^{-1/2}(\mathbf{D} - \mathbf{W}) \gamma^{-1} {\bf e}_i                    &=&      \lambda_j' \gamma^{-1} \mathbf{D}^{1/2} {\bf e}_i \nonumber\\
\mathbf{D}^{-1/2}(\mathbf{D} - \mathbf{W})\mathbf{D}^{-1/2} \mathbf{D}^{1/2} \gamma^{-1} {\bf e}_i    &=&      \lambda_j' \gamma^{-1} \mathbf{D}^{1/2} {\bf e}_i \nonumber\\
\mathbf{L} \mathbf{D}^{1/2} \gamma^{-1} {\bf e}_i                          &=&      \lambda_j' \gamma^{-1} \mathbf{D}^{1/2} {\bf e}_i \nonumber
\end{eqnarray}
\vspace{-1cm}
\end{proof}

\begin{restatable}[\citeauthor{Machado17} \citeyear{Machado17}]{lem}{lemma_pvf_lfa}
\label{lemma:pvf_lfa}
In the tabular case, if all transitions in the MDP have been sampled once, $\mathbf{T}^\top \mathbf{T} = 2 \mathbf{L}$.
\end{restatable}
\begin{proof}
Let $t_{ij}$ and $tt_{ij}$ denote the entries in the $i$-th row and $j$-th column of matrices $\mathbf{T}$ and $\mathbf{T}^\top \mathbf{T}$. We can write $tt_{ij}$ as:
\begin{equation}
tt_{ij} = \sum_k t_{ik} \times t_{jk}.
\end{equation}
In the tabular case, $t_{ij}$ has three possible values:
\begin{itemize}
    \item $t_{ij} = +1$, meaning that the agent arrived in state $j$ at time step $i$, 
    \item $t_{ij} = -1$, meaning that the agent left state $j$ at time step $i$,
    \item $t_{ij} = 0$, meaning that the agent did not arrive nor leave state $j$ at time step~$i$.
\end{itemize}
We decompose $\mathbf{T}^\top \mathbf{T}$ in two matrices, $\mathbf{K}$ and $\mathbf{Z}$, such that $\mathbf{T}^\top \mathbf{T} = \mathbf{K} + \mathbf{Z}$. Here $\mathbf{Z}$ is a diagonal matrix such that $z_{ii} = tt_{ii}$, for all $i$; and $\mathbf{K}$ contains all elements from $\mathbf{T}^\top \mathbf{T}$ that lie outside the main diagonal.

When computing the elements of $\mathbf{Z}$ we have $i = j$. Thus $z_{ii} = \sum_k t_{ik}^2$. Because we square all elements, we are in fact summing over all transitions leaving ($-1^2$) \underline{and} arriving ($1^2$) in state $i$, counting the node's degree twice. Thus, $\mathbf{Z} = 2\mathbf{D}s$.

When not computing the elements in the main diagonal, for the element $tt_{ij}$, we add all transitions that leave state $i$ arriving in state $j$ ($-1 \times 1$), \underline{and} those that leave state $j$ arriving in state $i$ ($1 \times -1$). We assume each transition has been sampled once, thus:
$$tt_{ij} = \left\{\begin{array}{rl}
-2, &\mbox{if the transition between states $i$ and $j$ exists},\\
0, &\mbox{otherwise}.
\end{array}\right.
$$
Therefore, we have $\mathbf{K} = -2 \mathbf{W}$ and $\mathbf{T}^\top \mathbf{T} = \mathbf{K} + \mathbf{Z} = 2 (\mathbf{D} - \mathbf{W})$.
\end{proof}

\begin{restatable}[\citeauthor{Machado17} \citeyear{Machado17}]{theorem}{thlinearfa}
Consider the singular value decomposition of the matrix $\mathbf{T}$ s.t. $\mathbf{T} = \mathbf{U \Sigma V}$, with each row of $\mathbf{T}$ consisting of the difference between observations, i.e., $\bm{\phi}(s') - \bm{\phi}(s)$. In the tabular case, if all transitions in the MDP have been sampled once, the orthonormal vectors of $\mathbf{L} = \mathbf{D} - \mathbf{W}$ are the columns of $\mathbf{V}^\top$.
\label{thlinearfa}
\end{restatable}
\begin{proof}
Given the SVD decomposition of a matrix $\mathbf{A} = \mathbf{U \Sigma V}^\top$, the columns of $\mathbf{V}$ are the eigenvectors of $\mathbf{A}^\top \mathbf{A}$~\citep{Strang05}. We know that $\mathbf{T}^\top \mathbf{T} = 2 \mathbf{L}$, where $\mathbf{L}~=~\mathbf{D} - \mathbf{W}$ (Lemma~\ref{lemma:pvf_lfa}). Thus, the columns of $\mathbf{V}$ are the eigenvectors of $\mathbf{T}^\top \mathbf{T}$, which can be rewritten as $2 (\mathbf{D} - \mathbf{W})$. Therefore, the columns of $\mathbf{V}$ are also the eigenvectors of $\mathbf{L}$.
\end{proof}

\section{Pseudo-code for the Discussed Algorithms}~\label{sec:pseudocode}

Algorithm~\ref{alg:eigenoptions_closed_form} and~\ref{alg:eigenoptions_online} summarize eigenoption discovery. They present in an algorithm box the algorithm outlined in Section~\ref{sec:temporally_extended_exploration}. Of note, is the assumption that the function \emph{eigendecomposition} returns eigenvectors sorted by their eigenvalues---we return the eigenvalues from this function only for clarity, but they are not used anywhere. The matrix $E$ contains all eigenvectors. In the closed-form algorithm, we assume access to the transition matrix $P$. In the online case, we use $\mathcal{U}$ to denote a uniform distribution and we use $||$ to represent the operation of appending a transition to the data set.

To be able to write the covering options algorithm in closed-form more succinctly, we assumed access to a \emph{getInducedTransitionMatrix} function that receives as input the matrix $P$ that contains the underlying transition dynamics, as well as the option set $\Omega$, and returns a new transition matrix induced by such options. For simplicity, we also define a \emph{getTopEigenvector} function, and we abstract away the fact that we use both directions of each eigenvector (we do the same for eigenoptions). In the online case, notice we store temporally extended transitions when an option is sampled, as if the agent had teleported from one state to another.

\begin{algorithm}
\caption{Eigenoptions Closed-Form}\label{alg:eigenoptions_closed_form}
\KwInput{$\gamma_{SR}, \ \gamma_{o}$ \Comment*[r]{Discount factor for the SR and the options' policies}} 
\hspace{1.2cm} $P \in \mathbb{R}^{|\mathscr{S}| \times |\mathscr{S}|}$ \Comment*[r]{Transition matrix}
\KwOutput{$\Omega$ \Comment*[r]{Set of eigenoptions \newline}}

$\triangleright$ Learn representation\\
$\Psi \gets (I - \gamma_{SR} P)^{-1}$\\
$\mathbf{\lambda}, E \gets \textrm{eigendecomposition}(\Psi)$\\
% $E_+ \gets \textrm{getBothDirections}(E)$\\
$\Omega \gets \emptyset$\\

\For{$\textrm{\normalfont \textbf{e}}$ \textrm{\normalfont \textbf{in}} $E$}{
$\triangleright$ Derive intrinsic reward function from learned representation\\
$r^{\mathbf{e}}(s, s') \gets \textrm{\normalfont \textbf{e}}(s') - \textrm{\normalfont \textbf{e}}(s) \ \ \forall s, s' \in \mathscr{S}$ \Comment*[r]{Define eigenpurpose}
$\triangleright$ Learn to maximize intrinsic reward\\
 $Q \gets \textrm{PolicyIteration}(r^{\mathbf{e}}, P, \gamma_o)$ \Comment*[r]{Learn value function}
 $\triangleright$ Define option\\
 $\mathcal{I} \gets \emptyset$; $\pi(s) \gets \bot$, $\beta(s) \gets 1 \ \ \forall s \in \mathscr{S}$ \Comment*[r]{Initialize option tuple}
 \For{$s$ \textrm{\normalfont \textbf{in}} $\mathscr{S}$}{
    $\triangleright$ \ If $s$ is not a terminal state for the eigenoption being learned:\\
    \uIf{$\exists a \in \mathscr{A}(s) \ s.t. \ Q(s, a) > 0 $}{
    $\mathcal{I} \gets \mathcal{I} \cup \{s\}$\\
    $\pi(s) \gets \argmax_a Q(s, a)$\\
    $\beta(s) \gets 0$
  }
 }
 $\Omega \gets \Omega \cup \langle \mathcal{I}, \pi, \beta \rangle$
}
\end{algorithm}

\begin{algorithm}
\caption{Eigenoptions Online}\label{alg:eigenoptions_online}
\KwInput{$\eta, \ \alpha_{o}$ \Comment*[r]{Step-sizes for learning the SR and the options' policies}}
\hspace{1.2cm} $\gamma_{SR}, \ \gamma_{o}$ \Comment*[r]{Discount factor for the SR and the options' policies}
\hspace{1.2cm} $N_{steps}$ \Comment*[r]{Maximum number of interactions with the environment}
\KwOutput{$\Omega$ \Comment*[r]{Set of eigenoptions \newline}}

%interact with the environment
$\mathcal{D} \gets \emptyset$\\
$\triangleright$ Collect samples\\
\For{$i \gets 0$ \KwTo $N_{steps}$}{
    $a \gets \mathcal{U}\Big(\mathscr{A}(s)\Big)s$ \Comment*[r]{Choose action randomly}
    In state $s$, take action $a$ and observe state $s'$\\
    $\mathcal{D} \gets \mathcal{D} \ \| \ (s, a, s')$ \Comment*[r]{Append transition to data set $\mathcal{D}$}
}

$\triangleright$ Learn representation\\
$\Psi \gets \textrm{Successor Representation}(\eta, \gamma_{SR}, \mathcal{D})$\\
$\mathbf{\lambda}, E \gets \textrm{eigendecomposition}(\Psi)$\\
% $E_+ \gets \textrm{getBothDirections}(E)$\\
$\Omega \gets \emptyset$\\

\For{$\textrm{\normalfont \textbf{e}}$ \textrm{\normalfont \textbf{in}} $E$}{
$\triangleright$ Derive intrinsic reward function from learned representation\\
$r^{\mathbf{e}}(s, s') \gets \textrm{\normalfont \textbf{e}}(s') - \textrm{\normalfont \textbf{e}}(s) \ \ \forall s, s' \in \mathscr{S}$ \Comment*[r]{Define eigenpurpose}
 $\triangleright$ Learn to maximize intrinsic reward\\
 $Q \gets \textrm{Q-Learning}(r^{\mathbf{e}}, \alpha_o, \gamma_o)$ \Comment*[r]{Learn value function}
 $\triangleright$ Define option\\
 $\mathcal{I} \gets \emptyset$; $\pi(s) \gets \bot$, $\beta(s) \gets 1 \ \ \forall s \in \mathscr{S}$ \Comment*[r]{Initialize option tuple}
 \For{$s$ \textrm{\normalfont \textbf{in}} $\mathscr{S}$}{
    $\triangleright$ \ If $s$ is not a terminal state for the eigenoption being learned:\\
    \uIf{$\exists a \in \mathscr{A}(s) \ s.t. \ Q(s, a) > 0 $}{
    $\mathcal{I} \gets \mathcal{I} \cup \{s\}$\\
    $\pi(s) \gets \argmax_a Q(s, a)$\\
    $\beta(s) \gets 0$
  }
 }
 $\Omega \gets \Omega \cup \langle \mathcal{I}, \pi, \beta \rangle$
}
\end{algorithm}

\begin{algorithm}
\caption{Covering Options Closed-Form}\label{alg:co_closed_form}
\KwInput{$\gamma_{SR}, \ \gamma_{o}$ \Comment*[r]{Discount factor for the SR and the options' policies}} 
\hspace{1.2cm} $P \in \mathbb{R}^{|\mathscr{S}| \times |\mathscr{S}|}$ \Comment*[r]{Transition matrix}
\hspace{1.2cm} $N_{iter}$ \Comment*[r]{Number of iterations}
\KwOutput{$\Omega$ \Comment*[r]{Set of covering options \newline}}

$\Omega \gets \emptyset$\\
\For{$i$ \textrm{\normalfont \textbf{in}} $N_{iter}$}{
$\triangleright$ Learn representation\\
$T \gets \textrm{getInducedTransitionMatrix}(P, \Omega)$\\
$\Psi \gets (I - \gamma_{SR} T)^{-1}$\\
$\triangleright$ Derive intrinsic reward function from learned representation\\
$\mathbf{e} \gets \textrm{getTopEigenvector}(\Psi)$\\

$r^{\mathbf{e}}(s, s') \gets 0 \ \ \forall s, s' \in \mathscr{S}$ \Comment*[r]{Define reward function}
$r^{\mathbf{e}}(s, \argmax \mathbf{e}) = 1 \ \ \forall s \in \mathscr{S}$\\
 
 $\triangleright$ Learn to maximize intrinsic reward\\
 $Q \gets \textrm{PolicyIteration}(r^{\mathbf{e}}, P, \gamma_o)$ \Comment*[r]{Learn value function}
 $\triangleright$ Define option\\
 $\mathcal{I} \gets \emptyset$; $\pi(s) \gets \bot$, $\beta(s) \gets 1 \ \ \forall s \in \mathscr{S}$ \Comment*[r]{Initialize option tuple}
 $\mathcal{I} \gets \argmin \mathbf{e}$\\
 $\pi(s) \gets \argmax_a Q(s, a)$\\
 $\beta(\argmax \mathbf{e}) \gets 1$\\
 $\Omega \gets \Omega \cup \langle \mathcal{I}, \pi, \beta \rangle$
}
\end{algorithm}

\begin{algorithm}
\caption{Covering Options Online}\label{alg:co_online}
\KwInput{$\eta, \ \alpha_{o}$ \Comment*[r]{Step-sizes for learning the SR and the options' policies}}
\hspace{1.2cm} $\gamma_{SR}, \ \gamma_{o}$ \Comment*[r]{Discount factor for the SR and the options' policies}
\hspace{1.2cm} $N_{steps}$ \Comment*[r]{Maximum number of interactions with the environment}
\hspace{1.2cm} $N_{iter}$ \Comment*[r]{Number of option discovery iterations}
\KwOutput{$\Omega$ \Comment*[r]{Set of covering options \newline}}

$\Omega \gets \emptyset$\\
\For{$i$ \textrm{\normalfont \textbf{in}} $N_{iter}$}{

%interact with the environment
$\mathcal{D} \gets \emptyset$\\
$\triangleright$ Collect samples\\
\For{$i \gets 0$ \KwTo $N_{steps}$}{
    $\omega \gets \mathcal{U}\big(\mathscr{A}(s) \cup \Omega\big)$ \Comment*[r]{Choose action (or option) randomly}
    
    \uIf{$\omega \in \mathscr{A}(s)$}{
    In state $s$, take action $a$ and observe state $s'$\\
    }
    \Else{
    Act according to sampled option $\omega$ from $s$ until termination in $s'$ 
    }
    $\mathcal{D} \gets \mathcal{D} \ \| \ (s, \omega, s')$ \Comment*[r]{Append transition to data set $\mathcal{D}$}
}
$\triangleright$ Learn representation\\
$\Psi \gets \textrm{Successor Representation}(\eta, \gamma_{SR}, \mathcal{D})$\\
$\triangleright$ Derive intrinsic reward function from learned representation\\
$\mathbf{e} \gets \textrm{getTopEigenvector}(\Psi)$\\

$r^{\mathbf{e}}(s, s') \gets 0 \ \ \forall s, s' \in \mathscr{S}$ \Comment*[r]{Define reward function}
$r^{\mathbf{e}}(s, \argmax \mathbf{e}) = 1 \ \ \forall s \in \mathscr{S}$\\

$\triangleright$ Learn to maximize intrinsic reward\\
 $Q \gets \textrm{Q-Learning}(r^{\mathbf{e}}, \alpha_o, \gamma_o)$ \Comment*[r]{Learn value function}
$\triangleright$ Define option\\
 $\mathcal{I} \gets \emptyset$; $\pi(s) \gets \bot$, $\beta(s) \gets 1 \ \ \forall s \in \mathscr{S}$ \Comment*[r]{Initialize option tuple}
 $\mathcal{I} \gets \argmin \mathbf{e}$\\
 $\pi(s) \gets \argmax_a Q(s, a)$\\
 $\beta(\argmax \mathbf{e}) \gets 1$\\
 $\Omega \gets \Omega \cup \langle \mathcal{I}, \pi, \beta \rangle$
}
\end{algorithm}

\clearpage

\section{Additional Results for Return Maximization with Eigenoptions and Covering Options}\label{sec:appendix_accum_return}

\vspace{2.56cm}

\begin{figure}[h]
     \centering
         \includegraphics[width=0.9\textwidth]{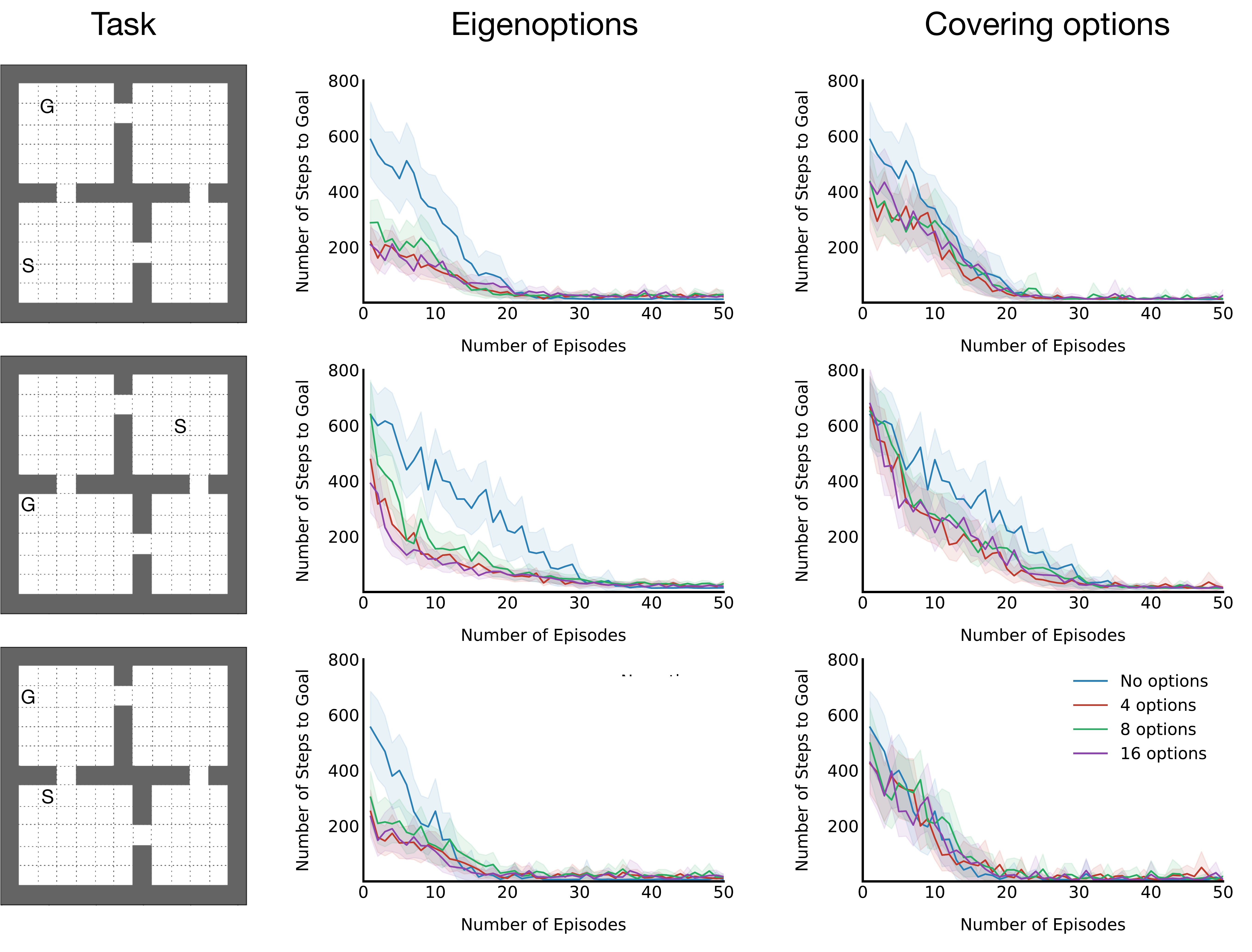}
          \caption{Performance of Q-learning augmented with eigenoptions and covering options in additional tasks as described in Section~\ref{sec:reward_max}. In the task, \textsf{S} and \textsf{G} denote the start and goal state. Results are averaged across 50 runs and shaded regions denote a 99\% confidence interval. See text for details.}
          
\end{figure}

\begin{figure}[h]
     \centering
         \includegraphics[width=0.9\textwidth]{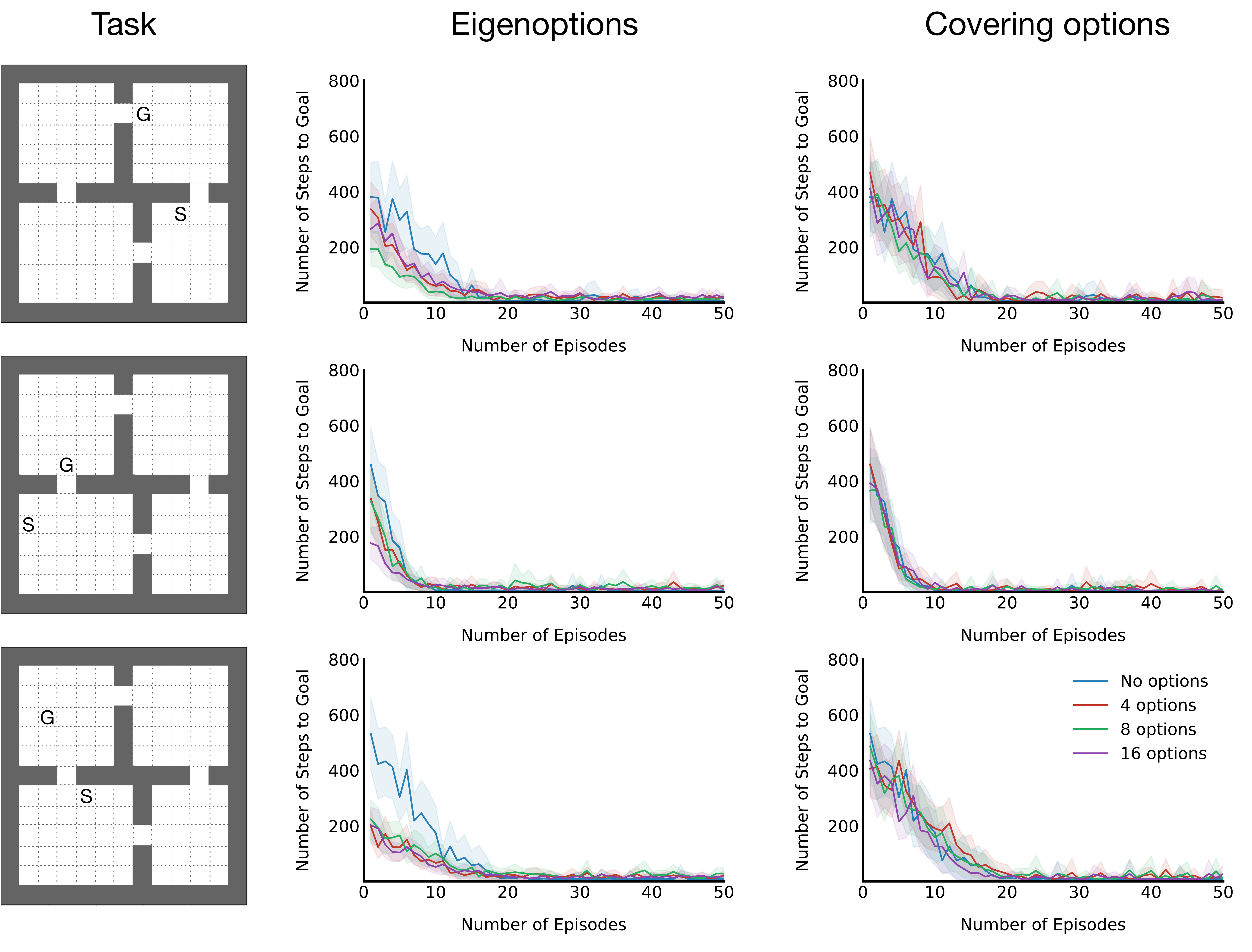}
          \caption{Performance of Q-learning augmented with eigenoptions and covering options in additional tasks as described in Section~\ref{sec:reward_max}. In the task, \textsf{S} and \textsf{G} denote the start and goal state. Results are averaged across 50 runs and shaded regions denote a 99\% confidence interval. See text for details.}
\end{figure}

\begin{figure}[h]
     \centering
         \includegraphics[width=0.9\textwidth]{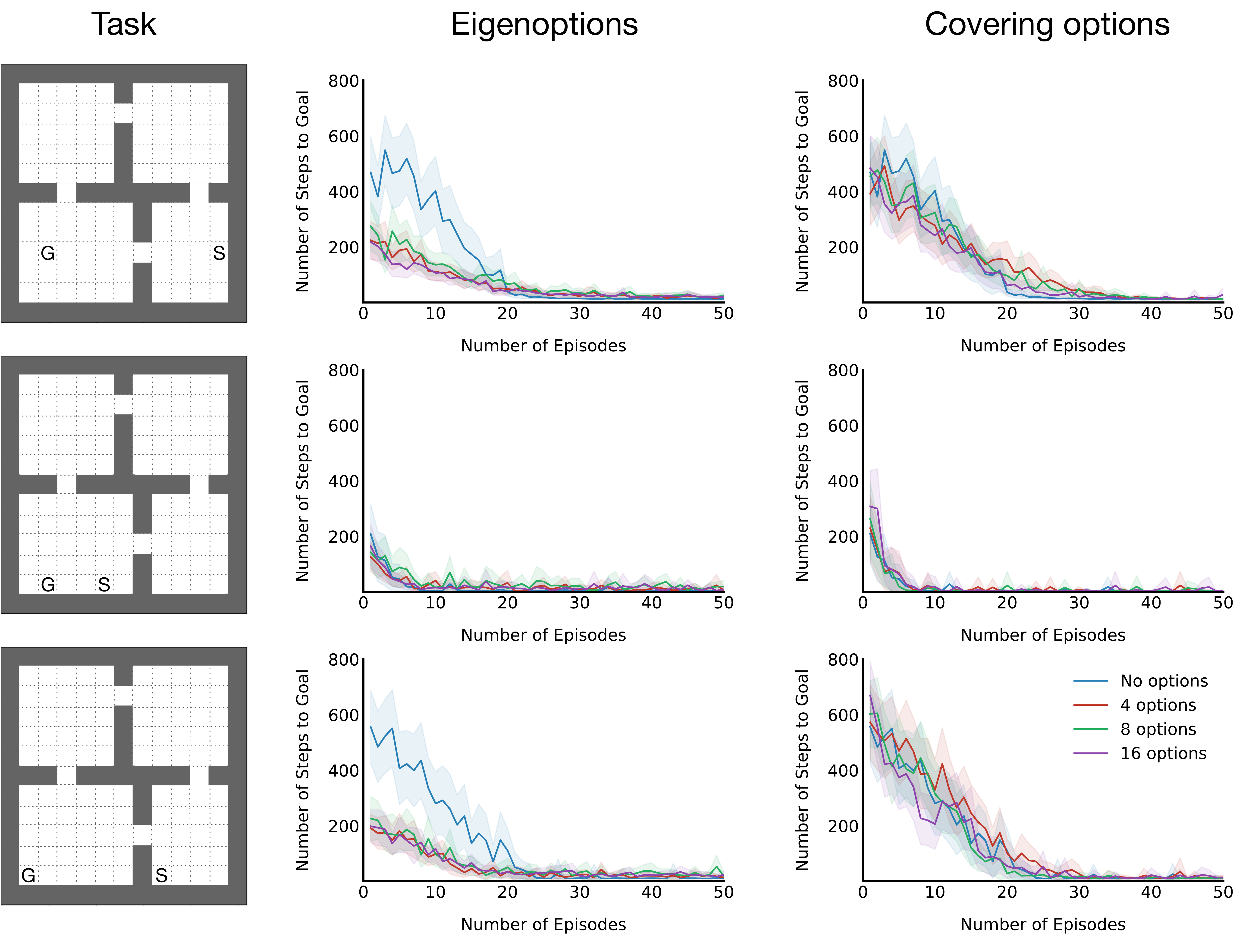}
          \caption{Performance of Q-learning augmented with eigenoptions and covering options in additional tasks as described in Section~\ref{sec:reward_max}. In the task, \textsf{S} and \textsf{G} denote the start and goal state. Results are averaged across 50 runs and shaded regions denote a 99\% confidence interval. See text for details.}
\end{figure}

\clearpage

\section{Eigenoptions and Covering Options Learned Online}\label{sec:appendix_evolution}

\vspace{3.4cm}

\begin{figure}[h]
     \centering
         \includegraphics[width=0.9\textwidth]{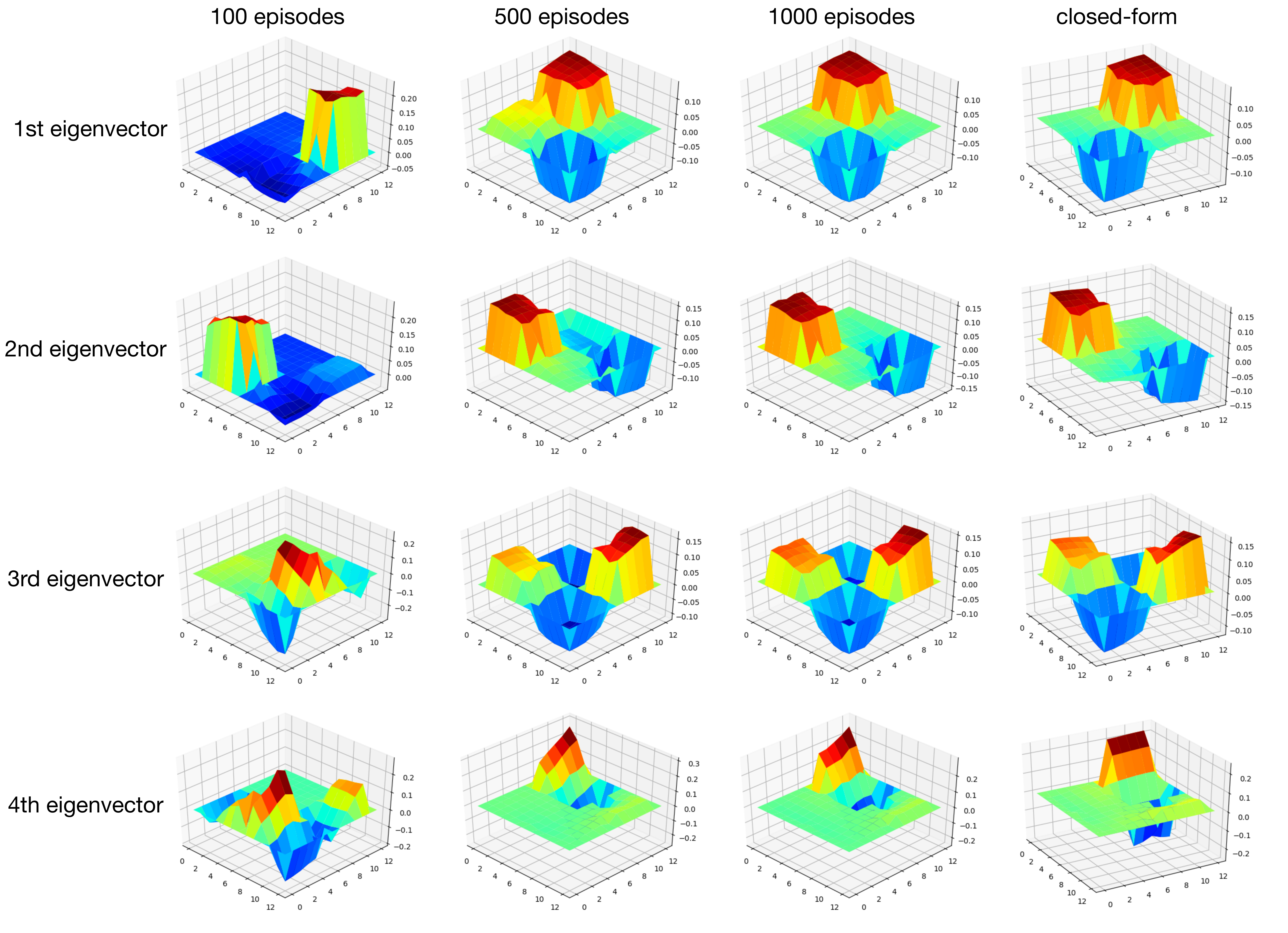}
          \caption{Top four eigenvectors of the SR when learning eigenoptions online for a varied number of episodes. We randomly selected one of the fifty runs to plot. The agent interacted with the environment for 1,000 steps in each episode.} \label{fig:evolution_eigenvecs_eigenoptions}
\end{figure}

\begin{figure}[h]
     \centering
         \includegraphics[width=0.9\textwidth]{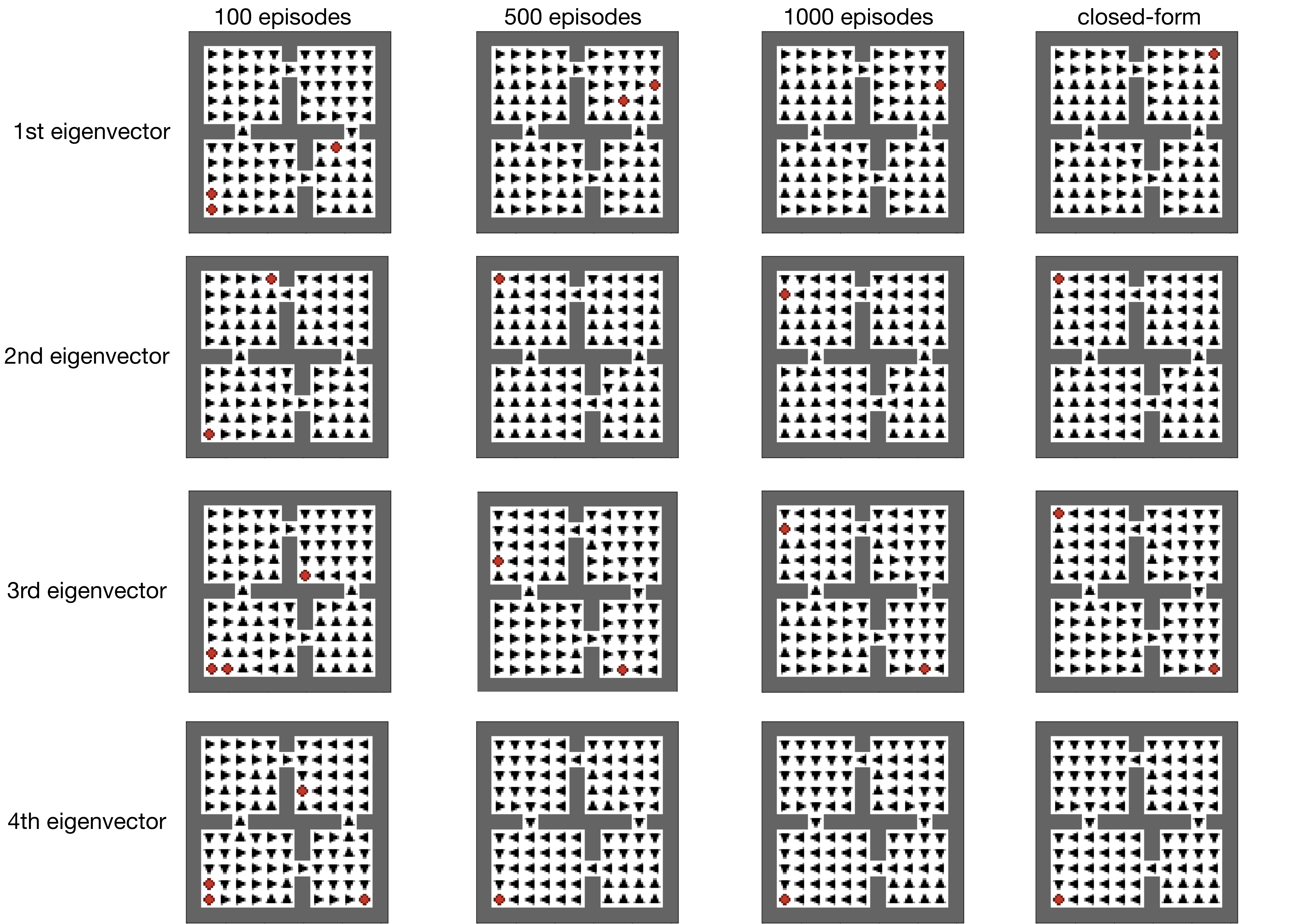}
          \caption{First four eigenoptions discovered online for a varied number of episodes, computed from the eigenvectors in Figure~\ref{fig:evolution_eigenvecs_eigenoptions}. We plot only one eigenoption per eigenvector (instead of two). Policies were learned off-policy from the data collected when learning~the~SR.}
          \label{fig:evolution_policies_eigenoptions}
\end{figure}

\begin{figure}[p]
     \centering
         \includegraphics[width=0.9\textwidth]{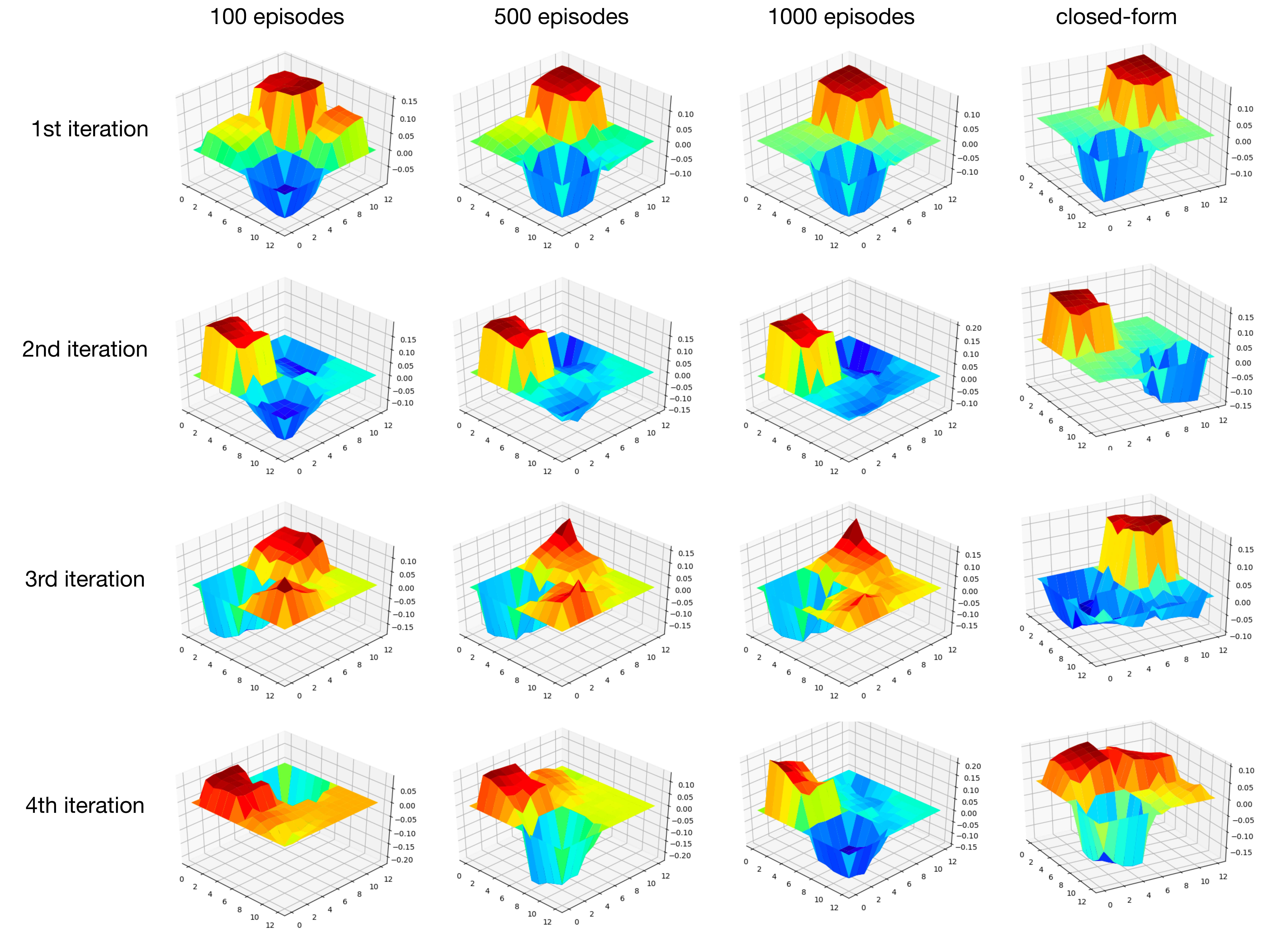}
          \caption{Top four eigenvectors of the SR when learning covering options online for a varied number of episodes. We randomly selected one of the fifty runs to plot. The agent interacted with the environment for 1,000 steps in each episode.} \label{fig:evolution_eigenvecs_co}
\end{figure}

\begin{figure}[p]
     \centering
         \includegraphics[width=0.9\textwidth]{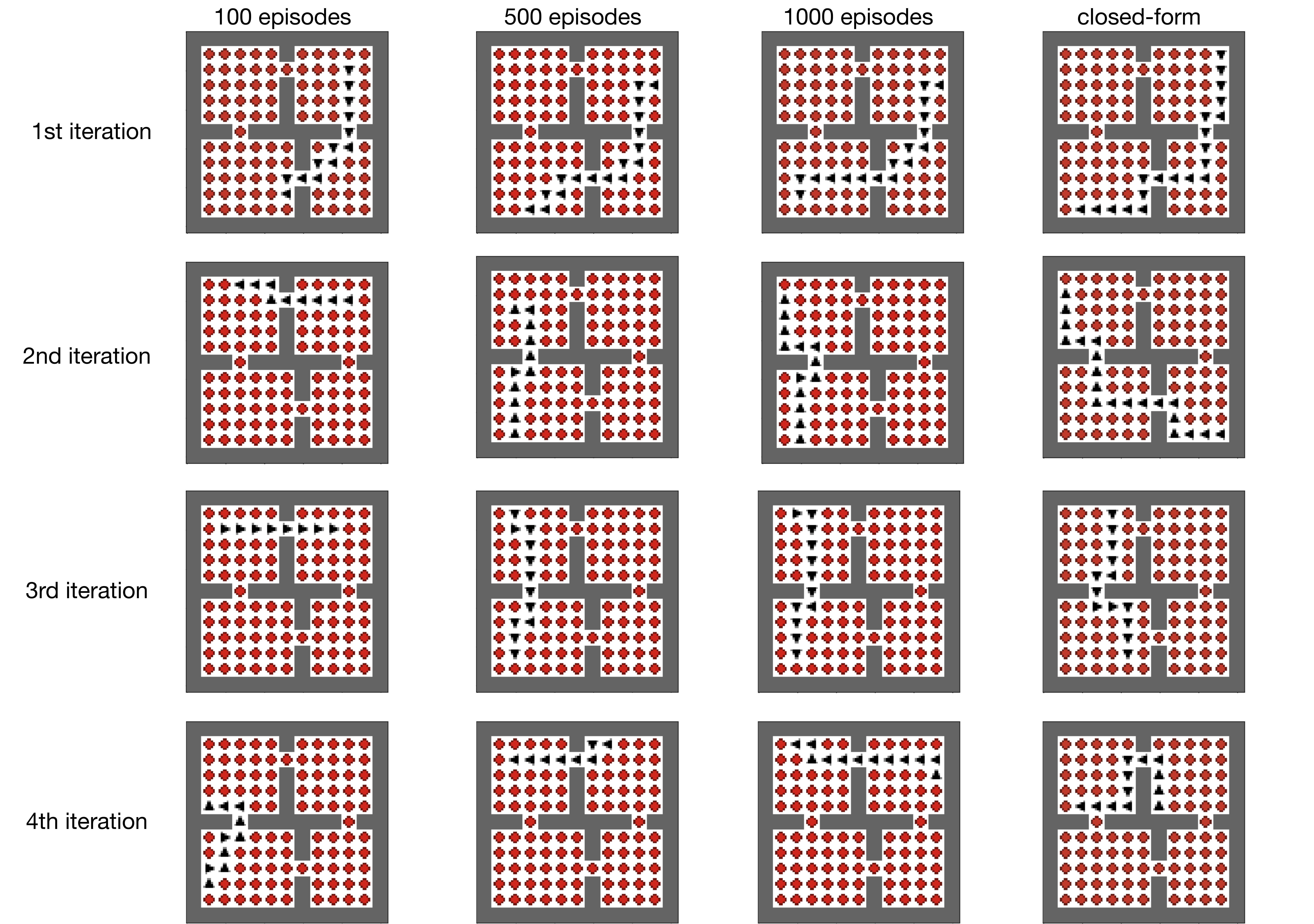}
          \caption{First four covering options discovered online for a varied number of episodes, computed from the eigenvectors in Figure~\ref{fig:evolution_eigenvecs_co}. We plot only one covering option per eigenvector (instead of two). Policies were learned off-policy from the data collected when learning the SR. Notice how the covering options discovered online in the second and third iteration overlap, and how none of the first four options capture the NW-SE diagonal when discovered online.}
          \label{fig:evolution_policies_co}
\end{figure}

\iffalse
\clearpage

\newlength\mylen
\newcommand\IndKwIn[1]{%
  \settowidth\mylen{\KwIn{}}%
  \setlength\hangindent{\mylen}%
  \hspace*{\mylen}#1\\}

\begin{algorithm}[H]
\SetAlgoLined
\KwIn{matrix $E \in \mathbb{R}^{k \times n}$ of $k$ eigenvectors of size $n$;}
\IndKwIn{state to be visited, $s$;}
\IndKwIn{slack variable, $\xi > 0$;}

\KwOut{${\bm w} \in \mathbb{R}^k$ if possible, $\bm{0}$ otherwise.} \vspace{0.4cm}

$X \leftarrow \emptyset$

\For{$i < k$ \& $i \neq s$}{$
X \leftharpoondown (E_i -  E_s) \odot \bm{w}$ \ \ \ \ \ \ \ \ \ \ // $\leftharpoondown$ denotes stacking} \vspace{0.4cm} 

solve $\min \bm{w}_s$

\ \ \ \ \ \ \ \ s.t. $X \leq 0$

\ \ \ \ \ \ \ \ \ \ \ \ \ \ $\xi \leq \bm{w}_s < \infty $

\If{$\bm{w} == 0$}{
solve $\min \bm{w}_s$

\ \ \ \ \ \ \ \ s.t. $X \leq 0$

\ \ \ \ \ \ \ \ \ \ \ \ \ $-\infty < \bm{w}_s \leq \xi $
}

 \caption{Search for weights leading to a given state}
\end{algorithm}
\fi

\clearpage

\section{Impact of using the Eigenvectors of the SR instead of the Eigenvectors of the Laplacian when Discovering Covering Options}~\label{app:mismatch_co_sr_pvfs}

The theoretical guarantees we discussed for the equivalence between the eigenvectors of the graph Laplacian and the eigenvectors of the SR do not hold in later iterations of covering options. In the four-room domain, for example, after the first iteration of covering options, in some states the agent has access to four primitive actions and one option, while only four actions are available in the other states. This creates an asymmetry that violates one of the assumptions of Theorem~\ref{th:equivalence}. One of the consequences of this is that, in practice, when estimating the SR instead of the graph Laplacian, we sometimes observe eigenvectors with a non-zero complex part. We empirically evaluated, in closed-form, the impact of using the  SR to learn covering options, ignoring its complex component, which is what we actually used in Section~\ref{sec:online_experiments}. We used the same parameters we used when learning eigenoptions, also learning the options' policies off-policy. The results are depicted in Figure~\ref{fig:comparison_sr_laplacin_co}. Despite some variations, at least in the four-room domain, when computing covering options in closed-form, the mismatches between the eigenvectors of the Laplacian and of the SR do not lead to very different results.\footnote{Because the environment we consider is ultimately symmetric, to avoid asymmetries in the SR, after we estimate the SR matrix $\Psi$, we actually compute the eigendecomposition of the matrix $(\Psi + \Psi^\top)/2$. Alternative solutions for dealing with the asymmetry of time-based representations include using a singular value decomposition~\citep{Machado18b} or applying spectral analysis on the result of a polar decomposition~\citep{Bar20}. We consider this problem to be outside the scope of our paper.}

\begin{figure}[h!]
     \centering
         \includegraphics[width=\textwidth]{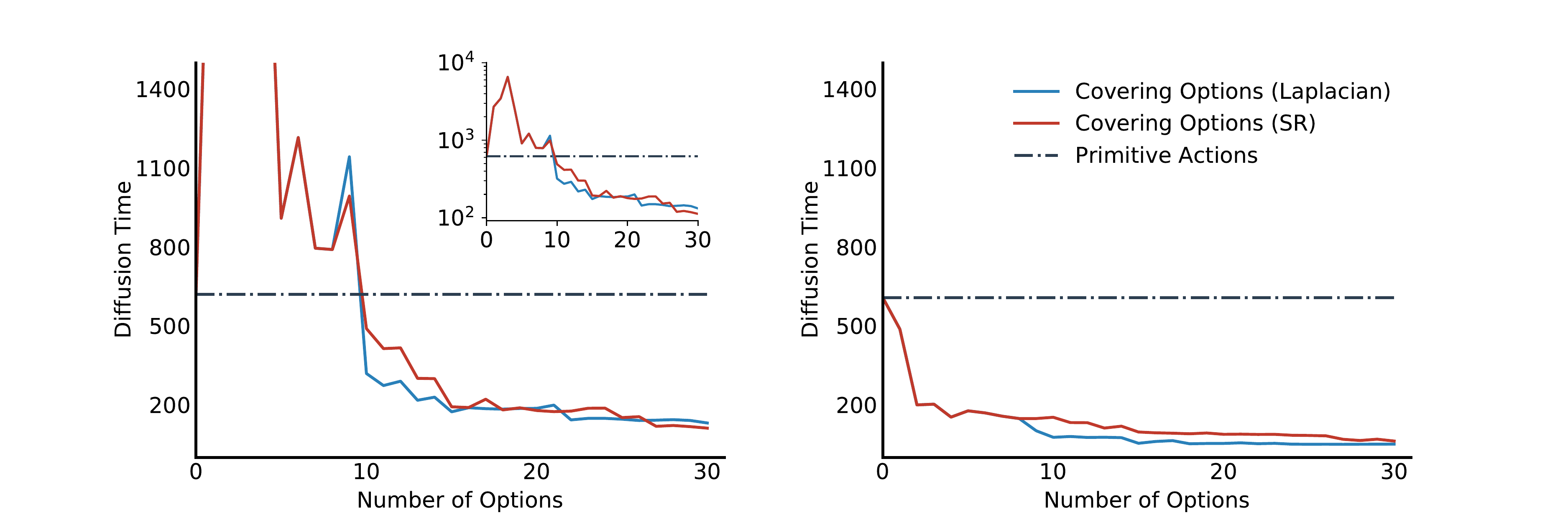}
     \caption{Average (left) and median (right) diffusion time in the four-room domain induced by covering options when computed with the SR and the eigenvectors of the graph Laplacian, in closed-form. The inset plot on the left figure depicts, in log-scale, the range the diffusion time lies.}
     \label{fig:comparison_sr_laplacin_co}
\end{figure}

\clearpage

\bibliography{refs}

\begin{thebibliography}{80}
\providecommand{\natexlab}[1]{#1}
\providecommand{\url}[1]{\texttt{#1}}
\expandafter\ifx\csname urlstyle\endcsname\relax
  \providecommand{\doi}[1]{doi: #1}\else
  \providecommand{\doi}{doi: \begingroup \urlstyle{rm}\Url}\fi

\bibitem[Bacon et~al.(2017)Bacon, Harb, and Precup]{Bacon17}
Pierre-Luc Bacon, Jean Harb, and Doina Precup.
\newblock {The option-critic architecture}.
\newblock In \emph{Conference on Artificial Intelligence (AAAI)}, 2017.

\bibitem[Bar et~al.(2020)Bar, Talmon, and Meir]{Bar20}
Amitay Bar, Ronen Talmon, and Ron Meir.
\newblock {Option discovery in the absence of rewards with manifold analysis}.
\newblock In \emph{International Conference on Machine Learning (ICML)}, 2020.

\bibitem[Baranes and Oudeyer(2013)]{Baranes13}
Adrien Baranes and Pierre{-}Yves Oudeyer.
\newblock {Active learning of inverse models with intrinsically motivated goal
  exploration in robots}.
\newblock \emph{{Robotics and Autonomous Systems}}, 61\penalty0 (1):\penalty0
  49--73, 2013.

\bibitem[Barreto et~al.(2017)Barreto, Dabney, Munos, Hunt, Schaul, Silver, and
  van Hasselt]{Barreto17}
Andr\'e Barreto, Will Dabney, R\'emi Munos, Jonathan Hunt, Tom Schaul, David
  Silver, and Hado van Hasselt.
\newblock {Successor features for transfer in reinforcement learning}.
\newblock In \emph{{Advances in Neural Information Processing Systems
  (NeurIPS)}}, 2017.

\bibitem[Barreto et~al.(2018)Barreto, Borsa, Quan, Schaul, Silver, Hessel,
  Mankowitz, Zidek, and Munos]{Barreto18}
Andr{\'{e}} Barreto, Diana Borsa, John Quan, Tom Schaul, David Silver, Matteo
  Hessel, Daniel~J. Mankowitz, Augustin Zidek, and R{\'{e}}mi Munos.
\newblock {Transfer in deep reinforcement learning using successor features and
  generalised policy improvement}.
\newblock In \emph{International Conference on Machine Learning (ICML)}, 2018.

\bibitem[Barreto et~al.(2019)Barreto, Borsa, Hou, Comanici, Ayg\"un, Hamel,
  Toyama, Hunt, Mourad, Silver, and Precup]{Barreto19}
Andr\'e Barreto, Diana Borsa, Shaobo Hou, Gheorghe Comanici, Eser Ayg\"un,
  Philippe Hamel, Daniel Toyama, Jonathan Hunt, Shibl Mourad, David Silver, and
  Doina Precup.
\newblock {The option keyboard: Combining skills in reinforcement learning}.
\newblock In \emph{Advances in Neural Information Processing Systems
  (NeurIPS)}, 2019.

\bibitem[Barreto et~al.(2020)Barreto, Hou, Borsa, Silver, and
  Precup]{Barreto20}
Andr{\'{e}} Barreto, Shaobo Hou, Diana Borsa, David Silver, and Doina Precup.
\newblock {Fast reinforcement learning with generalized policy updates}.
\newblock \emph{National Academy of Sciences}, 117\penalty0 (48):\penalty0
  30079--30087, 2020.

\bibitem[Bellemare et~al.(2020)Bellemare, Candido, Castro, Gong, Machado,
  Moitra, Ponda, and Wang]{Bellemare20}
Marc~G. Bellemare, Salvatore Candido, Pablo~Samuel Castro, Jun Gong, Marlos~C.
  Machado, Subhodeep Moitra, Sameera~S. Ponda, and Ziyu Wang.
\newblock {Autonomous navigation of stratospheric balloons using reinforcement
  learning}.
\newblock \emph{Nature}, 588:\penalty0 77--82, 2020.

\bibitem[Bellman(1957)]{Bellman57}
Richard~E. Bellman.
\newblock \emph{{Dynamic Programming}}.
\newblock Princeton University Press, 1957.

\bibitem[Borsa et~al.(2019)Borsa, Barreto, Quan, Mankowitz, van Hasselt, Munos,
  Silver, and Schaul]{Borsa19}
Diana Borsa, Andr{\'{e}} Barreto, John Quan, Daniel~J. Mankowitz, Hado van
  Hasselt, R{\'{e}}mi Munos, David Silver, and Tom Schaul.
\newblock {Universal successor features approximators}.
\newblock In \emph{International Conference on Learning Representations
  (ICLR)}, 2019.

\bibitem[Broder and Karlin(1989)]{Broder89}
Andrei~Z. Broder and Anna~R. Karlin.
\newblock {Bounds on the cover time}.
\newblock \emph{Journal of Theoretical Probability}, 2\penalty0 (1):\penalty0
  101--120, 1989.

\bibitem[Brunskill and Li(2014)]{Brunskill14}
Emma Brunskill and Lihong Li.
\newblock {PAC-inspired option discovery in lifelong reinforcement learning}.
\newblock In \emph{International Conference on Machine Learning (ICML)}, 2014.

\bibitem[Chevalier-Boisvert et~al.(2018)Chevalier-Boisvert, Willems, and
  Pal]{gym_minigrid}
Maxime Chevalier-Boisvert, Lucas Willems, and Suman Pal.
\newblock {Minimalistic gridworld environment for OpenAI gym}.
\newblock \url{https://github.com/maximecb/gym-minigrid}, 2018.

\bibitem[Coifman et~al.(2005)Coifman, Lafon, Lee, Maggioni, Nadler, Warner, and
  Zucker]{Coifman05}
R.~R. Coifman, S.~Lafon, A.~B. Lee, M.~Maggioni, B.~Nadler, F.~Warner, and
  S.~W. Zucker.
\newblock {Geometric diffusions as a tool for harmonic analysis and structure
  definition of data: Diffusion maps}.
\newblock \emph{National Academy of Sciences}, 102\penalty0 (21):\penalty0
  7426--7431, 2005.

\bibitem[Dabney et~al.(2021)Dabney, Ostrovski, and Barreto]{Dabney20}
Will Dabney, Georg Ostrovski, and Andr{\'{e}} Barreto.
\newblock {Temporally-extended {\(\epsilon\)}-greedy exploration}.
\newblock In \emph{International Conference on Learning Representations
  (ICLR)}, 2021.

\bibitem[Dayan(1993)]{Dayan93}
Peter Dayan.
\newblock {Improving generalization for temporal difference learning: The
  successor representation}.
\newblock \emph{Neural Computation}, 5\penalty0 (4):\penalty0 613--624, 1993.

\bibitem[Dayan and Hinton(1992)]{Dayan92}
Peter Dayan and Geoffrey~E. Hinton.
\newblock {Feudal reinforcement learning}.
\newblock In \emph{Advances in Neural Information Processing Systems
  (NeurIPS)}, 1992.

\bibitem[Erraqabi et~al.(2022)Erraqabi, Machado, Zhao, Sukhbaatar, Lazaric,
  Denoyer, and Bengio]{Erraqabi21}
Akram Erraqabi, Marlos~C. Machado, Mingde Zhao, Sainbayar Sukhbaatar,
  Alessandro Lazaric, Ludovic Denoyer, and Yoshua Bengio.
\newblock {Temporal abstractions-augmented temporally contrastive learning: An
  alternative to the Laplacian in RL}.
\newblock In \emph{Conference on Uncertainty in Artificial Intelligence (UAI)},
  2022.

\bibitem[Eysenbach et~al.(2019)Eysenbach, Gupta, Ibarz, and
  Levine]{Eysenbach19}
Benjamin Eysenbach, Abhishek Gupta, Julian Ibarz, and Sergey Levine.
\newblock {Diversity is all you need: Learning skills without a reward
  function}.
\newblock In \emph{International Conference on Learning Representations
  (ICLR)}, 2019.

\bibitem[Fruit and Lazaric(2017)]{Fruit17}
Ronan Fruit and Alessandro Lazaric.
\newblock {Exploration-exploitation in MDPs with options}.
\newblock In \emph{International Conference on Artificial Intelligence and
  Statistics (AISTATS)}, 2017.

\bibitem[Gregor et~al.(2017)Gregor, Rezende, and Wierstra]{Gregor16}
Karol Gregor, Danilo Rezende, and Daan Wierstra.
\newblock {Variational intrinsic control}.
\newblock In \emph{International Conference on Learning Representations (ICLR),
  Workshop track}, 2017.

\bibitem[Haarnoja et~al.(2018)Haarnoja, Hartikainen, Abbeel, and
  Levine]{haarnoja2018latent}
Tuomas Haarnoja, Kristian Hartikainen, Pieter Abbeel, and Sergey Levine.
\newblock {Latent space policies for hierarchical reinforcement learning}.
\newblock In \emph{International Conference on Machine Learning (ICML)}, 2018.

\bibitem[Hansen et~al.(2020)Hansen, Dabney, Barreto, Warde{-}Farley, de~Wiele,
  and Mnih]{Hansen20}
Steven Hansen, Will Dabney, Andr{\'{e}} Barreto, David Warde{-}Farley, Tom~Van
  de~Wiele, and Volodymyr Mnih.
\newblock {Fast task inference with variational intrinsic successor features}.
\newblock In \emph{International Conference on Learning Representations
  (ICLR)}, 2020.

\bibitem[Heess et~al.(2016)Heess, Wayne, Tassa, Lillicrap, Riedmiller, and
  Silver]{hees2016learning}
Nicolas Heess, Gregory Wayne, Yuval Tassa, Timothy~P. Lillicrap, Martin~A.
  Riedmiller, and David Silver.
\newblock {Learning and transfer of modulated locomotor controllers}.
\newblock \emph{CoRR}, abs/1610.05182, 2016.

\bibitem[Hoang et~al.(2021)Hoang, Sohn, Choi, Carvalho, and Lee]{Hoang21}
Christopher Hoang, Sungryull Sohn, Jongwook Choi, Wilka Carvalho, and Honglak
  Lee.
\newblock {Successor feature landmarks for long-horizon goal-conditioned
  reinforcement learning}.
\newblock In \emph{Advances in Neural Information Processing Systems
  (NeurIPS)}, 2021.

\bibitem[Jaderberg et~al.(2017)Jaderberg, Mnih, Czarnecki, Schaul, Leibo,
  Silver, and Kavukcuoglu]{Jaderberg17}
Max Jaderberg, Volodymyr Mnih, Wojciech~Marian Czarnecki, Tom Schaul, Joel~Z.
  Leibo, David Silver, and Koray Kavukcuoglu.
\newblock {Reinforcement learning with unsupervised auxiliary tasks}.
\newblock In \emph{{International Conference on Learning Representations
  (ICLR)}}, 2017.

\bibitem[Jinnai et~al.(2019{\natexlab{a}})Jinnai, Abel, Hershkowitz, Littman,
  and Konidaris]{Jinnai19b}
Yuu Jinnai, David Abel, David~Ellis Hershkowitz, Michael~L. Littman, and
  George~D. Konidaris.
\newblock {Finding options that minimize planning time}.
\newblock In \emph{International Conference on Machine Learning (ICML)},
  2019{\natexlab{a}}.

\bibitem[Jinnai et~al.(2019{\natexlab{b}})Jinnai, Park, Abel, and
  Konidaris]{Jinnai19}
Yuu Jinnai, Jee~Won Park, David Abel, and George~D. Konidaris.
\newblock {Discovering options for exploration by minimizing cover time}.
\newblock In \emph{International Conference on Machine Learning (ICML)},
  2019{\natexlab{b}}.

\bibitem[Jinnai et~al.(2020)Jinnai, Park, Machado, and Konidaris]{Jinnai20}
Yuu Jinnai, Jee~Won Park, Marlos~C. Machado, and George~D. Konidaris.
\newblock {Exploration in reinforcement learning with deep covering options}.
\newblock In \emph{{International Conference on Learning Representations
  (ICLR)}}, 2020.

\bibitem[Kompella et~al.(2017)Kompella, Stollenga, Luciw, and
  Schmidhuber]{Kompella17}
Varun~Raj Kompella, Marijn~F. Stollenga, Matthew~D. Luciw, and J{\"{u}}rgen
  Schmidhuber.
\newblock {Continual curiosity-driven skill acquisition from high-dimensional
  video inputs for humanoid robots}.
\newblock \emph{Artificial Intelligence}, 247:\penalty0 313--335, 2017.

\bibitem[Kondor and Lafferty(2002)]{Kondor02}
Risi Kondor and John~D. Lafferty.
\newblock {Diffusion kernels on graphs and other discrete input spaces}.
\newblock In \emph{{International Conference on Machine Learning (ICML)}},
  2002.

\bibitem[Konidaris and Barto(2007)]{Konidaris07}
George~D. Konidaris and Andrew~G. Barto.
\newblock {Building portable options: Skill transfer in reinforcement
  learning}.
\newblock In \emph{International Joint Conference on Artificial Intelligence
  (IJCAI)}, 2007.

\bibitem[Konidaris and Barto(2009)]{Konidaris09}
George~D. Konidaris and Andrew~G. Barto.
\newblock {Skill discovery in continuous reinforcement learning domains using
  skill chaining}.
\newblock In \emph{Advances in Neural Information Processing Systems
  (NeurIPS)}, 2009.

\bibitem[Koren(2003)]{Koren03}
Yehuda Koren.
\newblock {On spectral graph drawing}.
\newblock In \emph{International Computing and Combinatorics Conference
  (COCOON)}, 2003.

\bibitem[Kulkarni et~al.(2016)Kulkarni, Saeedi, Gautam, and
  Gershman]{Kulkarni16a}
Tejas~D. Kulkarni, Ardavan Saeedi, Simanta Gautam, and Samuel~J. Gershman.
\newblock {Deep successor reinforcement learning}.
\newblock \emph{CoRR}, abs/ 1606.02396, 2016.

\bibitem[Lagoudakis and Parr(2003)]{Lagoudakis03}
Michail~G. Lagoudakis and Ronald Parr.
\newblock {Least-squares policy iteration}.
\newblock \emph{Journal of Machine Learning Research}, 4:\penalty0 1107--1149,
  2003.

\bibitem[Li et~al.(2021)Li, Zheng, Wang, and Zhang]{Li21}
Siyuan Li, Lulu Zheng, Jianhao Wang, and Chongjie Zhang.
\newblock {Learning subgoal representations with slow dynamics}.
\newblock In \emph{International Conference on Learning Representations
  (ICLR)}, 2021.

\bibitem[Liu and Abbeel(2021)]{Liu21}
Hao Liu and Pieter Abbeel.
\newblock {{APS:} Active pretraining with successor features}.
\newblock In \emph{International Conference on Machine Learning (ICML)}, 2021.

\bibitem[Liu et~al.(2017)Liu, Machado, Tesauro, and Campbell]{Liu17}
Miao Liu, Marlos~C. Machado, Gerald Tesauro, and Murray Campbell.
\newblock {The eigenoption-critic framework}.
\newblock \emph{CoRR}, abs/1712.04065, 2017.

\bibitem[Liu and Brunskill(2018)]{Liu18}
Yao Liu and Emma Brunskill.
\newblock {When simple exploration is sample efficient: Identifying sufficient
  conditions for random exploration to yield {PAC} {RL} algorithms}.
\newblock \emph{CoRR}, abs/1805.09045, 2018.

\bibitem[Machado(2019)]{Machado19}
Marlos~C. Machado.
\newblock \emph{Efficient Exploration in Reinforcement Learning through
  Time-Based Representations}.
\newblock PhD thesis, University of Alberta, Canada, 2019.

\bibitem[Machado and Bowling(2016)]{Machado16}
Marlos~C. Machado and Michael Bowling.
\newblock {Learning purposeful behaviour in the absence of rewards}.
\newblock \emph{CoRR}, abs/1605.07700, 2016.

\bibitem[Machado et~al.(2017)Machado, Bellemare, and Bowling]{Machado17}
Marlos~C. Machado, Marc~G. Bellemare, and Michael Bowling.
\newblock {A Laplacian framework for option discovery in reinforcement
  learning}.
\newblock In \emph{{International Conference on Machine Learning (ICML)}},
  2017.

\bibitem[Machado et~al.(2018)Machado, Rosenbaum, Guo, Liu, Tesauro, and
  Campbell]{Machado18b}
Marlos~C. Machado, Clemens Rosenbaum, Xiaoxiao Guo, Miao Liu, Gerald Tesauro,
  and Murray Campbell.
\newblock {Eigenoption discovery through the deep successor representation}.
\newblock In \emph{{International Conference on Learning Representations
  (ICLR)}}, 2018.

\bibitem[Machado et~al.(2020)Machado, Bellemare, and Bowling]{Machado20}
Marlos~C. Machado, Marc~G. Bellemare, and Michael Bowling.
\newblock {Count-based exploration with the successor representation}.
\newblock In \emph{Conference on Artificial Intelligence (AAAI)}, 2020.

\bibitem[Mahadevan(2005)]{Mahadevan05}
Sridhar Mahadevan.
\newblock {Proto-value functions: Developmental reinforcement learning}.
\newblock In \emph{International Conference on Machine Learning (ICML)}, 2005.

\bibitem[Mahadevan and Maggioni(2007)]{Mahadevan07}
Sridhar Mahadevan and Mauro Maggioni.
\newblock {Proto-value functions: {A} Laplacian framework for learning
  representation and control in Markov decision processes}.
\newblock \emph{{Journal of Machine Learning Research (JMLR)}}, 8:\penalty0
  2169--2231, 2007.

\bibitem[Mankowitz et~al.(2018)Mankowitz, Z{\'{\i}}dek, Barreto, Horgan,
  Hessel, Quan, Oh, van Hasselt, Silver, and Schaul]{Mankowitz18}
Daniel~J. Mankowitz, Augustin Z{\'{\i}}dek, Andr{\'{e}} Barreto, Dan Horgan,
  Matteo Hessel, John Quan, Junhyuk Oh, Hado van Hasselt, David Silver, and Tom
  Schaul.
\newblock {Unicorn: Continual learning with a universal, off-policy agent}.
\newblock \emph{CoRR}, abs/1802.08294, 2018.

\bibitem[Mann and Mannor(2014)]{Mann14}
Timothy~A. Mann and Shie Mannor.
\newblock {Scaling up approximate value iteration with options: Better policies
  with fewer iterations}.
\newblock In \emph{International Conference on Machine Learning (ICML)}, 2014.

\bibitem[Mendon{\c{c}}a et~al.(2019)Mendon{\c{c}}a, Ziviani, and
  Barreto]{Mendonca19}
Matheus R.~F. Mendon{\c{c}}a, Artur Ziviani, and Andr{\'{e}} Barreto.
\newblock {Laplacian using abstract state transition graphs: {A} framework for
  skill acquisition}.
\newblock In \emph{Brazilian Conference on Intelligent Systems (BRACIS)}, 2019.

\bibitem[Mnih et~al.(2013)Mnih, Kavukcuoglu, Silver, Graves, Antonoglou,
  Wierstra, and Riedmiller]{Mnih13}
Volodymyr Mnih, Koray Kavukcuoglu, David Silver, Alex Graves, Ioannis
  Antonoglou, Daan Wierstra, and Martin Riedmiller.
\newblock {Playing Atari with deep reinforcement learning}.
\newblock \emph{CoRR}, abs/1312.5602, 2013.

\bibitem[Mnih et~al.(2015)Mnih, Kavukcuoglu, Silver, Rusu, Veness, Bellemare,
  Graves, Riedmiller, Fidjeland, Ostrovski, Petersen, Beattie, Sadik,
  Antonoglou, King, Kumaran, Wierstra, Legg, and Hassabis]{Mnih15}
Volodymyr Mnih, Koray Kavukcuoglu, David Silver, Andrei~A. Rusu, Joel Veness,
  Marc~G. Bellemare, Alex Graves, Martin Riedmiller, Andreas~K. Fidjeland,
  Georg Ostrovski, Stig Petersen, Charles Beattie, Amir Sadik, Ioannis
  Antonoglou, Helen King, Dharshan Kumaran, Daan Wierstra, Shane Legg, and
  Demis Hassabis.
\newblock {Human-level control through deep reinforcement learning}.
\newblock \emph{Nature}, 518:\penalty0 529--533, 2015.

\bibitem[Momennejad et~al.(2017)Momennejad, Russek, Cheong, Botvinick, Daw, and
  Gershman]{Momennejad17}
Ida Momennejad, Evan Russek, Jin~Hyun Cheong, Matthew Botvinick, Nathaniel Daw,
  and Samuel Gershman.
\newblock {The successor representation in human reinforcement learning}.
\newblock \emph{Nature Human Behaviour}, 1:\penalty0 680--–692, 2017.

\bibitem[Ng et~al.(1999)Ng, Harada, and Russell]{Ng99}
Andrew~Y. Ng, Daishi Harada, and Stuart~J. Russell.
\newblock {Policy invariance under reward transformations: Theory and
  application to reward shaping}.
\newblock In \emph{International Conference on Machine Learning (ICML)}, 1999.

\bibitem[Pfau et~al.(2019)Pfau, Petersen, Agarwal, Barrett, and
  Stachenfeld]{Pfau18}
David Pfau, Stig Petersen, Ashish Agarwal, David Barrett, and Kimberly
  Stachenfeld.
\newblock {Spectral inference networks: Unifying spectral methods with deep
  learning}.
\newblock In \emph{International Conference on Learning Representations
  (ICLR)}, 2019.

\bibitem[Piaget(1963)]{Piaget63}
Jean Piaget.
\newblock \emph{{The Origins of Intelligence in Children}}.
\newblock W. W. Norton \& Company, 1963.

\bibitem[Precup(2000)]{Precup00}
Doina Precup.
\newblock \emph{{Temporal Abstraction in Reinforcement Learning}}.
\newblock PhD thesis, University of Massachusetts Amherst, 2000.

\bibitem[Ramesh et~al.(2019)Ramesh, Tomar, and Ravindran]{Ramesh19}
Rahul Ramesh, Manan Tomar, and Balaraman Ravindran.
\newblock {Successor options: An option discovery framework for reinforcement
  learning}.
\newblock In \emph{International Joint Conference on Artificial Intelligence
  (IJCAI)}, 2019.

\bibitem[Shi and Malik(2000)]{Shi00}
Jianbo Shi and Jitendra Malik.
\newblock {Normalized cuts and image segmentation}.
\newblock \emph{{IEEE} Transactions on Pattern Analysis and Machine
  Intelligence (T-PAMI)}, 22\penalty0 (8):\penalty0 888--905, 2000.

\bibitem[Solway et~al.(2014)Solway, Diuk, C\'ordova, Yee, Barto, Niv, and
  Botvinick]{Solway14}
Alec Solway, Carlos Diuk, Natalia C\'ordova, Debbie Yee, Andrew~G. Barto, Yael
  Niv, and Matthew~M. Botvinick.
\newblock {Optimal behavioral hierarchy}.
\newblock \emph{{PLOS Computational Biology}}, 10\penalty0 (8):\penalty0 1--10,
  2014.

\bibitem[Sprekeler(2011)]{Sprekeler11}
Henning Sprekeler.
\newblock {On the relation of slow feature analysis and Laplacian eigenmaps}.
\newblock \emph{{Neural Computation}}, 23\penalty0 (12):\penalty0 3287--3302,
  2011.

\bibitem[Stachenfeld et~al.(2014)Stachenfeld, Botvinick, and
  Gershman]{Stachenfeld14}
Kimberly Stachenfeld, Matthew Botvinick, and Samuel Gershman.
\newblock {Design principles of the hippocampal cognitive map}.
\newblock In \emph{Advances in Neural Information Processing Systems
  (NeurIPS)}, 2014.

\bibitem[Stachenfeld et~al.(2017)Stachenfeld, Botvinick, and
  Gershman]{Stachenfeld17}
Kimberly Stachenfeld, Matthew Botvinick, and Samuel Gershman.
\newblock {The hippocampus as a predictive map}.
\newblock \emph{Nature Neuroscience}, 20:\penalty0 1643--1653, 2017.

\bibitem[Stone et~al.(2005)Stone, Sutton, and Kuhlmann]{Stone05}
Peter Stone, Richard~S. Sutton, and Gregory Kuhlmann.
\newblock {Reinforcement learning for RoboCup soccer keepaway}.
\newblock \emph{Adaptive Behaviour}, 13\penalty0 (3):\penalty0 165--188, 2005.

\bibitem[Strang(2005)]{Strang05}
Gilbert Strang.
\newblock \emph{{Linear Algebra and Its Applications}}.
\newblock Brooks Cole, 2005.

\bibitem[Sutton(2016)]{sutton2016towards}
Richard Sutton.
\newblock {Toward a new view of action selection: The subgoal keyboard}.
\newblock Slides presented at the {B}arbados {W}orkshop on {R}einforcement
  {L}earning, 2016.
\newblock URL
  \url{http://barbados2016.rl-community.org/RichSutton2016.pdf?attredirects=0&d=1}.

\bibitem[Sutton(1988)]{Sutton88}
Richard~S. Sutton.
\newblock {Learning to predict by the methods of temporal differences}.
\newblock \emph{Machine Learning}, 3:\penalty0 9--44, 1988.

\bibitem[Sutton et~al.(1999)Sutton, Precup, and Singh]{Sutton99}
Richard~S. Sutton, Doina Precup, and Satinder Singh.
\newblock {Between MDPs and semi-MDPs: A framework for temporal abstraction in
  reinforcement learning}.
\newblock \emph{Artificial Intelligence}, 112\penalty0 (1–2):\penalty0 181 --
  211, 1999.

\bibitem[Sutton et~al.(2011)Sutton, Modayil, Delp, Degris, Pilarski, White, and
  Precup]{Sutton11}
Richard~S. Sutton, Joseph Modayil, Michael Delp, Thomas Degris, Patrick~M.
  Pilarski, Adam White, and Doina Precup.
\newblock {Horde: A scalable real-time architecture for learning knowledge from
  unsupervised sensorimotor interaction}.
\newblock In \emph{International Conference on Autonomous Agents and Multiagent
  Systems (AAMAS)}, 2011.

\bibitem[Sutton et~al.(2022)Sutton, Machado, Holland, Szepesvari, Timbers,
  Tanner, and White]{Sutton22}
Richard~S. Sutton, Marlos~C. Machado, G.~Zacharias Holland, David Szepesvari,
  Finbarr Timbers, Brian Tanner, and Adam White.
\newblock {Reward-respecting subtasks for model-based reinforcement learning}.
\newblock \emph{CoRR}, abs/2202.03466, 2022.

\bibitem[Taylor and Stone(2009)]{Taylor09}
Matthew~E. Taylor and Peter Stone.
\newblock {Transfer learning for reinforcement learning domains: {A} survey}.
\newblock \emph{Journal of Machine Learning Research (JMLR)}, 10:\penalty0
  1633--1685, 2009.

\bibitem[Teh et~al.(2017)Teh, Bapst, Czarnecki, Quan, Kirkpatrick, Hadsell,
  Heess, and Pascanu]{Teh17}
Yee~Whye Teh, Victor Bapst, Wojciech~M. Czarnecki, John Quan, James
  Kirkpatrick, Raia Hadsell, Nicolas Heess, and Razvan Pascanu.
\newblock {Distral: Robust multitask reinforcement learning}.
\newblock In \emph{Advances in Neural Information Processing Systems
  (NeurIPS)}, 2017.

\bibitem[Tomov et~al.(2021)Tomov, Schulz, and Gershman]{Tomov21}
Momchil~S. Tomov, Eric Schulz, and Samuel~J. Gershman.
\newblock {Multi-task reinforcement learning in humans}.
\newblock \emph{Nature Human Behaviour}, 5:\penalty0 764--773, 2021.

\bibitem[Topin et~al.(2015)Topin, Haltmeyer, Squire, Winder, desJardins, and
  MacGlashan]{Topin15}
Nicholay Topin, Nicholas Haltmeyer, Shawn Squire, John Winder, Marie
  desJardins, and James MacGlashan.
\newblock {Portable option discovery for automated learning transfer in
  object-oriented Markov decision processes}.
\newblock In \emph{International Joint Conference on Artificial Intelligence
  (IJCAI)}, 2015.

\bibitem[Vinyals et~al.(2019)Vinyals, Babuschkin, Czarnecki, Mathieu, Dudzik,
  Chung, Choi, Powell, Ewalds, Georgiev, Oh, Horgan, Kroiss, Danihelka, Huang,
  Sifre, Cai, Agapiou, Jaderberg, Vezhnevets, Leblond, Pohlen, Dalibard,
  Budden, Sulsky, Molloy, Paine, Gulcehre, Wang, Pfaff, Wu, Ring, Yogatama,
  Wünsch, McKinney, Smith, Schaul, Lillicrap, Kavukcuoglu, Hassabis, Apps, and
  Silver]{Vinyals19}
Oriol Vinyals, Igor Babuschkin, Wojciech~M. Czarnecki, Micha\"el Mathieu,
  Andrew Dudzik, Junyoung Chung, David~H. Choi, Richard Powell, Timo Ewalds,
  Petko Georgiev, Junhyuk Oh, Dan Horgan, Manuel Kroiss, Ivo Danihelka, Aja
  Huang, Laurent Sifre, Trevor Cai, John~P. Agapiou, Max Jaderberg,
  Alexander~S. Vezhnevets, R\'emi Leblond, Tobias Pohlen, Valentin Dalibard,
  David Budden, Yury Sulsky, James Molloy, Tom~L. Paine, Caglar Gulcehre, Ziyu
  Wang, Tobias Pfaff, Yuhuai Wu, Roman Ring, Dani Yogatama, Dario Wünsch,
  Katrina McKinney, Oliver Smith, Tom Schaul, Timothy Lillicrap, Koray
  Kavukcuoglu, Demis Hassabis, Chris Apps, and David Silver.
\newblock {Grandmaster level in StarCraft II using multi-agent reinforcement
  learning}.
\newblock \emph{Nature}, 575:\penalty0 350--354, 2019.

\bibitem[Wang et~al.(2021)Wang, Zhou, Zhang, Shao, Hooi, and Feng]{Wang21}
Kaixin Wang, Kuangqi Zhou, Qixin Zhang, Jie Shao, Bryan Hooi, and Jiashi Feng.
\newblock {Towards better Laplacian representation in reinforcement learning
  with generalized graph drawing}.
\newblock In \emph{International Conference on Machine Learning (ICML)}, 2021.

\bibitem[Wang et~al.(2007)Wang, Bowling, and Schuurmans]{Wang07}
Tao Wang, Michael Bowling, and Dale Schuurmans.
\newblock {Dual representations for dynamic programming and reinforcement
  learning}.
\newblock In \emph{IEEE International Symposium on Approximate Dynamic
  Programming and Reinforcement Learning (ADPRL)}, 2007.

\bibitem[Watkins and Dayan(1992)]{Watkins92}
Christopher J. C.~H. Watkins and Peter Dayan.
\newblock {Technical note: \cal Q-learning}.
\newblock \emph{Machine Learning}, 8\penalty0 (3-4), 1992.

\bibitem[Wiskott and Sejnowski(2002)]{Wiskott02}
Laurenz Wiskott and Terrence~J. Sejnowski.
\newblock {Slow feature analysis: Unsupervised learning of invariances}.
\newblock \emph{Neural Computation}, 14\penalty0 (4):\penalty0 715--770, 2002.

\bibitem[Wu et~al.(2019)Wu, Tucker, and Nachum]{Wu18}
Yifan Wu, George Tucker, and Ofir Nachum.
\newblock {The Laplacian in {RL:} Learning representations with efficient
  approximations}.
\newblock In \emph{International Conference on Learning Representations
  (ICLR)}, 2019.

\end{thebibliography}

\end{document}